\newif\ifarxiv
\arxivtrue

\newif\ifanonymous
\anonymousfalse

\newif\ifshowrevisions
\showrevisionsfalse

\newif\ifshowrevisionstwo
\showrevisionstwofalse

\documentclass[journal]{IEEEtran}

\ifarxiv
  \newcommand{\FIGSCALE}{1.0}    %

\else
  \newcommand{\FIGSCALE}{0.70}   %

\fi
\usepackage{letltxmacro}
\AtBeginDocument{%
  \LetLtxMacro{\ORIGincludegraphics}{\includegraphics}%
  \renewcommand{\includegraphics}[2][]{%
    \scalebox{\FIGSCALE}{\ORIGincludegraphics[#1]{#2}}}%
}

\ifarxiv\else %
\AtBeginDocument{\setlength{\abovecaptionskip}{3pt}}
\makeatletter
\def\section{\@startsection{section}{1}{\z@}{2.0ex plus 1ex minus 0.5ex}{0.5ex plus 0.5ex minus 0ex}{\normalfont\normalsize\centering\scshape}}
\def\subsection{\@startsection{subsection}{2}{\z@}{2.0ex plus 1ex minus 0.5ex}{0.5ex plus 0.5ex minus 0ex}{\normalfont\normalsize\itshape}}
\makeatother
\fi

\PassOptionsToPackage{table}{xcolor} %
\usepackage{amsmath}
\usepackage{amssymb}
\usepackage{mathtools}
\usepackage{bm}
\usepackage[T1]{fontenc}
\usepackage{newtxtext} %
\usepackage{newtxmath} %

\usepackage{graphicx}
\usepackage[usenames,dvipsnames]{xcolor}

\ifshowrevisions
  \colorlet{revcolor}{blue}
\else
  \colorlet{revcolor}{black}
\fi
\newcommand{\rev}[1]{\textcolor{revcolor}{#1}}
\newcommand{\revframe}[1]{%
  \ifshowrevisions
    {\setlength{\fboxrule}{2pt}\setlength{\fboxsep}{0pt}\fcolorbox{revcolor}{white}{#1}}%
  \else
    #1%
  \fi
}

\ifshowrevisionstwo
  \colorlet{revtwocolor}{blue}
\else
  \colorlet{revtwocolor}{black}
\fi
\newcommand{\revtwo}[1]{\textcolor{revtwocolor}{#1}}

\usepackage{tikz}
\usetikzlibrary{shapes, snakes, tikzmark, calc, arrows.meta, positioning, shapes.multipart, patterns}
\ifarxiv
  \usepackage{caption}
  \DeclareCaptionLabelSeparator{periodspace}{.\quad}
\fi

\usepackage{multirow}
\usepackage{booktabs}
\usepackage{makecell}

\newcolumntype{H}{>{\setbox0=\hbox\bgroup}c<{\egroup}@{}}

\usepackage[sort,compress]{cite}
\usepackage[pdfstartview={FitV},pdfpagelayout={OneColumn},bookmarks=true,bookmarksopen=true,plainpages=false,colorlinks=true,linkcolor=black,citecolor=blue,urlcolor=blue,filecolor=black,pagebackref=false,hypertexnames=false,pdfpagelabels]{hyperref}
\usepackage[nameinlink,capitalize]{cleveref}

\Crefname{figure}{Fig.}{Figs.}
\crefname{figure}{Fig.}{Figs.}
\Crefname{table}{Table}{Tables}
\crefname{table}{Table}{Tables}
\Crefname{section}{Sec.}{Secs.}
\crefname{section}{Sec.}{Secs.}
\Crefname{equation}{}{}
\crefname{equation}{}{}
\Crefname{algorithm}{Alg.}{Algs.}
\crefname{algorithm}{Alg.}{Algs.}

\usepackage[]{siunitx}
\usepackage{xspace}
\usepackage{etoolbox}
\usepackage{pifont}         %
\usepackage{textcomp}
\usepackage{calc}

\robustify\bfseries

\ifdefined\qtyproduct
\else
  \ifdefined\NewCommandCopy
    \NewCommandCopy\qtyproduct\SI
  \else
    \NewDocumentCommand\qtyproduct{O{}mm}{\SI[#1]{#2}{#3}}
  \fi
\fi

\usepackage{amsthm}                                                                                                             
     
\newtheoremstyle{nodefindent}%
  {6pt}   %
  {6pt}   %
  {\normalfont} %
  {0pt}   %
  {\bfseries} %
  {}     %
  { }     %
  {}      %

\theoremstyle{nodefindent}
\newtheorem{assumption}{Assumption}

\newtheorem{theorem}{Theorem}

\newcommand{\ie}{i.e.,\xspace}
\newcommand{\eg}{e.g.,\xspace}
\newcommand{\etalcite}[1]{et al.~\cite{#1}}

\newcommand{\vect}[1]{\boldsymbol{{#1}}}

\newcommand{\cmark}{\cellcolor{LimeGreen!25} \textcolor{ForestGreen}{\ding{51}}}   %
\newcommand{\xmark}{\cellcolor{red!10} \textcolor{red}{\ding{55}}}                 %
\newcommand{\wmark}{\cellcolor{yellow!20} \textcolor{YellowOrange}{\tikz[baseline=-0.1ex]{\draw[line width=0.8pt] (0,0) -- (0.35em,0.6em) -- (0.7em,0) -- cycle;}}} %

\newcommand{\static}{\cellcolor{red!10} \textbf{\textcolor{red}{Static}}}
\newcommand{\dynamic}{\cellcolor{LimeGreen!25} \textbf{\textcolor{ForestGreen}{Both}}}
\newcommand{\open}{\cellcolor{red!10} \textbf{\textcolor{red}{Open}}}
\newcommand{\both}{\cellcolor{LimeGreen!25} \textbf{\textcolor{ForestGreen}{Both}}}

\newcommand{\YesGreen}{\textbf{\textcolor{ForestGreen}{Yes}}}
\newcommand{\NoRed}{\textbf{\textcolor{red}{No}}}
\newcommand{\best}[1]{\textbf{\textcolor{ForestGreen}{#1}}}
\newcommand{\worst}[1]{\textbf{\textcolor{red}{#1}}}

\title{SANDO: Safe Autonomous Trajectory Planning for Dynamic Unknown Environments}

\ifanonymous
  \author{The authors' names are withheld for double-blind reviews.}
\else
  \author{Kota Kondo, Jes\'us Tordesillas, Juan Rached, Lili Sun, Yixuan Jia, and Jonathan P.\ How%
  \thanks{This work was supported by DSTA. \textit{(Corresponding author: Kota Kondo.)}}%
  \thanks{Kota Kondo, Juan Rached, Lili Sun, Yixuan Jia, and Jonathan P.\ How are with the Department of Aeronautics and Astronautics, Massachusetts Institute of Technology, Cambridge, MA 02139 USA (e-mail: kkondo@mit.edu; jrached@mit.edu; lilisun@mit.edu; yixuany@mit.edu; jhow@mit.edu).}%
  \thanks{Jes\'us Tordesillas is with IIT, ICAI, Comillas Pontifical University, Madrid, Spain, and also with the Trustworthy Robotics and AI Lab (TRAIL) (e-mail: jtordesillas@comillas.edu).}%
  \ifarxiv\else\thanks{Supplementary material. Video: \url{https://youtu.be/_T10DJiLQXg}. Code: \url{https://github.com/mit-acl/sando.git}.}\fi}
\fi

\begin{document}

\maketitle

\begin{abstract}
This paper presents SANDO, a safe trajectory planner for 3D dynamic unknown environments.
In such environments, the locations and motions of obstacles are not known in advance, and a collision-free plan can become unsafe at any moment as obstacles move unpredictably, requiring fast replanning.
Existing soft-constraint planners are fast but do not guarantee collision-free paths, while hard-constraint methods typically ensure safety at the cost of longer computation times.
SANDO addresses this trade-off through three main contributions.
First, to avoid overly conservative or infeasible corridors near dynamic obstacles, a heat map-based A$^\ast$ global planner steers the path away from high-risk regions using soft costs, and a spatiotemporal safe flight corridor (STSFC) generator produces time-layered polytopes that inflate obstacles only by their worst-case reachable set at each time layer, rather than by the worst case over the entire horizon.
Second, trajectory optimization is formulated as a mixed-integer quadratic program with hard collision-avoidance constraints, and a variable elimination technique reduces the number of decision variables, enabling fast computation.
Third, a formal safety analysis establishes collision-free guarantees under explicit velocity-bound\revtwo{, size-bound,} and estimation-error assumptions.
Ablation studies confirm that variable elimination provides up to $7.4\times$ speedup in optimization time, and that STSFCs are critical for maintaining feasibility in dense dynamic environments.
Benchmark simulations against state-of-the-art methods across standardized static benchmarks, obstacle-rich forests, and dynamic environments show that SANDO achieves a 100\% success rate across all difficulty levels of the forest and dynamic benchmarks with no constraint violations in any benchmark and additional perception-only experiments without ground truth obstacle information demonstrate the full perception-to-planning pipeline under realistic sensing conditions.
Hardware experiments on a UAV with fully onboard planning, perception, and localization demonstrate six safe flights in static environments and \rev{twelve} safe flights among dynamic obstacles.
\end{abstract}


\ifarxiv
\section*{Supplementary Material}
  \noindent\textbf{Video}: \url{https://youtu.be/_T10DJiLQXg} \\
  \noindent\textbf{Code}: \url{https://github.com/mit-acl/sando.git}
\fi
\ifanonymous
\section*{Supplementary Material}
  \noindent\textbf{Video}: \href{https://doi.org/10.6084/m9.figshare.31964682.v2}{https://doi.org/10.6084/m9.figshare.31964682.v2} \\
  \noindent\textbf{\rev{Rebuttal Video}}: \href{https://doi.org/10.6084/m9.figshare.32667108}{https://doi.org/10.6084/m9.figshare.32667108} \\
  \noindent\textbf{Code}: \href{https://anonymous.4open.science/r/sando-4774}{https://anonymous.4open.science/r/sando-4774}
\fi

\section{Introduction}\label{sec:introduction}

\begin{figure}
    \centering
    \includegraphics[width=\columnwidth]{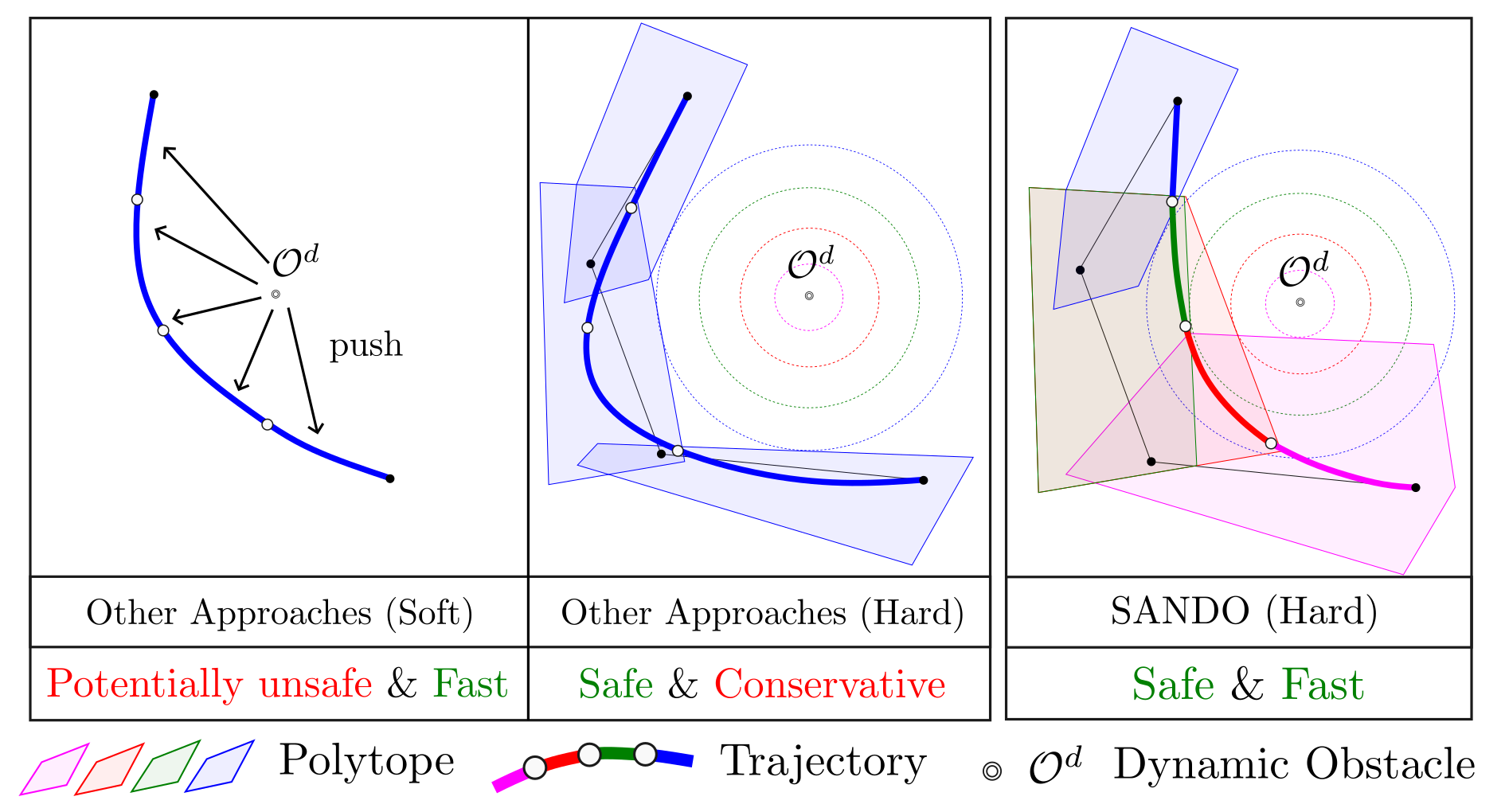}
    \caption{Comparison of trajectory planning approaches near a dynamic obstacle $\mathcal{O}^d$.
    \textbf{Left (Other Approaches -- Soft):} Soft-constraint methods use penalty-based collision avoidance, producing fast but potentially unsafe trajectories.
    \textbf{Center (Other Approaches -- Hard):} Other hard-constraint methods inflate the obstacle by its full worst-case reachable set and generate purely spatial corridors, producing safe but overly conservative trajectories that detour far from the obstacle.
    \textbf{Right (SANDO):} SANDO generates spatiotemporal safe flight corridors (STSFCs) that account for when the trajectory passes each region, inflating the obstacle only by its reachable set at the corresponding time layer, yielding a safe and fast trajectory.
    \label{fig:money_shot}}
\end{figure}

\IEEEPARstart{A}{utonomous} unmanned aerial vehicle (UAV) navigation in dynamic unknown environments is challenging: the future trajectories of obstacles are unknown, so a collision-free path can become unsafe moments after it is computed.
This requires planners that can quickly recompute safe trajectories while handling spatiotemporal collision avoidance.
Existing approaches address parts of this challenge but fall into three categories with distinct limitations (see Table~\ref{tab:state_of_the_art_comparison}).

First, planners designed for unknown static environments (\eg EGO-Planner~\cite{zhou2021ego-planner}, FASTER~\cite{tordesillas2022faster}, HDSM~\cite{toumieh2024high}, and SUPER~\cite{ren2025super}) achieve fast replanning but assume that once space is observed as free, it remains free, making them unsuitable for environments with moving obstacles.
Second, planners that handle dynamic obstacles with soft constraints (\eg HOPA~\cite{feng2021uav}, ViGO~\cite{xu2022vision}, FAPP~\cite{lu2024fapp}, Risk-Aware~\cite{zong2025risk}, and FHD~\cite{Fan2025flying}) can react to moving obstacles but provide no formal safety guarantee, as collision avoidance is encouraged via penalty terms rather than enforced (see Fig.~\ref{fig:money_shot}).
Third, planners that enforce hard safety constraints in dynamic environments (\eg CC-MPC~\cite{lin2020robust}, OA-MPC~\cite{firoozi2024oa}, Liu~\etalcite{liu2023tight}, Stamouli~\etalcite{stamouli2024recursively}, STS~\cite{quan2025state}, and I-MPC~\cite{xu2025intent}) guarantee safety only at discretized model predictive control (MPC) or sample points, leaving the trajectory potentially unsafe between time steps.
While PANTHER~\cite{tordesillas2022panther} provides continuous-time safety in dynamic environments, it has only been demonstrated in open settings and the computation scales poorly with the number of obstacles.
As summarized in Table~\ref{tab:state_of_the_art_comparison}, no existing method simultaneously (1) handles spatiotemporal dynamic obstacle collision avoidance, (2) guarantees continuous-time safety, and (3) operates in dynamic unknown confined environments.
To address these gaps, we introduce SANDO (\textbf{S}afe \textbf{A}uto\textbf{N}omous trajectory planning for \textbf{D}ynamic unkn\textbf{O}wn environments), a trajectory planner that integrates a heat map-based A$^\ast$ global planner, a novel spatiotemporal safe flight corridor (STSFC) generation method, and a variable elimination-based hard-constraint optimization technique to provide continuous-time safety guarantees \rev{within each planning horizon} in dynamic unknown environments.

\begin{table*}[!t]
    \renewcommand{\arraystretch}{1.2}
    \scriptsize
    \begin{centering}
    \caption{State-of-the-art UAV trajectory planners in unknown environments. Safety guarantee: \textcolor{ForestGreen}{\ding{51}}\,=\,continuous, \textcolor{YellowOrange}{$\triangle$}\,=\,discretized, \textcolor{red}{\ding{55}}\,=\,none.}
    \label{tab:state_of_the_art_comparison}
    \resizebox{\textwidth}{!}{
    \begin{tabular}{>{\centering\arraybackslash}m{0.16\textwidth}
                    >{\centering\arraybackslash}m{0.08\textwidth}
                    >{\centering\arraybackslash}m{0.08\textwidth}
                    >{\centering\arraybackslash}m{0.06\textwidth}
                    >{\raggedright\arraybackslash}m{0.35\textwidth}}
    \toprule
    \multirow{2}{*}{\textbf{Method}} & \multicolumn{2}{c}{\textbf{Environments}} & \multirow{2}{*}{\makecell{\textbf{Safety} \\ \textbf{Guarantee}}} & \multicolumn{1}{c}{\multirow{2}{*}{\textbf{Note}}} \tabularnewline
    \cline{2-3}
    & \textbf{Static/Dynamic} & \textbf{Open/Confined} && \tabularnewline
    \midrule
    \textbf{CC-MPC}~\cite{lin2020robust} (2020) & \dynamic & \both & \wmark & Chance hard constraint at discretized MPC steps \tabularnewline
    \cline{1-1} \cline{5-5}
    \textbf{EGO-Planner} \cite{zhou2021ego-planner} (2021) & \static & \both & \xmark & Soft constraint \tabularnewline
    \cline{1-1} \cline{5-5}
    \textbf{HOPA} \cite{feng2021uav} (2021) & \dynamic & \open & \xmark & Soft constraint \tabularnewline
    \cline{1-1} \cline{5-5}
    \textbf{FASTER} \cite{tordesillas2022faster} (2022) & \static & \both & \cmark & Continuous guarantee; Assume static env. \tabularnewline
    \cline{1-1} \cline{5-5}
    \textbf{PANTHER} \cite{tordesillas2022panther} (2022) & \dynamic & \open & \cmark & Continuous guarantee; Probabilistic inflation for obstacle uncertainty \tabularnewline
    \cline{1-1} \cline{5-5}
    \textbf{EGO-Swarm2} \cite{zhou2022swarm} (2022) & \dynamic & \both & \xmark & Soft constraint \tabularnewline
    \cline{1-1} \cline{5-5}
    \textbf{ViGO} \cite{xu2022vision} (2023) & \dynamic & \both & \xmark & Soft constraint \tabularnewline
    \cline{1-1} \cline{5-5}
    \textbf{Liu et al.} \cite{liu2023tight} (2023) & \dynamic & \both & \wmark & Chance hard constraint at discretized MPC steps \tabularnewline
    \cline{1-1} \cline{5-5}
    \textbf{HDSM} \cite{toumieh2024high} (2024) & \static & \both & \cmark & Continuous guarantee; Assume static env. \tabularnewline
    \cline{1-1} \cline{5-5}
    \textbf{Stamouli et al.}~\cite{stamouli2024recursively} (2024) & \dynamic & \open & \wmark & Probabilistic safety guarantee at discretized MPC steps \tabularnewline
    \cline{1-1} \cline{5-5}
    \textbf{Xu et al.} \cite{xu2024heuristic} (2024) & \dynamic & \both & \xmark & Use ViGO~\cite{xu2022vision} as planner \tabularnewline
    \cline{1-1} \cline{5-5}
    \textbf{FAPP} \cite{lu2024fapp} (2025) & \dynamic & \both & \xmark & Soft constraint \tabularnewline
    \cline{1-1} \cline{5-5}
    \textbf{OA-MPC}~\cite{firoozi2024oa} (2025) & \dynamic & \both & \wmark & Hard constraint at discretized MPC steps \& Temporally conservative (See Fig.~\ref{fig:money_shot}) \tabularnewline
    \cline{1-1} \cline{5-5}
    \textbf{SUPER} \cite{ren2025super} (2025) & \static & \both & \xmark & Soft constraint \tabularnewline
    \cline{1-1} \cline{5-5}
    \textbf{FHD} \cite{Fan2025flying} (2025) & \dynamic & \both & \xmark & Learning-based approach \tabularnewline
    \cline{1-1} \cline{5-5}
    \textbf{STS} \cite{quan2025state} (2025) & \dynamic & \both & \wmark & Safety only guaranteed at discretized sample points \tabularnewline
    \cline{1-1} \cline{5-5}
    \textbf{I-MPC} \cite{xu2025intent} (2025) & \dynamic & \both & \wmark & Hard constraint at discretized MPC steps \tabularnewline
    \cline{1-1} \cline{5-5}
    \textbf{Risk-Aware} \cite{zong2025risk} (2025) & \dynamic & \both & \xmark & Soft constraint \tabularnewline
    \cline{1-1} \cline{5-5}
    \textbf{SANDO (proposed)} & \dynamic & \both & \cmark & Continuous guarantee; Bound dynamic obstacles' maximum speed \tabularnewline
    \bottomrule
    \end{tabular}}
    \par\end{centering}
\end{table*}

\section{Related Work}\label{sec:related_work}

\subsection{Environment Taxonomy}\label{subsec:classification_of_environments}

We focus on unknown environments, where no prior map is available, and classify them into two categories: (1) static vs.\ dynamic, and (2) open vs.\ confined.
Static environments have fixed obstacles, whereas dynamic environments include moving obstacles.
Open environments have relatively sparse obstacle distributions, whereas confined environments involve occlusions and narrow passages.

We organize our review around the three criteria in Table~\ref{tab:state_of_the_art_comparison}, namely, environment type (static vs.\ dynamic), environment structure (open vs.\ confined), and safety guarantee level (none, discretized, or continuous).

\subsection{Planning in Unknown Static Environments}\label{sec:lit_review_static}

Several planners are designed for unknown static environments and achieve fast replanning by assuming that once space is observed as free, it remains free.
EGO-Planner~\cite{zhou2021ego-planner} uses a gradient-based local optimizer with soft collision penalties, enabling fast replanning in both open and confined static settings but providing no safety guarantees.
FASTER~\cite{tordesillas2022faster} builds on the mixed-integer quadratic programming (MIQP) trajectory planning approaches~\cite{deits2015miqp, landry2016aggressive} by generating safe flight corridors via convex decomposition and solving a hard-constraint MIQP within them, providing continuous-time collision-free guarantees in static environments; however, it does not model dynamic obstacles.
HDSM~\cite{toumieh2024high} extends the corridor-based approach with a high-degree spline representation that improves trajectory smoothness and corridor utilization, also providing continuous safety guarantees but only under the static assumption.
SUPER~\cite{ren2025super} adopts FASTER's exploratory-safe trajectory framework and replaces the hard-constraint optimizer with a gradient-based solver for faster computation, but trades off safety guarantees by using soft constraints and is limited to static environments.
While these methods work well in static settings, none of them handle dynamic obstacles.

\subsection{Soft-Constrained Planning in Dynamic Environments}\label{sec:lit_review_soft}

A number of planners handle dynamic obstacles but use soft constraints, meaning collision avoidance is encouraged via penalty terms but not strictly enforced.
HOPA~\cite{feng2021uav} uses obstacle positions and potential fields to perform avoidance maneuvers, but operates only in open environments and provides no hard safety guarantee.
ViGO~\cite{xu2022vision} fuses vision-based obstacle detection with a gradient-based planner to navigate dynamic clutter in both open and confined settings, but relies on soft collision penalties.
FAPP~\cite{lu2024fapp} tightly integrates perception and planning for dynamic cluttered environments, demonstrating fast replanning and good practical performance; however, its collision avoidance is soft-constrained using MINCO~\cite{wang2022minco}, so it does not guarantee safety.
Xu~\etalcite{xu2024heuristic} propose a heuristic-based incremental probabilistic roadmap (PRM) for efficient replanning in dynamic environments, using ViGO~\cite{xu2022vision} as the underlying local planner and thus inheriting its soft-constraint limitation.
Risk-Aware~\cite{zong2025risk} incorporates risk metrics into the planning objective to improve safety awareness in dynamic environments, but the risk terms act as soft penalties rather than hard constraints.
FHD~\cite{Fan2025flying} uses a learning-based approach to navigate dynamic environments, achieving agile flight without explicit obstacle modeling; however, the learned policy provides no formal safety guarantee.
While these methods demonstrate strong empirical performance, the absence of hard constraints means safety cannot be formally guaranteed.

\subsection{Safety-Guaranteed Planning in Dynamic Environments}\label{sec:lit_review_safe}

Early approaches to dynamic collision avoidance include velocity obstacles~\cite{fiorini1998velocity}. Formal safety frameworks such as control barrier functions~\cite{ames2019cbf} and Hamilton--Jacobi reachability~\cite{chen2018hj} provide rigorous guarantees but can be computationally expensive in 3D cluttered environments. In the trajectory planning literature, a smaller set of planners enforce collision avoidance as hard constraints in dynamic environments.
We distinguish between discretized guarantees (safety enforced only at sampled time steps or MPC knot points) and continuous guarantees (safety enforced continuously along the trajectory).

\subsubsection{Discretized Safety Guarantees}

CC-MPC~\cite{lin2020robust} formulates chance constraints within an MPC framework to handle obstacle uncertainty, providing probabilistic safety guarantees at discretized MPC steps but not between them.
OA-MPC~\cite{firoozi2024oa} provides hard-constraint guarantees at MPC steps via worst-case reachability analysis and accounts for occluded regions, but ignores the temporal motion history of obstacles, leading to excessive conservatism (see Fig.~\ref{fig:money_shot}).
Liu~\etalcite{liu2023tight} compute tight collision probability bounds and enforce chance constraints at MPC steps, but safety is only guaranteed at discrete time points.
Stamouli~\etalcite{stamouli2024recursively} combine conformal prediction with shrinking-horizon MPC to provide probabilistic safety guarantees at MPC steps without distributional assumptions; however, their approach has only been demonstrated in open environments.
STS~\cite{quan2025state} uses state-time space to handle dynamic environments; however, it is only demonstrated in 2D scenarios and only provides safety guarantees at discretized sample points but not continuously.
I-MPC~\cite{xu2025intent} incorporates intent prediction of dynamic obstacles into an MPC framework with hard constraints at MPC steps, improving avoidance quality but without continuous-time guarantees.

All of these methods share a fundamental limitation: collisions can occur between the discrete time points at which constraints are enforced.

\subsubsection{Continuous Safety Guarantees}

PANTHER~\cite{tordesillas2022panther} constructs per-obstacle convex hulls to enforce continuous-time collision avoidance in dynamic unknown environments using probabilistic reachable-set inflation for obstacle uncertainty.
However, its per-obstacle convex hull representation becomes computationally expensive as the number of obstacles grows, and it has been demonstrated primarily in open environments without many static obstacles.

\subsection{Contributions}\label{subsec:sando_contributions}

To address the gaps identified above, SANDO makes the following contributions:

\begin{enumerate}
    \item \textbf{STSFC Generation:} A novel spatiotemporal safe flight corridor generation method that produces time-layered polytope sequences accounting for worst-case obstacle reachable sets at each time layer, addressing the lack of time-varying corridor generation framework for dynamic environments (Section~\ref{subsec:temporal_safety_corridor}).
    \item \textbf{Heat Map-based Global Planner:} A heat map-based A$^\ast$ planner that assigns soft costs around both static and dynamic obstacles, guiding the global path toward regions where larger STSFC corridors can be generated (Section~\ref{subsec:global_planner_heatmap}).
    \item \textbf{MIQP Trajectory Optimization with STSFCs:} A hard-constraint MIQP formulation that assigns each trajectory piece to a time-layered STSFC polytope, whose continuous-time collision-free guarantee is established by the formal safety analysis below (Section~\ref{sec:trajectory_optimization}).
    \item \textbf{Formal Safety Analysis:} \rev{Per-horizon continuous-time safety guarantees through reachable set inflation, under explicit assumptions on bounded obstacle velocity, bounded estimation error, and tracking error (Section~\ref{subsec:theoretical_safety_guarantees}).
    Recursive feasibility across replanning cycles is discussed as an open problem in dynamic environments, together with SANDO's fallback behavior when the trajectory optimization becomes infeasible.}
    \item \textbf{Extensive Evaluation} in simulation across diverse environments and hardware experiments on a fully automated UAV platform, demonstrating the practical effectiveness of SANDO in real-world scenarios.
\end{enumerate}

\section{System Overview}\label{sec:system_overview_section}

\begin{figure*}[htbp]
    \centering
    \includegraphics[width=\textwidth]{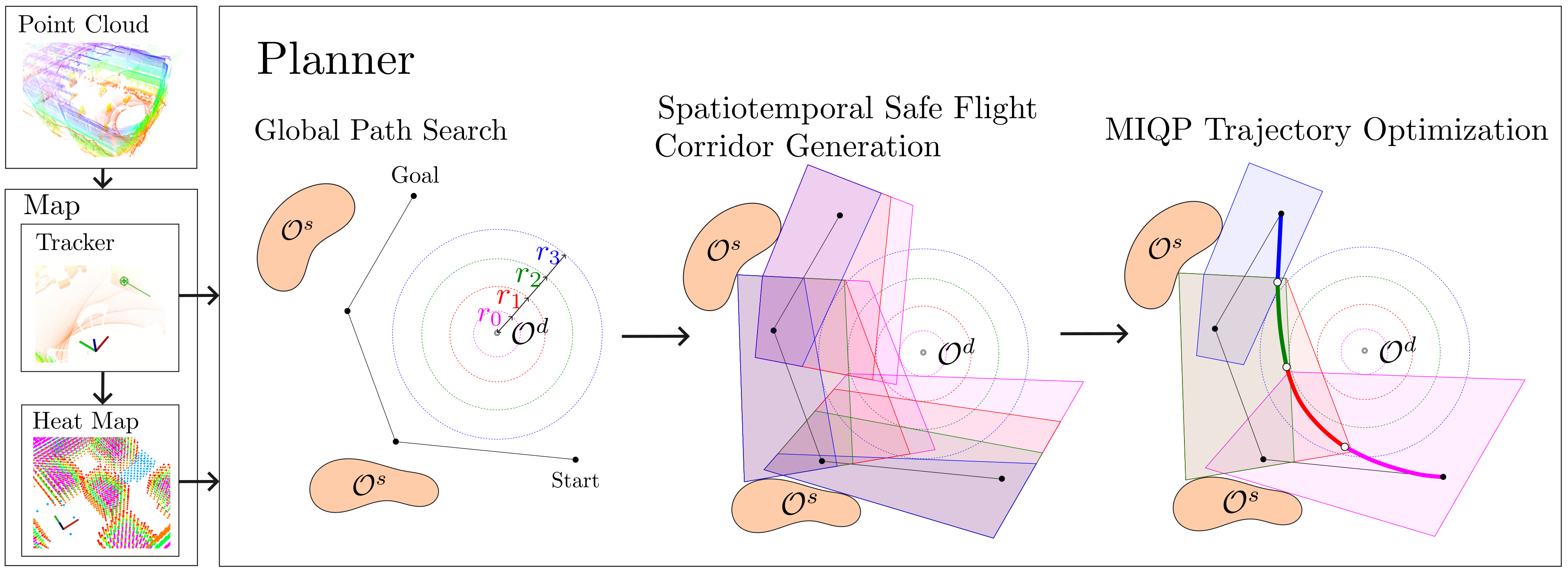}
    \caption{SANDO system overview: Point cloud data are processed by the Map Manager, which detects static obstacles (denoted $\mathcal{O}^s$) and dynamic obstacles (denoted $\mathcal{O}^d$) and estimates the states of dynamic obstacles.
    The Dynamic Obstacle Tracker feeds its predictions to the Heat Map module, which generates soft costs around both static and dynamic obstacles (Section~\ref{subsec:global_planner_heatmap}).
    The heat map-based global planner computes a global path using A$^\ast$ search, and the predicted obstacle trajectories are used by the STSFC Generation module to produce time-layered polytope sequences (Section~\ref{subsec:temporal_safety_corridor}).
    These time-varying polytopes are then used in the Trajectory Optimization module to compute a safe trajectory (Section~\ref{sec:trajectory_optimization}).\label{fig:system_overview}}
\end{figure*}

SANDO consists of five modules (see Fig.~\ref{fig:system_overview}): a dynamic obstacle tracker, a heat map generator, a global path planner, an STSFC generator, and a hard-constraint local trajectory optimizer.
Each module is designed to handle \textbf{dynamic unknown obstacles}, which requires planning in spatiotemporal space, adapting trajectories as obstacles move unpredictably, and computing safe trajectories in real time.

SANDO processes point cloud data from a LiDAR sensor and/or a depth camera.
The point cloud is processed by the map manager, which generates a voxel map with heat map costs for both static and dynamic obstacles.
Point cloud data are also used by a dynamic obstacle tracker, where dynamic obstacles are detected, clustered, and tracked to estimate their current positions and predict their future trajectories, as detailed in Section~\ref{sec:obstacle_tracking}.

The voxel map and predicted obstacle trajectories are provided as inputs to the heat map-based global planner and the STSFC Generation module.
The global planner computes a global path from the agent's start position to a \emph{subgoal} (a projection of the goal position on the local map around the agent) using heat map-based A$^\ast$ search with soft costs for both static and dynamic obstacles; see Section~\ref{subsec:global_planner_heatmap} for details.
SANDO then constructs STSFCs, which account for spatial and temporal dimensions, by inflating dynamic obstacles based on their worst-case reachable sets at each time layer, as described in Section~\ref{subsec:temporal_safety_corridor}.

Since dynamic obstacles can change their motion unpredictably, fast trajectory optimization is needed.
SANDO introduces a variable elimination approach that reduces the number of decision variables and constraints in the optimization problem, enabling fast hard-constraint optimization; see Section~\ref{sec:trajectory_optimization}.
The resulting trajectory is then sent to the low-level controller for execution.
Replanning is triggered at 100\,Hz; however, a new replanning cycle begins only after the previous one completes, so the effective replanning rate is bounded by the total replanning time (typically \rev{18--36}\,ms in hardware, yielding an effective rate of approximately \rev{28--56}\,Hz).
If replanning fails to find a feasible trajectory, the agent continues executing the previously computed trajectory; if no valid trajectory remains, the agent hovers in place.
Point cloud processing, dynamic obstacle tracking, map management, and SANDO's planning modules are all parallelized.

\section{Spatiotemporal Safe Flight Corridor (STSFC) Generation}\label{subsec:temporal_safety_corridor}

We first define the key terminology used throughout the paper.
Let $t_0$ denote the \emph{current planning time}.
The trajectory is composed of $N$ \emph{pieces} $\boldsymbol{x}_0, \ldots, \boldsymbol{x}_{N-1}$, each a cubic B\'ezier polynomial indexed by $n \in \{0, \ldots, N{-}1\}$.
Each piece $n$ executes during the \emph{time interval} $[t_0 + n \cdot dt,\; t_0 + (n{+}1) \cdot dt]$, where $dt$ is the uniform duration per piece, so that the full trajectory spans $[t_0,\; t_0 + N \cdot dt]$.
The global path consists of $P$ \emph{path segments}, indexed by $p \in \{0, \ldots, P{-}1\}$.
Each of the $K$ tracked dynamic obstacles, indexed by $k \in \{1, \ldots, K\}$, occupies a region $\mathcal{O}_k(t) \subset \mathbb{R}^3$ at time $t$, with true centroid $\mathbf{c}_k(t) \in \mathbb{R}^3$.
The obstacle region is contained in an axis-aligned bounding box (AABB) centered at $\mathbf{c}_k(t)$ with half-extents $\mathbf{h}_k = (h_k^x,\, h_k^y,\, h_k^z) \in \mathbb{R}_{>0}^3$:
\[
    \mathcal{O}_k(t) \;\subseteq\; \bigl\{\mathbf{x} \in \mathbb{R}^3 : |x_i - c_k^i(t)| \le h_k^i, \;\; i \in \{x,y,z\}\bigr\}.
\]

In dynamic environments, safety corridors must account for both spatial constraints and how obstacles move over time.
Traditional corridor generation methods, such as~\cite{liu2017planning,tordesillas2022faster,deits2014iris}, produce purely spatial polytopes along a path, which is sufficient for static environments but does not capture how free space changes as obstacles move.
To this end, we introduce \emph{STSFCs}, time-varying polytope sequences where each time layer represents the collision-free region at a given duration in the trajectory's time horizon.
This section describes the structure of STSFCs, the time allocation strategy, and the obstacle inflation method used to ensure safety guarantees.
The global path input used during corridor generation is described in Section~\ref{subsec:global_planner_heatmap}.

\subsection{STSFC Structure}

An STSFC is represented as a two-dimensional array $\mathcal{C}[n][p]$ where:
\begin{itemize}
    \item $n \in \{0, \ldots, N-1\}$ indexes the \emph{time layer}, representing discrete time intervals along the trajectory,
    \item $p \in \{0, \ldots, P-1\}$ indexes the \emph{spatial polytope} within each time layer, representing free-space regions at that time.
\end{itemize}

Each element $\mathcal{C}[n][p]$ is a convex polytope defined by linear constraints $\{\mathbf{F}_{np}, \mathbf{g}_{np}\}$ such that a point $\mathbf{x} \in \mathbb{R}^3$ is inside the polytope if $\mathbf{F}_{np}\mathbf{x} \le \mathbf{g}_{np}$.
Each time layer $n$ corresponds to the time interval $[t_0 + n \cdot dt,\; t_0 + (n{+}1) \cdot dt]$, matching the time interval of trajectory piece $n$.
This structure allows the trajectory optimizer to select different spatial corridors at different times, adapting to the evolving obstacle configuration.
Note that we generate a polytope for every combination of time layer $n$ and path segment $p$, even though some assignments may seem unlikely (\eg the last time layer in the first path segment, or the first time layer in the last path segment).
This is necessary because the MIQP solver determines the piece-to-polytope assignment, and the assignment of pieces is not known a priori.

\rev{Each polytope $\mathcal{C}[n][p]$ is generated using the ellipsoid-based convex decomposition of Liu~\etalcite{liu2017planning}, applied to the $p$-th segment of the smoothed global path against static obstacles and the per-layer inflated dynamic-obstacle set for time layer $n$ (the inflation construction is described in Section~\ref{subsec:obstacle_inflation}).
By construction, this decomposition produces a convex polytope that fully contains its associated path segment.}

\rev{This per-segment construction enables the MIQP to find continuous trajectories across consecutive polytopes without an explicit overlap constraint.
The continuity requirement $\mathbf{x}_{n+1}(0)=\mathbf{x}_{n}(dt)$ places the boundary point between pieces $n$ and $n{+}1$ in both assigned polytopes, and segment containment guarantees this intersection is non-empty.
Since consecutive pieces are assigned either to the same path segment, which lies in both polytopes with different per-layer reachable sets, or to adjacent segments, whose shared waypoint lies in both polytopes, if the ellipsoid-based decomposition succeeds, the resulting polytopes are guaranteed to have non-empty overlap at the piece boundaries.}

\subsection{Known Obstacle Inflation by Reachable Extent}\label{subsec:obstacle_inflation}

\rev{For each dynamic obstacle $k$ with estimated AABB $\hat{\mathcal{O}}_k(t_0)$ at planning time $t_0$ and each time layer $n$, we compute the worst-case reachable position set by Minkowski-summing the obstacle's AABB with the per-layer reachable set $\mathcal{R}_n$, the position-estimation buffer $C_\infty(\epsilon)$\revtwo{, and the vehicle's physical extent $\mathcal{D}$}:
\begin{equation}\label{eq:safety_rn}
\hat{\mathcal{O}}_k^n \coloneqq \hat{\mathcal{O}}_k(t_0) \;\oplus\; \mathcal{R}_n \;\oplus\; C_\infty(\epsilon) \revtwo{\;\oplus\; \mathcal{D}},
\end{equation}
where $\mathcal{R}_n$ is determined by the obstacle velocity bound in Assumption~\ref{ass:bounded_vel} (Section~\ref{subsec:theoretical_safety_guarantees}): the Euclidean ball $B_2(v^{\mathrm{obs}}_{\max}(n{+}1)\,dt)$ under Formulation~(A), or the axis-aligned box with half-widths $v^{\mathrm{obs}}_{\max,i}(n{+}1)\,dt$ under Formulation~(B); $C_\infty(\epsilon)$ is an axis-aligned cube of half-side $\epsilon$ accounting for the bounded position-estimation error from the obstacle tracker (Section~\ref{sec:obstacle_tracking})\revtwo{; and $\mathcal{D} \coloneqq \bigl\{ \boldsymbol{\delta} \in \mathbb{R}^3 : |\delta_i| \le d^i, \; i \in \{x,y,z\} \bigr\}$ is the vehicle's own axis-aligned bounding box, with half-extents $\mathbf{d} = (d^x, d^y, d^z)$ measured from the vehicle's center, so that the vehicle occupies $\{\vect{x}\} \oplus \mathcal{D}$ when its center is at $\vect{x}$}.
\revtwo{The last term makes $\hat{\mathcal{O}}_k^n$ a configuration-space obstacle~\cite{lozanoperez1983spatial}, so the guarantee applies to the whole vehicle and not only to its center (Theorem~\ref{thm:safety}).}
Because $\mathcal{R}_n$ grows with $n$, obstacles are inflated more in later time layers, reflecting the increasing uncertainty in obstacle position over time.}

\rev{In SANDO's implementation, the known-obstacle inflation always uses the per-axis form (Formulation~(B) of Assumption~\ref{ass:bounded_vel}), and therefore $C_\infty(\epsilon)$ is an axis-aligned cube of half-side $\epsilon$: the AABB half-extents $\mathbf{h}_k$ are grown to $\mathbf{h}_k + (v^{\mathrm{obs}}_{\max,x},v^{\mathrm{obs}}_{\max,y},v^{\mathrm{obs}}_{\max,z})(n{+}1)\,dt$ along each axis (plus $\epsilon$ on each side from the estimation buffer), keeping the inflated region as an AABB.
Which formulation SANDO uses in each part of the pipeline, and why each choice preserves the safety guarantee, is summarized after Assumption~\ref{ass:bounded_vel} in Section~\ref{subsec:theoretical_safety_guarantees}.}

\revtwo{We account for the vehicle's size in two ways, one for each obstacle type.
For tracked dynamic obstacles we grow the obstacle, adding $\mathcal{D}$ to the estimated AABB half-extents exactly as Eq.~\eqref{eq:safety_rn} is written.
For static occupied voxels and for the inflated unknown-space boundary we shrink the corridor instead, because the convex decomposition carves each polytope so that all of its points are at least $r_{\text{drone}}$ from the obstacle set.}

\revtwo{Growing the obstacle keeps the whole vehicle clear because $\mathcal{D}$ is symmetric about the vehicle's center.
If the vehicle body $\{\vect{x}\} \oplus \mathcal{D}$ overlapped the obstacle, then some point of that body, $\vect{x} + \boldsymbol{\delta}$ with $\boldsymbol{\delta} \in \mathcal{D}$, would lie in the obstacle, and since $-\boldsymbol{\delta} \in \mathcal{D}$ the center $\vect{x}$ would lie in the grown obstacle.
The corridor keeps $\vect{x}$ out of the grown obstacle $\hat{\mathcal{O}}_k^n$, so this cannot happen and the whole vehicle stays clear of obstacle $k$.
Shrinking the corridor does the same directly, since the decomposition models the vehicle as a ball of radius $r_{\text{drone}}$.
Either way the vehicle's extent is applied exactly once.
The grown region $\hat{\mathcal{O}}_k^n$ is the standard configuration-space obstacle of~\cite{lozanoperez1983spatial}.}

\subsection{Unknown Space Inflation}\label{subsubsec:unknown_space_inflation}

\begin{figure}
    \centering
    \revframe{\includegraphics[width=\columnwidth]{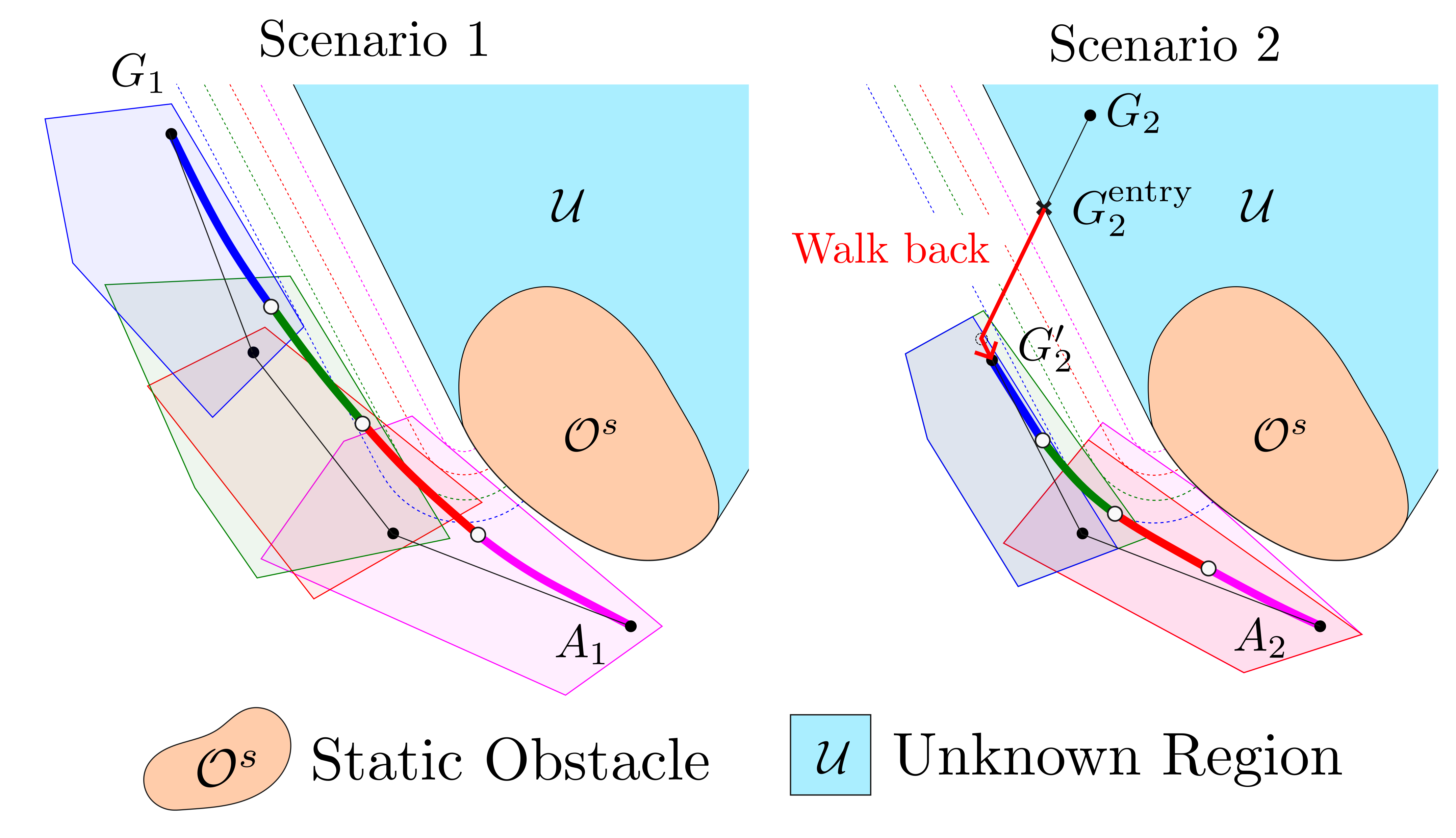}}
    \caption{Unknown-space inflation for STSFC corridor generation.
    Boundary voxels of unknown regions are inflated by the same per-layer \rev{$\epsilon$-buffered reachable set $\mathcal{R}_n^\epsilon$} as dynamic obstacles (shown as an L2 ball in this figure; the per-axis box variant is analogous).
    \textbf{Scenario~1:} The goal $G_1$ lies outside the unknown region. The trajectory from $A_1$ to $G_1$ stays outside the inflated boundary, guaranteeing safety against unobserved dynamic obstacles.
    \textbf{Scenario~2:} The goal $G_2$ lies inside the unknown region. SANDO walks back along the global path from the unknown boundary until a point outside the worst-case inflated region is found, yielding a new subgoal $G_2'$.
    \label{fig:unknown_inflation}}
\end{figure}

SANDO maintains a voxel map that is built incrementally from sensor observations: each voxel is classified as free, occupied, or unknown, where unknown voxels are those that have not yet been observed by the sensor.
In such environments, dynamic obstacles may emerge from unknown regions at any time.
If these regions are treated as free during corridor generation, the resulting corridors may overlap with space that is actually occupied, violating safety.
To address this, SANDO inflates the boundaries of unknown (unobserved) voxels, treating them as occupied during the ellipsoid-based corridor decomposition.

Specifically, let $\mathcal{U}(t_0) \subset \mathbb{R}^3$ denote the set of unknown voxels at planning time $t_0$, and let $\mathcal{B}_{\mathcal{U}}(t_0)$ denote its boundary voxels.
\rev{Each boundary voxel is inflated by the per-layer reachable set $\mathcal{R}_n$ used for dynamic obstacles (Eq.~\eqref{eq:safety_rn}), augmented by the same estimation buffer $\epsilon$ \revtwo{and the vehicle extent, applied here as the ball $B_2(r_{\text{drone}})$}:
\begin{equation}\label{eq:unknown_inflation}
    \hat{\mathcal{U}}^n \coloneqq \mathcal{U}(t_0) \;\cup\; \bigl(\mathcal{B}_{\mathcal{U}}(t_0) \oplus \mathcal{R}_n^\epsilon \revtwo{\oplus B_2(r_{\text{drone}})}\bigr),
\end{equation}
where $\oplus$ denotes the Minkowski sum and $\mathcal{R}_n^\epsilon$ is the $\epsilon$-buffered reachable set in the norm of the chosen formulation of Assumption~\ref{ass:bounded_vel} (Section~\ref{subsec:theoretical_safety_guarantees}): the Euclidean ball $B_2\bigl(v^{\mathrm{obs}}_{\max}(n{+}1)\,dt + \epsilon\bigr)$ under Formulation~(A), or the axis-aligned box with per-axis half-widths $v^{\mathrm{obs}}_{\max,i}(n{+}1)\,dt + \epsilon$ under Formulation~(B).
Since $\mathcal{R}_n \subseteq \mathcal{R}_n^\epsilon$, this inflation contains the reachable-set inflation required by the safety analysis.
\revtwo{The decomposition applies the vehicle term through its inflation distance, as it does for static voxels: for the unknown region the vehicle is modeled as the ball $B_2(r_{\text{drone}})$ rather than the box $\mathcal{D}$, and the guarantee for untracked obstacles is stated under that model.}}
The inflated unknown voxels are included in the obstacle set used for corridor generation, ensuring that any dynamic obstacles that could emerge from these regions are accounted for in the resulting polytopes $\mathcal{C}[n][p]$ (see Fig.~\ref{fig:unknown_inflation}).
When the goal lies outside the unknown region (Scenario~1 in Fig.~\ref{fig:unknown_inflation}), the trajectory stays entirely in observed space and the inflated boundary provides a safety buffer against unobserved obstacles.

\paragraph{Safe subgoal selection}
\rev{When the global path enters the unknown region (Scenario~2 in Fig.~\ref{fig:unknown_inflation}), corridors cannot be generated beyond the unknown boundary, and SANDO selects an intermediate subgoal $G'$ via a discrete backward traversal, executed before any polytopes are generated.
Once per replanning cycle, we precompute a single worst-case inflation radius $r_{\mathrm{inflate}} = v^{\mathrm{obs}}_{\max,\mathrm{enc}} \cdot T_{\max}$, where $T_{\max} = f_M \cdot N \cdot dt_0^{\mathrm{wc}}$ upper-bounds the planning horizon of every optimization thread launched in the cycle ($f_M$ is the largest factor in the current time-allocation window, Section~\ref{subsec:time_allocation}, and $dt_0^{\mathrm{wc}}$ a worst-case baseline piece duration, so $dt_0^{\mathrm{wc}} \ge dt_0$ in every cycle), and $v^{\mathrm{obs}}_{\max,\mathrm{enc}}$ is the smallest Euclidean bound whose ball encloses the per-layer reachable set $\mathcal{R}_n$: $v^{\mathrm{obs}}_{\max}$ under Formulation~(A), and $\sqrt{\sum_i (v^{\mathrm{obs}}_{\max,i})^2}$ under Formulation~(B).
The global path $\mathcal{P}: A \to G$ is then traversed forward at fixed step size $\Delta s = \SI{0.1}{m}$ to find the first sample whose drone footprint enters $\mathcal{U}(t_0)$, and from there backward until the first candidate lying outside $\mathcal{U}(t_0) \oplus B_2\bigl(r_{\text{drone}} + r_{\mathrm{inflate}}\bigr)$, tested via Euclidean nearest-neighbor queries on a KD-tree of unknown voxels; that candidate becomes $G'$.
Under Formulation~(B) this Euclidean query approximates the axis-aligned box inflation by its circumscribing ball, whose radius $\sqrt{\sum_i (v^{\mathrm{obs}}_{\max,i})^2}\,T_{\max}$ is the half-diagonal of the box.
Since the ball fully encloses the per-axis reachable set, this check is conservative and never accepts an unsafe subgoal.}

The resulting STSFCs are used as constraints in the MIQP trajectory optimization described in Section~\ref{sec:trajectory_optimization}.

\section{Heat Map-Based Global Planner}\label{subsec:global_planner_heatmap}

\begin{figure*}
    \centering
    \includegraphics[width=\textwidth]{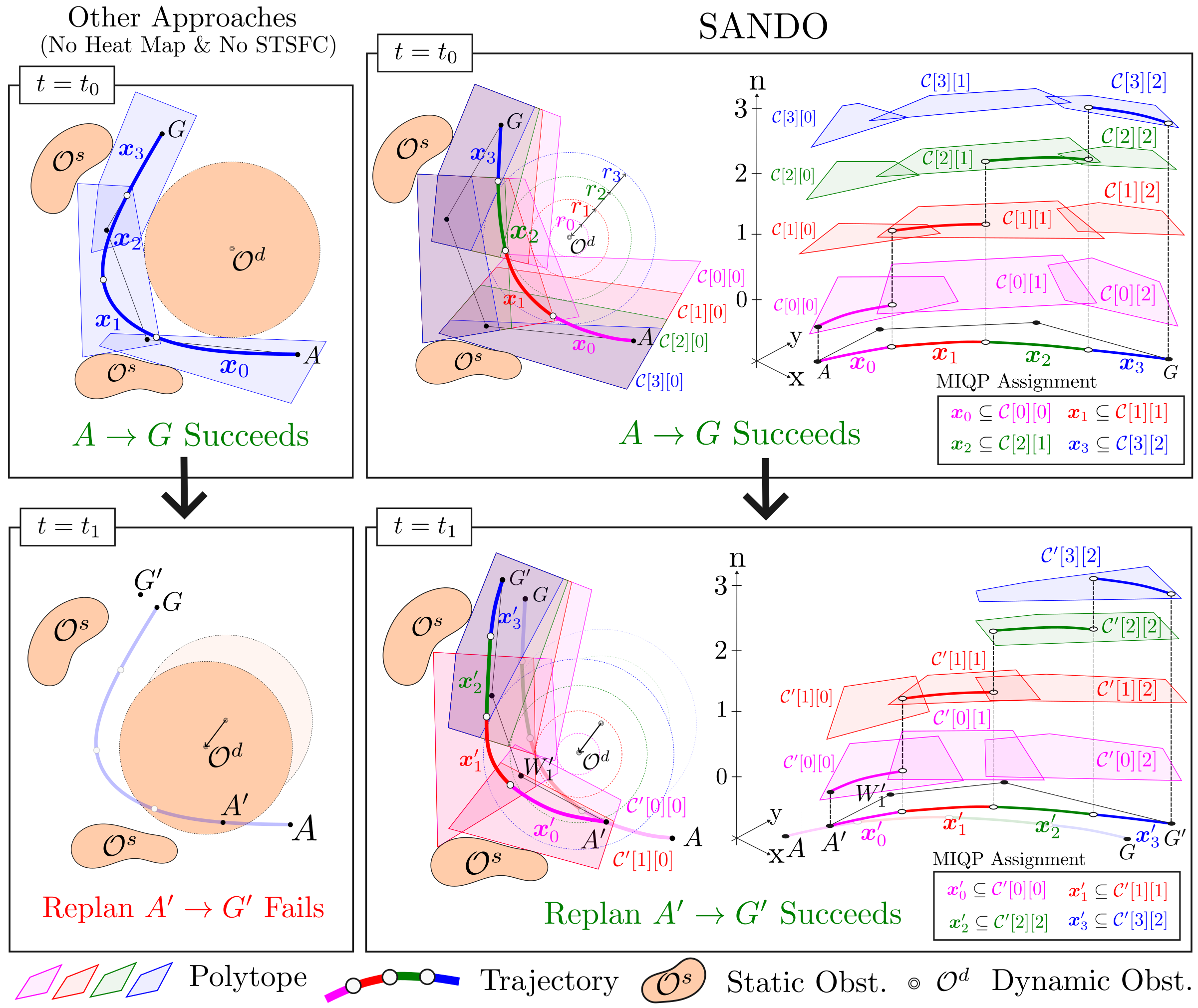}
    \caption{STSFC and MIQP trajectory optimization in dynamic environments. For visualization, the figure shows a 2D ($x$-$y$) case with the time-layer axis as the third dimension; in practice, SANDO uses full 3D ($x$, $y$, $z$) STSFCs.
    \textbf{Left column, other approaches (hard-blocking reachable sets):}
    At $t=t_0$ (top left), the planner hard-blocks the worst-case reachable set of $\mathcal{O}^d$ as occupied and computes a conservative global path $A \to G$ that avoids the entire blocked region.
    Although planning succeeds, the resulting trajectory is overly conservative.
    At $t=t_1$ (bottom left), the agent replans from the new position $A'$ toward $G'$, but $A'$ now falls inside the shifted blocked region, making the query $A' \to G'$ infeasible.
    \textbf{Right column, SANDO (heat map-based soft costs):}
    At $t=t_0$, the 2D view (top center) shows the global path $A \to G$ steered away from $\mathcal{O}^d$ by soft heat costs, producing a less conservative trajectory.
    The temporal view (top right) visualizes the STSFC corridors $\mathcal{C}[n][p]$ along the time-layer axis, with increasing per-layer inflation; the MIQP assigns each trajectory piece to a corridor ($\boldsymbol{x}_0 \subseteq \mathcal{C}[0][0]$, $\boldsymbol{x}_1 \subseteq \mathcal{C}[1][1]$, $\boldsymbol{x}_2 \subseteq \mathcal{C}[2][1]$, and $\boldsymbol{x}_3 \subseteq \mathcal{C}[3][2]$).
    At $t=t_1$, the 2D view (bottom center) and temporal view (bottom right) show successful replanning from $A'$ to $G'$: the updated corridors $\mathcal{C}'[n][p]$ remain valid and the MIQP assignment ($\boldsymbol{x}'_0 \subseteq \mathcal{C}'[0][0]$, $\boldsymbol{x}'_1 \subseteq \mathcal{C}'[1][1]$, $\boldsymbol{x}'_2 \subseteq \mathcal{C}'[2][2]$, and $\boldsymbol{x}'_3 \subseteq \mathcal{C}'[3][2]$) succeeds, since soft costs never blocked the region around $A'$.
    Even though in the first global path section ($A' \to W_1$), the corridor generation for $n=2$ (green) and $n=3$ (blue) fails to produce corridors, since $n=0$ (pink) and $n=1$ (red) still have valid corridors, the MIQP can find a feasible trajectory that safely navigates through the dynamic environment.
    \label{fig:safe_corridor_in_dynamic_environments}}
\end{figure*}

As described in Section~\ref{subsec:temporal_safety_corridor}, STSFC corridors shrink as dynamic obstacles are inflated by their worst-case reachable sets.
If the global path passes close to dynamic obstacles, the resulting corridors may become too small for the trajectory optimizer to find a feasible solution.
A naive approach would be to hard-block the entire worst-case reachable set as occupied in the planner's occupancy grid.
However, this creates a critical replanning failure mode: at the next replanning step, the agent's current position may fall inside the previously blocked region, since the obstacle has moved and the reachable set has shifted, making the new planning query infeasible (see the hard-blocking approach in the left column at $t=t_1$ in Fig.~\ref{fig:safe_corridor_in_dynamic_environments}).

To avoid this, SANDO employs a heat map-based A$^\ast$ planner that uses soft costs rather than hard occupancy constraints for dynamic obstacles.
The heat map proactively steers the global path away from regions where corridors will shrink, without blocking those regions entirely.
This ensures that the global path search always has a feasible start configuration, while still biasing the path toward regions with larger corridors.
Static and dynamic obstacles are both incorporated through heat maps, and only known static obstacles and the current location of dynamic obstacles enforce hard infeasibility.
These occupied voxels are inflated by the drone radius $r_{\text{drone}}$ (\revtwo{half the vehicle's width}) to account for the agent's physical extent.
Unknown space is treated as free in the global planner to allow the agent to plan paths toward unexplored regions.
In contrast, during STSFC generation, unknown space is treated as occupied, and when unknown-space inflation is enabled, the boundaries of unknown regions are further inflated so that the resulting corridors account for obstacles that may emerge from unobserved space (see Section~\ref{subsubsec:unknown_space_inflation}).

\subsection{Static Obstacle Heat Map}

We define a static heat map $H^{s} : \mathbf{q} \in \mathbb{R}^3 \to \mathbb{R}_{\ge 0}$ using a distance-based cost that decreases with distance from obstacle surfaces.
Only boundary voxels (surface voxels of obstacles) serve as heat sources.
For each boundary voxel $b$ with center $\mathbf{c}_b$ and halo radius $R^s$ (the maximum distance at which the voxel influences the cost), the heat contribution at a query point $\mathbf{q} \in \mathbb{R}^3$ is:
\[H^{s}_b(\mathbf{q}) =
\begin{cases}
\alpha^s \left(1 - \dfrac{\|\mathbf{q} - \mathbf{c}_b\|}{R^s}\right)^{p^s}, & \|\mathbf{q} - \mathbf{c}_b\| \le R^s, \\[6pt]
0, & \text{otherwise},
\end{cases}\]
where $\alpha^s$ is the intensity scale and $p^s$ controls the decay shape.
The aggregate static heat is $H^{s}(\mathbf{q}) = \min(\max_b H^{s}_b(\mathbf{q}),\, H_{\max})$, using max aggregation to avoid artificially inflating costs in dense obstacle regions, capped at $H_{\max}$ (the maximum allowable heat value).
See Table~\ref{tab:system_parameters} for the static heat parameters used in our experiments.

\subsection{Dynamic Obstacle Heat Map}

For each tracked dynamic obstacle $k$ with estimated position $\hat{\mathbf{c}}_k$, AABB half-extents $\mathbf{h}_k$, and predicted trajectory $\mu_k(t)$ from the obstacle tracker (Section~\ref{sec:obstacle_tracking}) over a prediction time horizon $T_h$, we construct a per-obstacle heat $H_k(\mathbf{q})$ combining a penalty around the current position and a penalty around the predicted trajectory tube.

\paragraph{Base heat around the current position}
We define the base obstacle radius as $R_{0,k} \coloneqq \max_i h_k^i + \rev{r_{\text{base}}}$, where \rev{$r_{\text{base}}$ is a fixed base inflation applied to dynamic obstacles during global planning} (see Table~\ref{tab:experiment_parameters}), and the horizon-limited reachable radius as $R^d_k \coloneqq R_{0,k} + \bigl(\max_i v^{\mathrm{obs}}_{\max,i}\bigr) T_h$.
The base heat is:
\[H^{\mathrm{base}}_k(\mathbf{q}) \;=\;
\begin{cases}
\alpha^d_0\,\left(1-\dfrac{\|\mathbf{q}-\hat{\mathbf{c}}_k\|}{R^d_k}\right)^{p^d}, & \|\mathbf{q}-\hat{\mathbf{c}}_k\| \le R^d_k, \\[6pt]
0, & \text{otherwise},
\end{cases}\]
where $\alpha^d_0$ is the base cost scale and $p^d$ controls the decay shape.

\paragraph{Tube penalty from predicted motion}
We discretize the horizon into $M_{\text{tube}}$ samples $\{t_j\}_{j=0}^{M_{\text{tube}}-1}$ with predicted centers $\hat{\mathbf{c}}_{k,j}=\mu_k(t_j)$.
The tube radius $R_{k,j} \coloneqq R_{0,k} + \gamma_k t_j$ expands with prediction time at rate $\gamma_k$ to account for increasing uncertainty, where $\gamma_k \ge 0$ is the \rev{uncertainty growth rate, set to $0$ in simulation (a constant-radius tube; Table~\ref{tab:system_parameters}) and $\SI{0.3}{m/s}$ on hardware}.
The tube penalty is:
\begin{multline*}
\!\! H^{\mathrm{tube}}_k\! (\mathbf{q}) \! =\! \alpha^d_1 \max_{j} \Biggl(\! w(t_j)
\! \left[\! \max\!\left(\! 0,\,1\! -\! \dfrac{\|\mathbf{q}-\hat{\mathbf{c}}_{k,j}\|}{R_{k,j}}\right)\right]^{q^d}\Biggr),
\end{multline*}
where $w(t)=\exp(-t/\tau^d)$ with $\tau^d \coloneqq \tau^d_{\text{ratio}} \cdot T_h$ is an exponential time weight that discounts predictions further into the future (since they are less reliable), $\tau^d_{\text{ratio}}$ is a scaling factor that sets the decay time constant as a fraction of the horizon, $\alpha^d_1$ scales the tube penalty, and $q^d$ controls how sharply the cost increases near the predicted path.
\rev{When $\gamma_k > 0$,} the tube radius $R_{k,j}$ grows with time while the time weight $w(t_j)$ decays: near-future predictions receive high weight over a small region, whereas far-future predictions receive low weight over a larger region.
This is a heuristic that balances spatial coverage against prediction reliability, and we found empirically that it produces effective obstacle avoidance paths across a wide range of environments.
The total per-obstacle heat combines both components: $H_k(\mathbf{q}) = H^{\mathrm{base}}_k(\mathbf{q}) + H^{\mathrm{tube}}_k(\mathbf{q})$, where $H^{\mathrm{base}}_k$ penalizes proximity to the obstacle's current position and $H^{\mathrm{tube}}_k$ penalizes proximity to its predicted future path.
The dynamic heat across all obstacles is $H^{d}(\mathbf{q}) = \max_k H_k(\mathbf{q})$, using max aggregation so that the most dangerous obstacle dominates the cost.

\subsection{Combined Heat Formulation}

\begin{figure}
    \centering
    \includegraphics[width=\columnwidth]{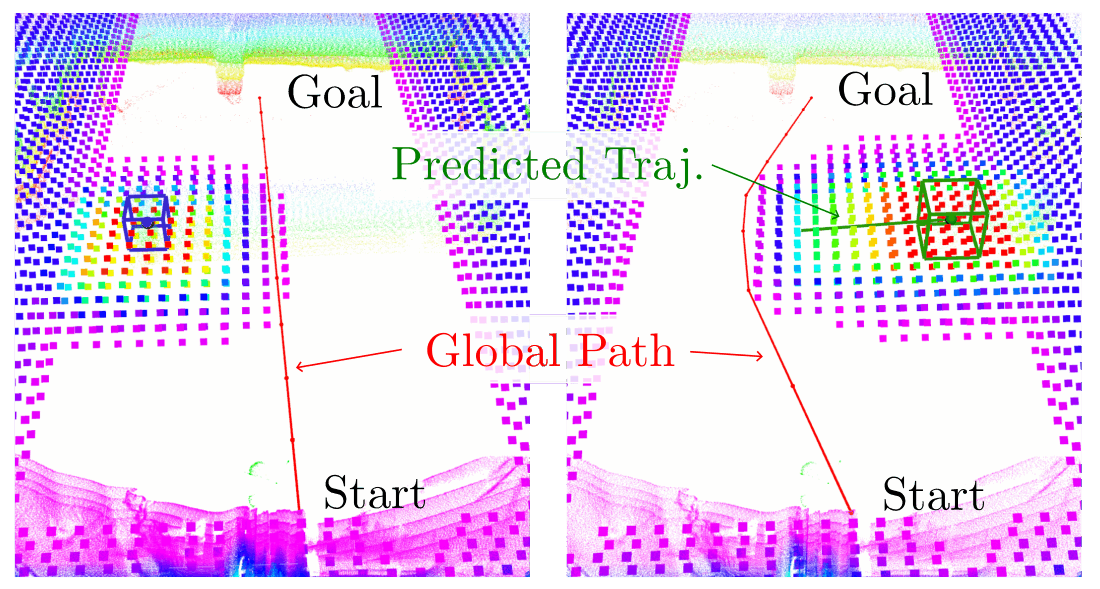}
    \caption{Visualization of the heat map-based global planner in hardware.
    From the wall on the side, the static heat decays with distance, while from the dynamic obstacle (blue and green boxes), the heat decays both spatially and temporally.
    On the left, the dynamic obstacle is not moving, so the heat is dominated by the base heat around its current position.
    On the right, the dynamic obstacle is moving toward the center of the room (the prediction is visualized as the green line), so the tube penalty creates a heat trail along the predicted path, steering the global path away from that region.
    \label{fig:heatmap_visualization}}
\end{figure}

Fig.~\ref{fig:heatmap_visualization} illustrates the combined heat map in a hardware experiment.
The combined heat map merges the static and dynamic components via max aggregation:
\[H(\mathbf{q}) \;=\; \min\!\left(\max\!\left(H^{s}(\mathbf{q}),\, H^{d}(\mathbf{q})\right),\; H_{\max}\right),\]
where $H_{\max}$ caps the total heat.
Max aggregation avoids double-penalization and ensures the highest risk dominates.

\subsection{A$^\ast$ with Hard Occupancy and Soft Cost Penalties}

Let $\mathcal{G} = (\mathcal{V}, \mathcal{E})$ denote the 26-connected voxel graph (\ie each voxel is connected to all $3^3 - 1 = 26$ neighbors in its $3 \times 3 \times 3$ neighborhood), where $\mathcal{V}$ is the set of free voxel centers and $\mathcal{E}$ contains edges between neighboring free voxels.
For an edge $(\mathbf{q}_i, \mathbf{q}_{i+1}) \in \mathcal{E}$, the edge cost is:
\[c(\mathbf{q}_i,\mathbf{q}_{i+1}) \;=\; d(\mathbf{q}_i,\mathbf{q}_{i+1}) \;+\; w_{\text{heat}}\, H(\mathbf{q}_{i+1}),\]
where $d(\mathbf{q}_i,\mathbf{q}_{i+1})$ is the Euclidean distance between neighboring voxel centers and $w_{\text{heat}} > 0$ is a tunable weight that balances the heat penalty against path length.
A$^\ast$ searches for the minimum-cost path from start $\mathbf{q}_s$ to goal $\mathbf{q}_g$ using the goal-directed heuristic $h(\mathbf{q}) = \|\mathbf{q} - \mathbf{q}_g\|$, yielding a global path in known-free space that avoids regions near both static and dynamic obstacles.

\paragraph{Heat map parameter selection}
The heat map parameters (Table~\ref{tab:system_parameters}) influence only the global path quality, not the safety guarantee, which is enforced by the hard STSFC constraints in the trajectory optimizer: a path that passes too close to obstacles may result in small corridors and reduced optimization feasibility, while overly aggressive heat penalties can produce unnecessarily long detours.
In practice, we found the same parameter set to be effective across all simulation environments without per-scenario tuning.
For the hardware experiments, we increased $w_{\text{heat}}$ from 5.0 to \rev{10.0 in the dynamic-obstacle experiments (Experiments~7--18, Table~\ref{tab:experiment_parameters}) to account for the smaller physical environments ($8 \times 8$\,m and $8 \times 20$\,m vs.\ $100 \times 40$\,m in simulation)}, where obstacles occupy a proportionally larger fraction of the space and stronger steering is needed.

\subsection{\rev{Post-processing of the A$^\ast$ output}}\label{subsec:global_path_postprocessing}

\rev{The raw waypoint sequence returned by A$^\ast$ is processed in three steps before it is handed to the STSFC generator (Section~\ref{subsec:temporal_safety_corridor}):
(i) collinear removal drops any waypoint that is collinear with its two neighbors,
(ii) cost-aware corner shortcutting replaces a waypoint $\mathbf{w}_i$ with a direct edge between $\mathbf{w}_{i-1}$ and $\mathbf{w}_{i+1}$ whenever the direct edge has lower edge cost, and (iii) resampling linearly subdivides any segment longer than $\Delta_{\mathrm{vtx}}$ so consecutive waypoints are at most $\Delta_{\mathrm{vtx}}$ apart.
The processed path is then truncated to its first $P+1$ waypoints, which define the $P$ segments used by the STSFC generator and set the final state of the local trajectory at the $(P+1)$-th waypoint (with zero velocity and acceleration).
The local planning horizon therefore spans at most $P \cdot \Delta_{\mathrm{vtx}}$ in path length per replanning cycle.}

\subsection{\rev{Approximation Limitations}}\label{subsec:global_planner_limitations}

\rev{Note that the heat map-based 3D planner is an approximation of the underlying spatiotemporal planning problem.
A 4D (spatiotemporal) global path search is a principled way to solve the problem, but the edge costs would require an estimate of ``how long it takes to reach $(x, y, z)$'' from the current state, which is itself a heuristic since the actual traversal time depends on the unsolved trajectory optimization.
Therefore, we use the 3D heat map-based planner, which trades construction-time guarantees for tractable real-time search.
Since there is no guarantee that a non-empty polytope $\mathcal{C}[n][p]$ exists for every $(n, p)$ due to the approximation, when MIQP infeasibility occurs, SANDO continues to execute the previously committed trajectory, which by construction satisfies all hard STSFC and dynamic constraints, and the planner retries on the next planning cycle.
If replanning keeps failing, the agent stops at the terminal stop state of the committed trajectory and continues to try replanning.
In dynamic environments, this stop state is guaranteed collision-free only within the committed planning horizon, as discussed in Section~\ref{subsec:recursive_feasibility}.
In practice, tuning the heat map weight $w_{\text{heat}}$ to push the global path further from obstacle-adjacent regions prevents such failures.}

\section{Trajectory Optimization}\label{sec:trajectory_optimization}

Building on the MIQP framework for static environments~\cite{deits2015miqp, mellinger2012miqp,tordesillas2022faster}, SANDO optimizes position trajectories using hard-constraint MIQP extended to dynamic settings with time-varying corridors.
Although MIQP introduces binary variables for piece-polytope assignment and increases computational cost relative to soft-constraint methods, it guarantees collision-free trajectories.
To reduce complexity, we leverage a variable elimination technique~\cite{nocedal2006numerical} that reduces the number of decision variables by symbolically solving the linear equality constraints, similar to the closed-form reductions in~\cite{richter2016polynomial} but extended to the MIQP setting (see Section~\ref{subsubsec:variable_elimination} for details).
\rev{SANDO also incorporates a parallelized time allocation strategy that launches multiple MIQP instances with different time allocations concurrently, and uses the first thread that returns a feasible solution (see Section~\ref{subsec:time_allocation} for details).}

\subsection{Trajectory Optimization}\label{subsec:trajectory_optimization_for_position}

We model the agent with triple-integrator dynamics and state vector $\mathbf{x}^T = \left[\vect{x}^T ~\vect{v}^T ~\vect{a}^T\right]$, where $\vect{x}$, $\vect{v}$, and $\vect{a}$ denote position, velocity, and acceleration, respectively.

We formulate trajectory optimization using an $N$-piece composite B\'ezier curve with $P$ spatial polytopes per time layer.
As described in Section~\ref{subsec:temporal_safety_corridor}, STSFCs in dynamic environments are time-layered, with polytopes indexed by both time layer and spatial location.
We reuse $n\in\{0:N-1\}$ and $p\in\{0{:}P-1\}$ from Section~\ref{subsec:temporal_safety_corridor}: since each trajectory piece $n$ executes during time layer $n$ of the STSFC, the same index identifies both the piece and its corresponding time layer. 
Similarly, $p$ indexes the spatial polytope within that layer. The time interval $dt$ per piece is uniform and matches the STSFC layer duration.
Fig.~\ref{fig:safe_corridor_in_dynamic_environments} illustrates this process: the MIQP assigns each trajectory piece to a spatial polytope within its time layer (\eg $\boldsymbol{x}_0 \subseteq \mathcal{C}[0][0]$, $\boldsymbol{x}_1 \subseteq \mathcal{C}[1][1]$), ensuring the trajectory remains in free space both spatially and temporally.

The control input, jerk, remains constant within each piece, allowing the position trajectory of each piece to be represented as a cubic polynomial.
Note that this implies jerk is discontinuous at piece boundaries; however, cubic polynomial representations are widely adopted in real-time planning~\cite{tordesillas2022faster,zhou2021ego-planner, zhou2021raptor}:
\begin{equation}\label{eq:cubic_spline}
    \boldsymbol{x}_{n}(\tau)=\boldsymbol{a}_{n}\tau^{3}+\boldsymbol{b}_{n}\tau^{2}+\boldsymbol{c}_{n}\tau+\boldsymbol{d}_{n},\; \; \tau\in[0,dt]
\end{equation}
where $\boldsymbol{a}_{n}, \boldsymbol{b}_{n}, \boldsymbol{c}_{n}, \boldsymbol{d}_{n} \in \mathbb{R}^3$ are the coefficients of the cubic spline in piece $n$.

We now discuss the constraints for the optimization formulation.
Continuity constraints are added between adjacent pieces to ensure the trajectory is continuous in position, velocity, and acceleration:
\begin{equation}\label{eq:continuity_constraints}
    \mathbf{x}_{n+1}(0)=\mathbf{x}_{n}(dt) \quad \text{for} \quad n\in\{0{:}N-2\}
\end{equation}
The B\'ezier curve control points $\boldsymbol{p}_{nj}$ $(j \in \{0{:}3\})$ associated with each piece $n$ are:
\begin{align}
	&\boldsymbol{p}_{n0}=\boldsymbol{d}_{n},
	\quad \boldsymbol{p}_{n1}=\frac{\boldsymbol{c}_{n}dt+3\boldsymbol{d}_{n}}{3} \nonumber\\
	&\boldsymbol{p}_{n2}=\frac{\boldsymbol{b}_{n}dt^{2}+2\boldsymbol{c}_{n}dt+3\boldsymbol{d}_{n}}{3} \\
	&\boldsymbol{p}_{n3}=\boldsymbol{a}_{n}dt^{3}+\boldsymbol{b}_{n}dt^{2}+\boldsymbol{c}_{n}dt+\boldsymbol{d}_{n} \nonumber
\end{align}
Since a B\'ezier curve lies within the convex hull of its control points~\cite{farin2002curves}, constraining all control points to lie inside a convex polytope guarantees that the entire piece remains inside that polytope.
To assign pieces to polytopes, we introduce binary variables $z_{np}$, where $z_{np}=1$ if piece $n$ is assigned to polytope $p$, and $z_{np}=0$ otherwise.
This condition is enforced through the following constraint:
\begin{align}
    & z_{np}=1\implies \mathbf{F}_{np}\boldsymbol{p}_{nj}\leq\mathbf{g}_{np}, \, \text{for} \quad j\in\{0{:}3\}, \, \forall n, \forall p
    \label{eq:polytope_constraints_MIQP}
\end{align}
where polytopes are time-layered as described in Section~\ref{subsec:temporal_safety_corridor}, with $(\mathbf{F}_{np}, \mathbf{g}_{np})$ denoting the polytope at time layer corresponding to piece $n$ and spatial index $p$.
In dynamic environments, the polytope constraints $\mathbf{F}_{np}$ and $\mathbf{g}_{np}$ vary with both the trajectory piece $n$ (time) and spatial polytope index $p$, reflecting the temporal evolution of free space as obstacles move.
\rev{Eq.~\eqref{eq:polytope_constraints_MIQP} is implemented through Gurobi's indicator-constraint interface.}
Each piece must be assigned to at least one polytope, which is ensured by the constraint:
\begin{align}
	& \sum_{p=0}^{P-1}z_{np}\ge1, \quad \forall n
    \label{eq:polytope_assignment_constraints}
\end{align}
To ensure the trajectory starts at the initial state and ends at the final state, we impose the following constraints:
\begin{align} \label{eq:initial_and_final_state}
    \mathbf{x}_{0}(0) &= \mathbf{x}_{\text{init}}, \qquad \mathbf{x}_{N-1}(dt) = \mathbf{x}_{\text{final}}
\end{align}
where $\mathbf{x}_{\text{init}}$ is the initial state, and $\mathbf{x}_{\text{final}}$ is the final state.
The final state is set to the $(P{+}1)$-th waypoint on the global path with zero velocity and zero acceleration, so that the trajectory spans exactly $P$ path segments and stops at the $(P{+}1)$-th waypoint.
Since SANDO replans in a receding-horizon manner, only the initial portion of the trajectory is typically executed before a new plan is computed; the agent does not necessarily reach the final stop state of each optimized trajectory.
\rev{Each piece is a continuous-time cubic B\'ezier of duration $dt = f\,dt_0$ set by the time allocation (Section~\ref{subsec:time_allocation}), with the dynamic and corridor constraints enforced via the convex-hull property rather than at discrete sample points. 
The $N$ pieces therefore form a continuous-time horizon $T = N\,dt$.}

For dynamic constraints, we define the velocity control points $\boldsymbol{v}_{nj}$ ($j \in \{0{:}2\}$), acceleration control points $\boldsymbol{a}_{nj}$ ($j \in \{0{:}1\}$), and jerk $\boldsymbol{j}_{n}$ for each piece $n$.
By the convex hull property of B\'ezier curves, bounding these control points guarantees that the continuous velocity, acceleration, and jerk remain within limits throughout each piece:
\begin{align}
	\left\Vert \boldsymbol{v}_{nj} \right\Vert_{\infty} &\le v_{\max}, \quad j \in \{0{:}2\} \nonumber\\
	\left\Vert \boldsymbol{a}_{nj} \right\Vert_{\infty} &\le a_{\max}, \quad j \in \{0{:}1\} \label{eq:velocity_accel_jerk_constraints} \\
	\left\Vert \boldsymbol{j}_{n} \right\Vert_{\infty} &\le j_{\max} \nonumber
\end{align}
where $v_{\max}$, $a_{\max}$, and $j_{\max}$ denote the maximum allowable velocity, acceleration, and jerk, respectively.
The objective function penalizes the squared jerk along the trajectory for smooth motion:
\[
    J = \sum_{n=0}^{N-1} \left\| \boldsymbol{j}_n \right\|^2.
\]

The complete MIQP problem is then formulated as:
\begin{equation}\label{eq:MIQP_formulation}
\begin{split}
    \min_{\boldsymbol{a}_n, \boldsymbol{b}_n, \boldsymbol{c}_n, \boldsymbol{d}_n, z_{np}} \quad & J \\
    \text{s.t.} \quad \quad \quad \ & \text{Eqs.~\eqref{eq:continuity_constraints}, \eqref{eq:polytope_constraints_MIQP}, \eqref{eq:polytope_assignment_constraints}, \eqref{eq:initial_and_final_state}, and \eqref{eq:velocity_accel_jerk_constraints}}
\end{split}
\end{equation}
\rev{Each per-thread MIQP is solved with a time limit (\SI{0.5}{s} in simulation and \SI{1}{s} on hardware) and Gurobi's default relative MIP gap of $10^{-4}$, and we do not warm-start between replanning cycles since the corridor polytopes are rebuilt every cycle and the previous solution's binary assignment is not in general a reliable initial guess for the new MIQP.}

\subsubsection{Variable Elimination}\label{subsubsec:variable_elimination}

In the MIQP formulation of Eq.~\eqref{eq:MIQP_formulation}, each piece $n$ introduces four coefficient vectors ($\boldsymbol{a}_n$, $\boldsymbol{b}_n$, $\boldsymbol{c}_n$, $\boldsymbol{d}_n$), giving $4N$ decision variables per axis ($x$, $y$, $z$), \ie $12N$ in 3D.
However, the continuity constraints (Eq.~\eqref{eq:continuity_constraints}) and boundary conditions (Eq.~\eqref{eq:initial_and_final_state}) impose $3N+3$ equality constraints per axis: $3(N{-}1)$ from position, velocity, and acceleration continuity at the $N{-}1$ interior piece boundaries, plus $6$ from the initial and final boundary conditions (position, velocity, and acceleration at start and end).
By symbolically solving these equality constraints, we can express most coefficients as affine functions of a small set of remaining variables, which (1) reduces the number of decision variables to $N-3$ per axis and (2) removes all equality constraints from the optimization.

For instance, with $N = 4$ pieces there are $4N = 16$ variables and $3N+3 = 15$ equality constraints per axis, leaving only one decision variable per axis.
Symbolic elimination reveals this remaining variable to be $\boldsymbol{d}_3$ (the first control point of the final piece); all other coefficients become affine functions of $\boldsymbol{d}_3$.
In general, the number of remaining decision variables per axis is $4N - (3N+3) = N-3$; for example, $N=5$ yields 2 and $N=6$ yields 3 per axis, with the same elimination procedure applied.
However, the symbolic expressions grow combinatorially with $N$: each remaining variable's affine coefficients depend on all boundary conditions and continuity relations, so the closed-form expressions become prohibitively large for \rev{$N \ge 7$}.
In our implementation, we precompute the symbolic elimination offline for \rev{$N \in \{4, 5, 6\}$}, which covers the operating range used in all experiments.

This leads to a revised MIQP with many fewer decision variables and no equality constraints:
\[\begin{aligned}
    \min_{\boldsymbol{d}_3,\, z_{np}} \quad & J \\
    \text{s.t.} \quad \,\, & \text{Eqs.~\eqref{eq:polytope_constraints_MIQP}, \eqref{eq:polytope_assignment_constraints}, and \eqref{eq:velocity_accel_jerk_constraints}}.
\end{aligned}\]
This MIQP is solved using Gurobi~\cite{gurobi}.
Section~\ref{sec:variable-elimination-benchmarking} performs an ablation study to evaluate the computational benefits of this variable elimination technique, and it demonstrates up to $7.4\times$ reduction in optimization time.

\subsubsection{Time Allocation and Parallelization}\label{subsec:time_allocation}

\rev{The time allocated to each trajectory piece affects feasibility and optimality, and its range is bounded on both ends.
At the lower end, following~\cite{tordesillas2022faster}, SANDO computes a baseline piece duration $dt_0$ from the per-axis minimum-time solutions under the robot's dynamic limits on velocity, acceleration, and jerk, and scales it by a factor $f \ge 1$ to obtain $dt = f \cdot dt_0$, so $f = 1$ corresponds to the tightest dynamically feasible $dt$.
At the upper end, larger $f$ enlarges every per-layer reachable set $\mathcal{R}_n$ and shrinks the polytopes $\mathcal{C}[n][p]$, leading to more infeasibility, and very large $f$ values rarely return high-quality trajectories.
SANDO therefore searches over a bounded factor range $f \in [1, f_{\max}]$, with $f_{\max}$ configured per environment (Table~\ref{tab:experiment_parameters}).}

To search over time allocations efficiently, SANDO maintains a sliding window of $M$ candidate factors $\{f_1, \ldots, f_M\}$ spanning a range of width $2\kappa$ centered on a mean value, with uniform step size $\Delta f$, where $M = \lfloor 2\kappa / \Delta f \rfloor + 1$.
At each replanning iteration, $M$ MIQP solver threads are launched in parallel, one per factor, each with $dt = f_i \cdot dt_0$.
Since each factor yields a different $dt$, a separate STSFC is generated per thread to reflect the corresponding obstacle inflation; the per-thread STSFC computation cost is reported in Sections~\ref{sec:simulation-results} and~\ref{sec:hardware-experiments}.
As soon as any thread finds a feasible solution, the remaining threads are terminated and that solution is used.

The factor window adapts for the next replanning cycle based on the outcome:
\begin{itemize}
    \item \textbf{Success:} The window is recentered so that the successful factor becomes the median, biasing the next iteration toward similar time allocations.
    \item \textbf{All fail:} The window is shifted upward by one step $\Delta f$, increasing the time allocation to improve feasibility. If the window reaches a configurable upper bound $f_{\max}$ (see Table~\ref{tab:experiment_parameters}), it resets to the initial position.
\end{itemize}

\section{Dynamic Obstacle Detection and Tracking}\label{sec:obstacle_tracking}

SANDO detects and tracks dynamic obstacles from raw point cloud data through a multi-stage pipeline.

\subsection{Detection via Temporal Occupancy Grid}
Inspired by Dynablox~\cite{schmid2023dynablox}, SANDO constructs a temporal occupancy grid to classify voxels as static or dynamic.
Voxels that remain occupied beyond a configurable duration are classified as static.
When a voxel transitions from free to occupied and has fewer than a threshold number of static neighbors, it is classified as dynamic, indicating a newly appearing moving object rather than part of an existing static structure.
Dynamic labels persist for a configurable duration to maintain temporal continuity during brief sensor occlusions.
False positives are mitigated by the static-neighbor threshold: voxels adjacent to many static voxels are not classified as dynamic, preventing edges of static structures from being misidentified as moving objects.

\subsection{Clustering and Data Association}
Dynamic voxels are grouped into clusters using Euclidean clustering.
For each cluster, the centroid and AABB half-extents are computed \revtwo{from the extent of the observed points.
The half-extents are then smoothed over time with an exponential filter, which reduces the effect of frame-to-frame variation in the visible surface.
Onboard sensors usually observe only one face of an obstacle, so these half-extents can underestimate the true extent along directions that are never observed.
Assumption~\ref{ass:extent_error} (Section~\ref{subsec:theoretical_safety_guarantees}) states the condition under which Theorem~\ref{thm:safety} still holds, and Section~\ref{subsec:assumption_violations} describes what happens when it does not}.
Clusters are associated with existing tracks via nearest-neighbor matching within a distance threshold; unmatched clusters initialize new tracks.
Nearest-neighbor association can produce incorrect matches when obstacle trajectories cross or obstacles are closely spaced; however, the tracker is a modular component and can be replaced with more sophisticated methods (\eg the Hungarian algorithm) without changing the planning framework.

\subsection{Adaptive Extended Kalman Filter (AEKF)}
Each tracked obstacle $k$ is modeled with a 9-state vector $\mathbf{x}_k = [\mathbf{c}_k^\top,\, \mathbf{v}_k^\top,\, \mathbf{a}_k^\top]^\top$, where $\mathbf{c}_k$, $\mathbf{v}_k$, and $\mathbf{a}_k$ denote the obstacle's centroid position, velocity, and acceleration, using a constant-acceleration process model.
The measurement is the cluster centroid (position only).
To handle unknown and time-varying noise characteristics, we employ an AEKF~\cite{akhlaghi2017adaptive} that continuously updates both the measurement noise covariance $\mathbf{R}_i$ and the process noise covariance $\mathbf{Q}_i$ at each filter step $i$ using exponential forgetting with factor $\alpha$:
\begin{eqnarray}
    \mathbf{R}_i &=& \alpha \mathbf{R}_{i-1} + (1-\alpha)\,\mathbf{r}_i \mathbf{r}_i^\top, \\
    \mathbf{Q}_i &=& \alpha \mathbf{Q}_{i-1} + (1-\alpha)\,\mathbf{K}_i \mathbf{d}_i \mathbf{d}_i^\top \mathbf{K}_i^\top,
\end{eqnarray}
where $\mathbf{r}_i$ is the measurement residual, $\mathbf{d}_i$ is the innovation, and $\mathbf{K}_i$ is the Kalman gain.
When a new obstacle is detected, its initial $\mathbf{Q}_0$ and $\mathbf{R}_0$ are set to the average of the covariances from existing tracks, allowing the filter to use prior knowledge of the environment's noise characteristics.

\subsection{Prediction and Output}
Future positions are predicted using a constant-velocity model with the AEKF-estimated velocity.
\rev{The tracker receives only the centroid position of each clustered dynamic-obstacle point cloud.
For estimation, we use a constant-acceleration model because the acceleration state captures realistic maneuvers, such as a person stopping or turning, which a constant-velocity model cannot capture.
For prediction, however, we use the constant-velocity model $\mathbf{c}(t) = \mathbf{c}_0 + \mathbf{v}_0\,t$.
Prediction extrapolates the estimated state with no measurement corrections, and a constant-acceleration model integrates the acceleration estimate twice, so an error $\delta a$ between the estimated and ground truth acceleration produces a predicted-position error $\tfrac{1}{2}\,(\delta a)\,t^2$ that grows quadratically with the horizon, versus $(\delta v)\,t$ for a velocity error.
Because the acceleration estimate is estimated from the second derivative of the position measurements, the acceleration error, $\delta a$, can be large. 
As a result, this quadratic term dominates the prediction error, whereas the constant velocity model does not include it.}
For each tracked obstacle, the system publishes the predicted trajectory and AABB half-extents.
Tracks that are not updated for a configurable timeout are deleted, allowing the system to discard obstacles that have left the sensor's field of view or stopped moving.
The predicted trajectories and reachable sets are then used by the heat map-based global planner (Section~\ref{subsec:global_planner_heatmap}) and STSFC generator (Section~\ref{subsec:temporal_safety_corridor}).

\section{Safety Analysis}\label{subsec:theoretical_safety_guarantees}

SANDO provides formal collision-free guarantees through its STSFC framework.
\rev{Using the notation introduced in Section~\ref{subsec:temporal_safety_corridor}, the safety argument is parameterized by a per-layer reachable set $\mathcal{R}_n \subset \mathbb{R}^3$ that upper-bounds the obstacle's displacement from $t_0$ to $t_0+(n{+}1)\,dt$. The shape of $\mathcal{R}_n$ depends on the bound chosen in Assumption~\ref{ass:bounded_vel}; the safety theorem (Theorem~\ref{thm:safety}) and its proof are independent of that choice.}

\subsection{Assumptions}\label{subsec:assumptions}

\begin{assumption}[Bounded obstacle velocity]\label{ass:bounded_vel}
    \rev{Let $\mathbf{c}_k(t)$ denote the centroid of tracked dynamic obstacle $k \in \{1,\dots,K\}$, as defined in Section~\ref{subsec:temporal_safety_corridor}. Its centroid velocity $\dot{\mathbf{c}}_k(t)$ is bounded under either of the following two formulations. 
    Theorem~\ref{thm:safety} below holds under either one, with the per-layer reachable set $\mathcal{R}_n$ derived from the chosen bound.
    \textbf{(A) Euclidean bound:} there exists a known constant $v^{\mathrm{obs}}_{\max} > 0$ such that $\left\|\dot{\mathbf{c}}_k(t)\right\|_{2} \le v^{\mathrm{obs}}_{\max}$ for all $k$ and all $t \ge 0$, inducing the Euclidean ball $\mathcal{R}_n := B_2\bigl(v^{\mathrm{obs}}_{\max}\,(n{+}1)\,dt\bigr)$.
    \textbf{(B) Per-axis bound:} there exist known constants $v^{\mathrm{obs}}_{\max,i} > 0$ such that $\bigl|\dot c_k^{\,i}(t)\bigr| \le v^{\mathrm{obs}}_{\max,i}$ for $i\in\{x,y,z\}$, all $k$, and all $t \ge 0$, inducing the axis-aligned box $\mathcal{R}_n := \bigl\{ \boldsymbol{\delta}\in\mathbb{R}^3 : |\delta_i| \le v^{\mathrm{obs}}_{\max,i}\,(n{+}1)\,dt \bigr\}$.
    Either bound, applied as $\mathcal{R}_n$, contains the obstacle's true displacement, so the choice between (A) and (B) is a design decision that holds under this assumption either way.}
\end{assumption}

\rev{SANDO selects between the two formulations in three places. Each choice is practical, and because the selected $\mathcal{R}_n$ still upper-bounds the true reachable set in every case, none violates Assumption~\ref{ass:bounded_vel} or compromises the collision-free guarantee of Theorem~\ref{thm:safety}:
\begin{itemize}
    \item \textbf{Tracked-obstacle inflation: always (B).} The estimated obstacle is an AABB, and growing it by per-axis amounts keeps it an AABB, which is an efficient representation for corridor generation; (B) is also tighter than (A) when the per-axis peaks are known.
    \item \textbf{Unknown-space inflation: either, set per experiment} to match the worst-case motion of obstacles that may emerge from unobserved space. 
    In simulation the trefoil-knot obstacles have well-characterized per-axis peaks, so we use (B). 
    On the other hand, on hardware the wheeled platform moves at $\approx\SI{0.5}{m/s}$ with no preferred axis, so we use (A).
    \item \textbf{Walk-back subgoal test (Section~\ref{subsubsec:unknown_space_inflation}): always Euclidean.} The unknown voxels are stored in a KD-tree that uses Euclidean queries, so the test uses an L2 ball (Formulation~(A)), which conservatively encloses the per-axis box when (B) is in force.
\end{itemize}}

\rev{Throughout this section, $\hat{\mathcal{O}}_k^n = \hat{\mathcal{O}}_k(t_0) \oplus \mathcal{R}_n \oplus C_\infty(\epsilon) \revtwo{\oplus \mathcal{D}}$ denotes the inflated tracked-obstacle region for time layer $n$ (Eq.~\eqref{eq:safety_rn}), where $C_\infty(\epsilon)$ is the axis-aligned cube of half-side $\epsilon$ accounting for position estimation error (Assumption~\ref{ass:estimation_error} below) \revtwo{and $\mathcal{D}$ is the vehicle's own axis-aligned bounding box (Section~\ref{subsec:obstacle_inflation})}, and $\hat{\mathcal{U}}^n = \mathcal{U}(t_0) \cup \bigl(\mathcal{B}_{\mathcal{U}}(t_0) \oplus \mathcal{R}_n^\epsilon \revtwo{\oplus B_2(r_{\text{drone}})}\bigr)$ denotes the inflated unknown region (Eq.~\eqref{eq:unknown_inflation}); the $\epsilon$ buffer in $\mathcal{R}_n^\epsilon \supseteq \mathcal{R}_n$ only enlarges $\hat{\mathcal{U}}^n$, so the containment argument in Theorem~\ref{thm:safety} below is unaffected.
\revtwo{We write $\mathcal{A}(\vect{x}) \coloneqq \{\vect{x}\} \oplus \mathcal{D}$ for the region the vehicle occupies when its center is at $\vect{x}$.
Both inflated regions account for the vehicle's extent (Section~\ref{subsec:obstacle_inflation}): $\hat{\mathcal{O}}_k^n$ through the box $\mathcal{D}$ and $\hat{\mathcal{U}}^n$ through the ball $B_2(r_{\text{drone}})$.
The guarantee below is therefore stated for the vehicle body rather than for a point mass.}}

\begin{assumption}[Bounded position estimation error]\label{ass:estimation_error}
    The obstacle tracker provides an estimated position $\hat{\mathbf{c}}_k(t_0)$ of each obstacle at the planning time $t_0$.
    There exists a known constant $\epsilon \ge 0$ such that:
    \[
        \left\|\mathbf{c}_k(t_0) - \hat{\mathbf{c}}_k(t_0)\right\|_\infty \le \epsilon, \quad \forall\, k.
    \]
    That is, the per-axis estimation error is bounded by $\epsilon$.
\end{assumption}

\begin{assumption}[\revtwo{Bounded obstacle size}]\label{ass:extent_error}
    \revtwo{The obstacle tracker also provides estimated AABB half-extents $\hat{\mathbf{h}}_k$ for each obstacle at the planning time $t_0$.
    These are no smaller than the true half-extents $\mathbf{h}_k$ of Section~\ref{subsec:temporal_safety_corridor} along every axis:
    \[
        \hat{h}_k^{\,i} \;\ge\; h_k^{\,i}, \quad i \in \{x,y,z\}, \;\; \forall\, k .
    \]
    Equivalently, the true obstacle region, translated so that its centroid coincides with the estimated one, is contained in the estimated AABB $\hat{\mathcal{O}}_k(t_0)$.}
\end{assumption}

\begin{assumption}[Perfect trajectory tracking]\label{ass:perfect_tracking}
    The low-level controller tracks the planned trajectory exactly, \ie the executed position coincides with the planned position at all times.
\end{assumption}

In practice, Assumptions~\ref{ass:estimation_error}\revtwo{--\ref{ass:extent_error}} and~\ref{ass:perfect_tracking} are not satisfied exactly.
\rev{However, the estimation-error term $C_\infty(\epsilon)$ in the inflation (Eq.~\eqref{eq:safety_rn}; see Table~\ref{tab:experiment_parameters}), which enlarges each obstacle's AABB half-extents by $\epsilon$, provides a buffer that absorbs both position estimation errors and controller tracking errors: as long as $\epsilon$ exceeds the combined worst-case estimation and tracking error, the inflated region $\hat{\mathcal{O}}_k^n$ still contains the true obstacle and the safety guarantee holds.
In simulation, obstacle positions are obtained with negligible error, so we set $\epsilon = 0$; in hardware, we set $\epsilon = 0.3$\,m to absorb the larger estimation and tracking errors observed on the physical platform.}

\begin{assumption}[Untracked obstacles in unknown space]\label{ass:untracked_in_unknown}
    Let $\mathcal{U}(t_0) \subset \mathbb{R}^3$ denote the set of unobserved (unknown) voxels at the planning time $t_0$.
    Any dynamic obstacle that is not currently tracked by the obstacle tracker is located entirely within $\mathcal{U}(t_0)$ at time $t_0$.
\end{assumption}

\subsection{Corridor Construction}\label{subsubsec:corridor_construction}

\rev{As described in Section~\ref{subsec:temporal_safety_corridor}, each dynamic obstacle $k$ is inflated by the reachable set $\mathcal{R}_n$ to obtain $\hat{\mathcal{O}}_k^n$ (Eq.~\eqref{eq:safety_rn}), and when unknown-space inflation is enabled, the unknown region boundary is inflated by the same reachable set, with the $\epsilon$ buffer applied in the matching norm ($\mathcal{R}_n^\epsilon$), to obtain $\hat{\mathcal{U}}^n$ (Eq.~\eqref{eq:unknown_inflation}).
Each polytope $\mathcal{C}[n][p]$ is generated to be free of all inflated tracked dynamic obstacles and, when unknown-space inflation is enabled, the inflated unknown region:}
\begin{eqnarray}
    \mathcal{C}[n][p] \cap \hat{\mathcal{O}}_k^n &=& \emptyset, \quad \forall\, k,\; \forall\, n,\; \forall\, p, \label{eq:safety_corridor_free} \\
    \mathcal{C}[n][p] \cap \hat{\mathcal{U}}^n &=& \emptyset, \quad \forall\, n,\; \forall\, p. \label{eq:safety_corridor_free_unknown}
\end{eqnarray}

\paragraph{MIQP polytope assignment}
As described in Section~\ref{sec:trajectory_optimization}, the MIQP assigns each piece $n$ to a polytope $\mathcal{C}[n][p_n^*]$ such that all four control points lie inside (Eq.~\eqref{eq:polytope_constraints_MIQP}).
By the convex hull property of B\'ezier curves~\cite{farin2002curves}, the entire piece therefore remains inside $\mathcal{C}[n][p_n^*]$.

\subsection{Safety}

\begin{theorem}[Safety guarantee]\label{thm:safety}
    Under Assumptions~\ref{ass:bounded_vel}--\ref{ass:perfect_tracking} and the corridor construction in Section~\ref{subsubsec:corridor_construction}, the trajectory $\vect{x}(t)$ produced by the MIQP solver is collision-free with respect to all tracked dynamic obstacles for the duration of the trajectory.
    Furthermore, when unknown-space inflation is enabled (Eq.~\eqref{eq:safety_corridor_free_unknown}) and Assumption~\ref{ass:untracked_in_unknown} holds, the trajectory is also collision-free with respect to all untracked dynamic obstacles.
    \revtwo{The guarantee holds for the vehicle body rather than only for its center, where $\mathcal{A}(\vect{x}(t)) = \{\vect{x}(t)\} \oplus \mathcal{D}$ for tracked obstacles and $\mathcal{A}(\vect{x}(t)) = \{\vect{x}(t)\} \oplus B_2(r_{\text{drone}})$ for untracked obstacles (Section~\ref{subsec:obstacle_inflation}).}
    That is:
    \[
        \revtwo{\mathcal{A}(\vect{x}(t)) \cap \mathcal{O}_k(t) = \emptyset}, \quad \forall\, t \in [t_0,\; t_0 + N \cdot dt],\; \forall\, k,
    \]
    where $k$ ranges over all tracked dynamic obstacles, and additionally over all untracked dynamic obstacles when unknown-space inflation is enabled.
\end{theorem}

\begin{proof}
    Consider an arbitrary trajectory piece $n$ executing over $[t_0 + n \cdot dt,\; t_0 + (n{+}1) \cdot dt]$ and any time $\bar{t}$ within that interval.
    Let $p_n^*$ be the spatial polytope assigned to piece $n$ by the MIQP (Eq.~\eqref{eq:polytope_constraints_MIQP}).
    We proceed in three steps: first, we show the trajectory stays inside its assigned corridor polytope; second, we show the true obstacle remains inside its inflated region; third, we show the corridor and inflated region are disjoint, which together imply no collision.

    \textit{(i) Trajectory is inside a corridor polytope.}
    By Eq.~\eqref{eq:polytope_constraints_MIQP}, all four control points of piece $n$ lie inside $\mathcal{C}[n][p_n^*]$.
    Since a B\'ezier curve is contained within the convex hull of its control points~\cite{farin2002curves}, piece $n$ remains inside $\mathcal{C}[n][p_n^*]$ throughout its time interval, so $\vect{x}(\bar{t}) \in \mathcal{C}[n][p_n^*]$.

    \textit{(ii) The true obstacle is inside the inflated region.}
    \rev{For tracked dynamic obstacles, the region $\mathcal{O}_k(t)$ stays within the fixed half-extents $\mathbf{h}_k$ around its centroid at all times (Section~\ref{subsec:temporal_safety_corridor}), so bounding the centroid displacement bounds the displacement of the entire region.
    From $t_0$ to $\bar t$ this displacement $\mathbf{c}_k(\bar t) - \mathbf{c}_k(t_0)$ lies in the reachable set $\mathcal{R}_n$ by Assumption~\ref{ass:bounded_vel} (the shape of $\mathcal{R}_n$ depending on whether Formulation~(A) or~(B) is used).
    \revtwo{We write $\mathcal{H}_k$ and $\hat{\mathcal{H}}_k$ for the boxes of half-extents $\mathbf{h}_k$ and $\hat{\mathbf{h}}_k$.
    Each assumption then contributes one inclusion.
    \[
    \begin{aligned}
        \mathcal{O}_k(\bar t)
          &\subseteq \{\mathbf{c}_k(\bar t)\} \oplus \mathcal{H}_k \\
          &\subseteq \{\mathbf{c}_k(t_0)\} \oplus \mathcal{R}_n \oplus \mathcal{H}_k \\
          &\subseteq \{\hat{\mathbf{c}}_k(t_0)\} \oplus C_\infty(\epsilon) \oplus \mathcal{R}_n \oplus \mathcal{H}_k \\
          &\subseteq \hat{\mathcal{O}}_k(t_0) \oplus C_\infty(\epsilon) \oplus \mathcal{R}_n ,
    \end{aligned}
    \]
    by the AABB model of Section~\ref{subsec:temporal_safety_corridor}, then Assumption~\ref{ass:bounded_vel}, then Assumption~\ref{ass:estimation_error}, and finally Assumption~\ref{ass:extent_error} with $\hat{\mathcal{O}}_k(t_0) = \{\hat{\mathbf{c}}_k(t_0)\} \oplus \hat{\mathcal{H}}_k$.
    The chain gives what the proof needs.
    The true obstacle region at any time $\bar t$ in layer $n$ lies inside the estimated AABB, inflated by the estimation buffer $C_\infty(\epsilon)$ and the reachable set $\mathcal{R}_n$.
    The step labeled Assumption~\ref{ass:extent_error} is where the size estimate enters, since it replaces the true half-extents $\mathcal{H}_k$ by the estimated AABB $\hat{\mathcal{O}}_k(t_0)$.
    Without $\mathcal{H}_k \subseteq \hat{\mathcal{H}}_k$ the inflated region need not contain the obstacle (Section~\ref{subsec:assumption_violations}).
    Assumptions~\ref{ass:estimation_error} and~\ref{ass:extent_error} are sufficient for this containment, but they are not necessary.
    Written per axis, the two steps that use them need only
    \[
        \bigl|c_k^{\,i}(t_0) - \hat{c}_k^{\,i}(t_0)\bigr| + h_k^{\,i} \;\le\; \hat{h}_k^{\,i} + \epsilon,
        \quad i \in \{x,y,z\}.
    \]
    The two assumptions give this immediately.
    The inequality can also hold when one of them fails, so the two errors trade off.
    A centroid error larger than $\epsilon$ is tolerated when the estimated extent exceeds the true extent by at least the excess, and a size underestimate is tolerated when the centroid error is correspondingly smaller.
    Adding the vehicle extent $\mathcal{D}$ to both sides and comparing with Eq.~\eqref{eq:safety_rn} gives the configuration-space containment used below,
    \[
        \mathcal{O}_k(\bar t) \oplus \mathcal{D} \;\subseteq\; \hat{\mathcal{O}}_k^n .
    \]}
    The argument never uses the specific shape of $\mathcal{R}_n$, so the proof is identical under either formulation of Assumption~\ref{ass:bounded_vel}.}
    For untracked dynamic obstacles (when unknown-space inflation is enabled), by Assumption~\ref{ass:untracked_in_unknown} the obstacle starts in $\mathcal{U}(t_0)$, and \rev{by Assumption~\ref{ass:bounded_vel} the displacement of its centroid, and hence of its entire region, during time layer $n$ lies in $\mathcal{R}_n$}, so \revtwo{$\mathcal{O}_k(\bar{t}) \oplus B_2(r_{\text{drone}}) \subseteq \hat{\mathcal{U}}^n$} (Eq.~\eqref{eq:unknown_inflation})\revtwo{, where the ball term comes from the inflation of the unknown boundary}.

    \textit{(iii) Corridor and inflated regions are disjoint.}
    By construction, the corridor polytope is disjoint from all inflated tracked dynamic obstacles (Eq.~\eqref{eq:safety_corridor_free}): ${\mathcal{C}[n][p_n^*] \cap \hat{\mathcal{O}}_k^n = \emptyset}$.
    When unknown-space inflation is enabled, it is also disjoint from the inflated unknown region (Eq.~\eqref{eq:safety_corridor_free_unknown}): ${\mathcal{C}[n][p_n^*] \cap \hat{\mathcal{U}}^n = \emptyset}$.

    Combining (i)--(iii), the trajectory lies inside the corridor, the obstacle lies inside its inflated region, and the two are disjoint.
    Therefore $\vect{x}(\bar{t}) \notin \hat{\mathcal{O}}_k^n$, and \revtwo{$\vect{x}(\bar{t}) \notin \mathcal{O}_k(\bar{t}) \oplus \mathcal{D}$ by step~(ii).
    Because $\mathcal{D}$ is symmetric about the vehicle's center (Section~\ref{subsec:obstacle_inflation}), this gives
    \[
        \mathcal{A}(\vect{x}(\bar{t})) \cap \mathcal{O}_k(\bar{t}) = \emptyset .
    \]
    The vehicle body, and not only its center, therefore stays clear of the obstacle.
    For untracked obstacles the same argument applies with $\mathcal{D}$ replaced by the ball $B_2(r_{\text{drone}})$, which is also symmetric about the vehicle's center.}
    Since $\bar{t}$ and $k$ were arbitrary, the trajectory is collision-free with respect to all dynamic obstacles at all times.
    Safety with respect to static obstacles follows \revtwo{the same way}: the corridor polytopes are constructed \revtwo{so that every point of them is at least $r_{\text{drone}}$ from every static obstacle (Section~\ref{subsec:obstacle_inflation}), so the vehicle, under the same ball model,} cannot intersect any static obstacle either.
    \revtwo{The vehicle extent therefore enters the guarantee for all three obstacle classes.
    It enters as the $\mathcal{D}$ term of Eq.~\eqref{eq:safety_rn} for tracked dynamic obstacles, and as the decomposition's inflation distance for static obstacles and the unknown region (Section~\ref{subsec:obstacle_inflation}).}
\end{proof}

\rev{The reachable set $\mathcal{R}_n$ scales linearly with the per-piece duration $dt$ and with the obstacle velocity bound chosen in Assumption~\ref{ass:bounded_vel}.}
As described in Section~\ref{subsec:time_allocation}, $dt = f \cdot dt_0$, where $dt_0$ is the minimum feasible time per piece derived from the agent's dynamic constraints ($v_{\max}$, $a_{\max}$, $j_{\max}$) and $f \ge 1$ is the time allocation factor.
\rev{Hence, higher $v_{\max}$, $a_{\max}$, $j_{\max}$ reduce $dt_0$ and consequently $dt$, which shrinks $\mathcal{R}_n$ and produces larger corridors since the agent traverses each piece faster and obstacles have less time to move within it.
Conversely, tighter agent constraints enlarge $dt_0$ and hence $\mathcal{R}_n$, and a higher obstacle velocity bound enlarges $\mathcal{R}_n$ directly; both require more conservative corridors and potentially reduce feasibility in dense environments.}

\subsection{When Assumptions Are Violated}\label{subsec:assumption_violations}

The safety guarantee of Theorem~\ref{thm:safety} depends on both assumptions and the corridor construction properties.

\textbf{Velocity bound violated:}
If an obstacle's speed exceeds the velocity bound, its true position at time $t$ can lie outside the inflated region $\hat{\mathcal{O}}_k^n$.
The corridor remains obstacle-free with respect to the inflated region, but the true obstacle may have moved into the corridor.
As the obstacle velocity bound increases, the per-layer inflation grows linearly, shrinking the STSFC corridors and eventually making trajectory optimization infeasible.
For very fast obstacles, the planner would need to rely more heavily on trajectory prediction to tighten the inflation (inflating around the predicted future position rather than the worst-case reachable set from the current position), which is an avenue for future work.

\textbf{Estimation error exceeded:}
\rev{The buffer available to absorb estimation error is the term $\epsilon$.
In our experiments, $\epsilon = 0$ in simulation and $\epsilon = 0.3$\,m in hardware, so safety is maintained as long as the per-axis estimation error does not exceed $\epsilon$.
If the actual error exceeds this buffer, the inflated region may not fully contain the true obstacle, and collisions become possible.
Users can increase $\epsilon$ to accommodate larger estimation errors at the cost of more conservative corridors.}

\revtwo{\textbf{Size bound violated:}
We estimate the AABB half-extents using onboard sensors that might see only one face of an obstacle.
For a compact object the unobserved depth is similar to the observed width, but for an elongated object viewed end-on the true obstacle can reach well beyond the estimated box, and Assumption~\ref{ass:extent_error} fails.
The inflated region $\hat{\mathcal{O}}_k^n$ then might not contain the whole body of the true obstacle, and a collision can still occur.
The only margin available to absorb an underestimate is $\epsilon$, so safety is retained as long as the per-axis size underestimate and the centroid estimation error together stay within $\epsilon$.}

\revtwo{\textbf{Obstacle not yet tracked as dynamic:}
Theorem~\ref{thm:safety} covers obstacle $k$ once a trajectory is planned with the obstacle included in the corridor generation, or, when unknown-space inflation is enabled, while it lies entirely in unknown space (Assumption~\ref{ass:untracked_in_unknown}).
An obstacle whose points are already mapped but which the tracker has not yet detected as a dynamic obstacle falls between the two cases.
The corridor is then built without the reachable-set inflation for that obstacle, so the agent can already be close to it when the first estimate arrives.
Once the obstacle is tracked, the inflation by $\mathcal{R}_n$ applies at every time layer, which can remove enough free space to make the trajectory optimization infeasible until the agent moves clear.
In Experiment 11 an obstacle was first tracked \SI{1.2}{s} before the closest approach.
Replanning failed on \SI{78}{\percent} of cycles over the \SI{2.5}{s} that followed, against \SI{6.6}{\percent} over the rest of that flight, and this episode produces the smallest $\hat{d}_{\min}$ in Table~\ref{tab:hw_multiple_dynamic_obstacles}.
The same failure mode appears in the perception-only simulations of Section~\ref{sec:dynamic-no-gt}, where late detection is the primary cause of the collisions observed there.
Shortening this delay is a matter for the perception and detection front end, which is a modular component of the pipeline (Section~\ref{sec:obstacle_tracking}).}

\textbf{Tracking error:}
If the low-level controller does not track the planned trajectory perfectly, the actual trajectory may exit the corridor and potentially collide with obstacles.
\revtwo{In practice, tracking errors are absorbed by the estimation-error buffer $\epsilon$.
Safety is therefore retained as long as the combined position-estimation and tracking error stays within $\epsilon$.
The reachable-set inflation provides further margin whenever an obstacle moves slower than its assumed bound.}

Note that safety is guaranteed only within the local planning horizon $N \cdot dt$; long-term safety requires continuous replanning, as discussed in Section~\ref{subsec:recursive_feasibility}.

\subsection{Recursive Feasibility}\label{subsec:recursive_feasibility}

Recursive feasibility means that if a feasible trajectory exists at the current replanning step, one will also exist at the next step.
OA-MPC~\cite{firoozi2024oa} and Stamouli~\etalcite{stamouli2024recursively} guarantee this for MPC in dynamic environments using a shrinking horizon tied to a fixed mission duration $T$, so the remaining planning time decreases at each step.
Because the horizon shrinks toward the same end time $T$, the obstacle's reachable sets do not grow between replanning steps, and the previous plan remains feasible for the next step.
However, this strategy is not suitable for long-duration navigation where the goal may be far from the agent and the mission duration is not fixed.
As described in Section~\ref{sec:system_overview_section}, SANDO replans in a receding-horizon manner toward a subgoal that moves with the agent.
Since the subgoal moves with the agent and the horizon does not shrink toward a fixed end time, the assumptions required for recursive feasibility do not hold.
\rev{Moreover, in dynamic environments, the terminal stop state at the end of the trajectory is guaranteed collision-free up to the end of the planning horizon.
A dynamic obstacle may enter the stop position after that time, so if the agent hovers there because no new feasible plan is found, safety is no longer guaranteed.}
One alternative is to adopt the approach of Wu~\etalcite{wu2012guaranteed} when the agent reaches a subgoal but cannot find a new trajectory.
Wu~\etalcite{wu2012guaranteed} guarantee infinite-horizon safety by allowing the agent to fly away from obstacles, but this requires the agent to be faster than all obstacles and the environment to be open, which is unrealistic in cluttered settings.
\rev{Furthermore, computing the safe set is computationally expensive, so extending this guarantee to cluttered multi-obstacle environments within a real-time replanning budget remains an open problem.}

\rev{Therefore, Theorem~\ref{thm:safety} guarantees collision-free safety up to the time the agent reaches the terminal stop state of its current plan, and each successful replan extends this guarantee through the new plan's horizon.
What SANDO does not guarantee is that a new feasible plan will exist at every future replanning step, so safety beyond the end of the current plan is not guaranteed.}
However, as shown in simulations in Section~\ref{sec:simulation-results} and hardware experiments in Section~\ref{sec:hardware-experiments}, SANDO often finds new trajectories before reaching the end of the current plan, effectively maintaining safety through continuous replanning even without formal recursive feasibility guarantees.
\section{Simulation Results}\label{sec:simulation-results}

We performed all the simulations on an \texttt{Alienware Aurora R8} desktop computer with a 16-core Intel\textsuperscript{\textregistered} Core~\textsuperscript{TM} i9 CPU and 64 GB of RAM.
The operating system is Ubuntu 22.04 LTS, and we used ROS2 Humble for SANDO's implementation.
Other methods were implemented in ROS1 Noetic/Melodic, so we dockerized them and ran them in ROS1 Noetic/Melodic on the same machine for benchmarking.
Table~\ref{tab:system_parameters} lists the SANDO system parameters that remain fixed across all simulations, and Table~\ref{tab:experiment_parameters} summarizes the per-experiment configuration parameters for both simulation and hardware experiments (Section~\ref{sec:hardware-experiments}).
For all baseline methods, we used their default parameter values (planning horizon, number of trajectory pieces, etc.), with one exception: in the static forest benchmark, EGO-Swarm2's default planning horizon of 7.5\,m caused frequent collisions even in the easy environment, so we increased it to 20.0\,m to achieve a higher success rate (see Table~\ref{tab:static_benchmark}).

\begin{table}
  \centering
  \caption{SANDO system parameters across all simulations.}
  \label{tab:system_parameters}
  \small
  \renewcommand{\arraystretch}{1.0}
  \resizebox{\columnwidth}{!}{%
  \begin{tabular*}{\columnwidth}{@{\extracolsep{\fill}} >{\centering\arraybackslash}m{0.13\columnwidth} >{\centering\arraybackslash}m{0.11\columnwidth} >{\centering\arraybackslash}m{0.16\columnwidth} m{0.40\columnwidth}@{}}
    \toprule
    \textbf{Module} & \textbf{Symbol} & \textbf{Value} & \textbf{Parameter} \\
    \midrule
    \multirow{4}{*}{\makecell{Static\\Heat Map}}
      & $\alpha^s$ & 5.0 & Intensity scale \\
      & $p^s$ & 2 & Decay exponent \\
      & $R^s$ & $3.0 \cdot r_{\text{drone}}$ & Halo radius \\
      & $H_{\max}$ & 50.0 & Maximum heat \\
    \midrule
    \multirow{5}{*}{\makecell{Dynamic\\Heat Map}}
      & $\alpha^d_0$ & 1.0 & Base cost scale \\
      & $\alpha^d_1$ & 2.0 & Tube penalty scale \\
      & $p^d,\, q^d$ & 2 & Decay exponents \\
      & $\gamma_k$ & \rev{$0$} & Tube growth rate \\
      & $\tau^d_{\text{ratio}}$ & 0.5 & Time weight ratio \\
    \midrule
    \makecell{Global\\Path}
      & \rev{$\Delta_{\mathrm{vtx}}$} & \rev{1.0\,m} & \rev{Max waypoint spacing (swept in Sec.~\ref{sec:psweep-disentangle})} \\
    \midrule
    \multirow{2}{*}{\makecell{Time\\Alloc.}}
      & $\Delta f$ & 0.1 & Factor step size \\
      & $\kappa$ & 0.4 & Window half-width \\
    \midrule
    Tracker
      & $\alpha$ & 0.9 & AEKF forgetting factor \\
    \midrule
    Solver
      & N/A & Gurobi & MIQP solver~\cite{gurobi} \\
    \bottomrule
  \end{tabular*}}
\end{table}

\begin{table*}
  \caption{Per-experiment SANDO configuration. ``N/A'' indicates the parameter is not applicable.}
  \label{tab:experiment_parameters}
  \centering
  \small
  \renewcommand{\arraystretch}{1.05}
  \setlength{\tabcolsep}{4pt}
  {\setlength{\fboxrule}{\ifshowrevisions2pt\else0pt\fi}\setlength{\fboxsep}{0pt}%
  \fcolorbox{revcolor}{white}{%
  \resizebox{\dimexpr\textwidth-2\fboxrule\relax}{!}{%
  \begin{tabular*}{\textwidth}{@{\extracolsep{\fill}}l c c c c c c@{}}
    \toprule
    \multicolumn{1}{c}{\multirow[0em]{4}{*}{\textbf{Parameter}}}
    & \multicolumn{4}{c}{\textbf{Simulation (Sec.~\ref{sec:simulation-results})}} & \multicolumn{2}{c}{\textbf{Hardware (Sec.~\ref{sec:hardware-experiments})}} \\
    \cmidrule(lr){2-5} \cmidrule(lr){6-7}
      & \makecell{Standardized\\(Sec.~\ref{sec:benchmarking-standardized-static-environments}, \ref{sec:variable-elimination-benchmarking})}
      & \makecell{Static\\Forest\\(Sec.~\ref{sec:benchmarking-static-environments})}
      & \makecell{Dynamic\\w/ GT\\(Sec.~\ref{sec:benchmarking-dynamic-environments}, \ref{sec:temporal-sfc-ablation})}
      & \makecell{Dynamic\\w/o GT\\(Sec.~\ref{sec:dynamic-no-gt})}
      & \makecell{Static\\(Exp.~1--6)}
      & \makecell{Dynamic\\(Exp.~7--18)} \\
    \midrule
    Pieces $N$
      & 4--6
      & 5
      & \rev{5}
      & \rev{5, 6}
      & 5
      & 5 \\
    Polytopes per layer $P$
      & 3
      & 3
      & \rev{2}
      & \rev{2, 3, 4, 5}
      & 3
      & 2 \\
    Drone radius $r_{\text{drone}}$ [m]
      & 0.1
      & 0.1
      & 0.1
      & 0.1
      & 0.45
      & 0.45 \\
    Max obs.\ velocity \rev{$v^{\mathrm{obs}}_{\max,i}$} [m/s]
      & N/A
      & N/A
      & \rev{(1.0, 1.0, 1.5)}
      & \rev{(1.0, 1.0, 1.5)}
      & N/A
      & \rev{0.6} \\
    A$^\ast$ heat weight $w_{\text{heat}}$
      & 5.0
      & 5.0
      & 5.0
      & 5.0
      & 5.0
      & \rev{10.0} \\
    Estimation error $\epsilon$ [m]
      & N/A
      & N/A
      & 0.0
      & 0.0
      & N/A
      & 0.3 \\
    \rev{Base inflation $r_{\text{base}}$ [m]}
      & \rev{N/A}
      & \rev{N/A}
      & \rev{0.1}
      & \rev{0.1}
      & \rev{N/A}
      & \rev{0.5} \\
    Max time factor $f_{\max}$
      & 2.5
      & 2.5
      & 2.5
      & 2.5
      & 3.0
      & \rev{2.5} \\
    Unknown space inflation\rev{ ($L_2$)}
      & N/A
      & on
      & on
      & on
      & off
      & \rev{on} \\
    \bottomrule
  \end{tabular*}}}}
\end{table*}

\rev{This section is organized in two parts.
We first benchmark SANDO against state-of-the-art methods in environments of increasing difficulty: a standardized static environment (Section~\ref{sec:benchmarking-standardized-static-environments}), a cluttered static forest (Section~\ref{sec:benchmarking-static-environments}), a dynamic environment with ground-truth obstacle knowledge (Section~\ref{sec:benchmarking-dynamic-environments}), and a dynamic environment with onboard perception only (Section~\ref{sec:dynamic-no-gt}), followed by a controlled study that disentangles MIQP complexity from horizon length (Section~\ref{sec:psweep-disentangle}).
We then ablate two design choices of SANDO: the variable elimination technique used in the trajectory optimizer (Section~\ref{sec:variable-elimination-benchmarking}), evaluated in the same standardized static environment as Section~\ref{sec:benchmarking-standardized-static-environments}, and the STSFC design (Section~\ref{sec:temporal-sfc-ablation}), evaluated in the same dynamic environment as Section~\ref{sec:benchmarking-dynamic-environments}.}

\subsection{Benchmarking in Standardized Static Environments}\label{sec:benchmarking-standardized-static-environments}

\begin{figure}
  \centering
  \includegraphics[width=\columnwidth, clip, trim=0 200 0 0]{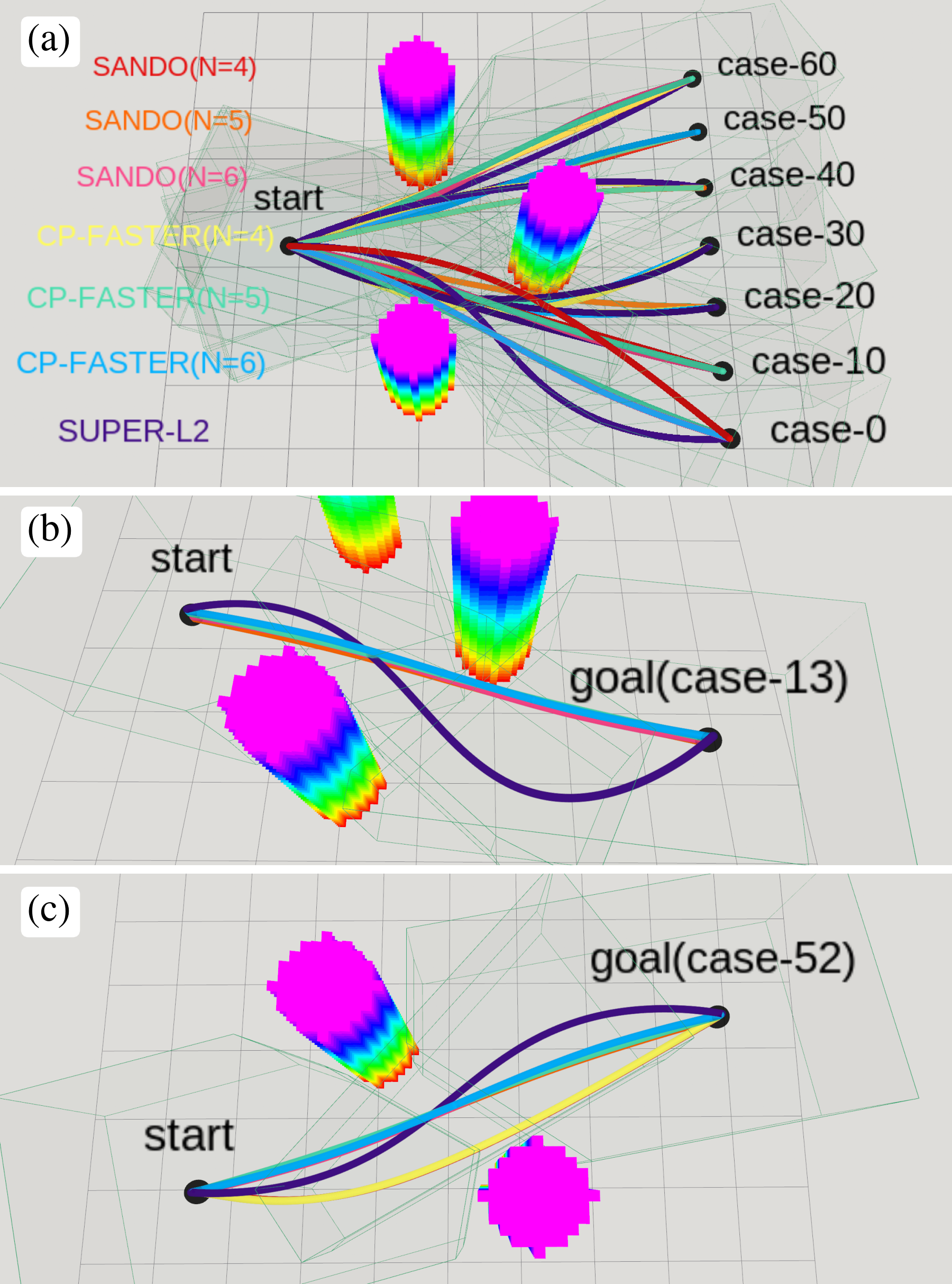}  %
  \caption{\rev{Standardized benchmarking environment: start at the left, goals at the right, safe flight corridors as green polytopes, trajectories color-coded by planner. (a)~Trajectories for 10 of the cases; (b) close-up of one case, where SUPER's soft constraints pull the trajectory toward polytope centers, yielding longer paths.}\label{fig:standardized_benchmarking_environment}}
\end{figure}

To compare SANDO against state-of-the-art methods, we first performed benchmarking experiments in a standardized static environment, as shown in Fig.~\ref{fig:standardized_benchmarking_environment}.
We evaluated 61 cases by varying the goal position from $-3$ to $3$\,m in $0.1$\,m increments along the $y$-axis.
For fair comparison, we used the same start and goal positions, dynamic constraints, and safe flight corridor constraints.
The dynamic constraints were set to $v_{\max} = 1.0$\,m/s, $a_{\max} = 2.0$\,m/s$^2$, and $j_{\max} = 3.0$\,m/s$^3$.
We compared SANDO against FASTER~\cite{tordesillas2022faster} and SUPER~\cite{ren2025super}, which are state-of-the-art static environment planners that both use safe flight corridors.
Since this benchmark has no dynamic obstacles, SANDO uses spatial safe flight corridors (SSFCs) rather than STSFCs; an SSFC is a single-time-layer STSFC (\ie $\mathcal{C}[0][p]$ only), identical to FASTER's corridor generation, since obstacle positions do not change over time.
FASTER uses MIQP-based hard constraints, while SUPER uses soft constraints.
SANDO and FASTER use Gurobi~\cite{gurobi} as the MIQP solver, while SUPER uses LBFGS~\cite{liu1989lbfgs} as the soft-constraint solver.
For a fair comparison, we relaxed SUPER to use per-axis dynamic constraints ($L_\infty$ norm), since SANDO and FASTER enforce per-axis constraints.
FASTER only applies dynamic constraints \rev{at the first control point} of each piece, so we also report the results of FASTER with additional dynamic constraints at all control points, denoted as FASTER (CP).
The original FASTER is denoted as FASTER (orig.).
For SANDO and FASTER, we set the number of pieces from 4 to 6.

\rev{The metrics are: $R^{\mathrm{opt}}_{\mathrm{succ}}$ [\%], the optimization success rate;
$T^{\mathrm{per}}_{\mathrm{opt}}$ [ms], the average runtime of each individual optimization (SANDO and FASTER iterate time allocation and trajectory optimization);
$T^{\mathrm{total}}_{\mathrm{opt}}$ [ms], the total optimization runtime (for SUPER, which couples spatial and temporal optimization in a single problem, the two coincide);
$T_{\mathrm{trav}}$ [s], the trajectory travel time;
$L_{\mathrm{path}}$ [m], the total path length;
$S_{\mathrm{jerk}}=\!\int_0^{T_{\mathrm{trav}}}\!\lVert\dddot{\mathbf{p}}(t)\rVert\,dt$ [m/s$^{2}$], the L1 jerk integral (smoothness);
and $\rho_{\mathrm{SFC}},\,\rho_{\mathrm{vel}},\,\rho_{\mathrm{acc}},\,\rho_{\mathrm{jerk}}$ [\%], the percentage of trajectory points violating the respective constraints ($\rho_{\mathrm{acc/jerk}}$ combines acceleration and jerk when reported jointly).
FASTER and SUPER each have exploratory and safe trajectory optimizers; we report the exploratory one, which performs better.
SUPER and SANDO use multi-threading for parallel optimization, while FASTER is single-threaded; Table~\ref{tab:standardized_benchmark} reports multi-threaded SANDO.}

Table~\ref{tab:standardized_benchmark} summarizes the benchmarking results.
As $N$ increases, both SANDO and FASTER \rev{(CP)} achieve better trajectory performance (shorter travel time) at the cost of increased computation time, since trajectories with more segments have larger degrees of freedom.
Note that FASTER (orig.) achieves the fastest travel time at $N=5$ but with a high velocity violation rate of \SI{38.0}{\percent}, while FASTER (CP) achieves zero velocity violation since it applies dynamic constraints at all control points.
\rev{Relative to a single-threaded run (not tabulated), multi-threaded SANDO has a slightly higher per-optimization time $T^{\mathrm{per}}_{\mathrm{opt}}$ from thread overhead (6.8 vs.\ 6.1\,ms at $N\!=\!5$) but achieves substantially faster total optimization time $T^{\mathrm{total}}_{\mathrm{opt}}$ through parallel execution (9.7 vs.\ 18.2\,ms at $N\!=\!5$).}
No method exhibited acceleration or jerk constraint violations.
Overall, SUPER achieves the fastest computation time due to its MINCO-based~\cite{wang2022minco} soft-constraint spatiotemporal optimization, but at the cost of constraint violations.
Among hard-constraint methods, SANDO and FASTER (CP) achieve the best trajectory performance with zero constraint violations, and SANDO is consistently computationally faster than FASTER across all $N$, especially with multi-threading.
These results illustrate the fundamental trade-off between $N$ (and similarly $P$) and computational cost: larger $N$ and $P$ increase the MIQP's degrees of freedom and the number of binary assignment variables ($N \times P$), improving trajectory quality but increasing solve time.
In dynamic environments where fast replanning is critical, we use $N = 5$ and \rev{$P = 2$} as a practical operating point that balances trajectory quality against real-time computation (see Table~\ref{tab:experiment_parameters}).

\begin{table*}
  \caption{Local trajectory optimization benchmarking results \rev{against FASTER~\cite{tordesillas2022faster} and SUPER~\cite{ren2025super}} (computation time, performance, and constraint violation).
  SANDO and FASTER (CP) with $N=6$ achieve the \rev{shortest travel time among the configurations with no constraint violations, and} SANDO achieves substantially faster computation time than FASTER (CP) due to multi-threading.
  We mark in \best{green} the best value in each column and in \worst{red} the worst value.}
  \label{tab:standardized_benchmark}
  \centering
  \renewcommand{\arraystretch}{1.0}
  \resizebox{\textwidth}{!}{
    \begin{tabular}{c @{\hspace{4pt}} l c c c c c c c c c c c}
      \toprule
      \multicolumn{2}{c}{\multirow{2}{*}[-0.4ex]{\textbf{Algorithm}}}
      & \multirow{2}{*}[-0.4ex]{\textbf{Thread}}
      & \multirow{2}{*}[-0.4ex]{\textbf{N}}
      & \multicolumn{1}{c}{\textbf{Success}}
      & \multicolumn{2}{c}{\textbf{Computation Time}}
      & \multicolumn{3}{c}{\textbf{Performance}}
      & \multicolumn{3}{c}{\textbf{Constraint Violation}}
      \\
      \cmidrule(lr){5-5}
      \cmidrule(lr){6-7}
      \cmidrule(lr){8-10}
      \cmidrule(lr){11-13}
      &&&&
      $R^{\mathrm{opt}}_{\mathrm{succ}}$ [\%] &
      $T^{\mathrm{per}}_{\mathrm{opt}}$ [ms] &
      $T^{\mathrm{total}}_{\mathrm{opt}}$ [ms] &
      $T_{\mathrm{trav}}$ [s] &
      $L_{\mathrm{path}}$ [m] &
      $S_{\mathrm{jerk}}$ [m/s$^{2}$] &
      $\rho_{\mathrm{SFC}}${[\%]} &
      $\rho_{\mathrm{vel}}${[\%]} &
      $\rho_{\mathrm{acc/jerk}}${[\%]}
      \\
      \midrule
      \multirow{2}{*}{SUPER} & ($L_2$) & \multirow{2}{*}{multi} & \multirow{2}{*}{--} & \best{100.0} & \multicolumn{2}{c}{0.6} & \worst{14.1} & \best{6.8} & \best{1.3} & 0.2 & 0.5 & \best{0.0} \\
       & ($L_\infty$) &  &  & \best{100.0} & \multicolumn{2}{c}{\best{0.5}} & 13.6 & \worst{6.9} & 1.5 & \worst{0.3} & 4.0 & \best{0.0} \\

      \midrule

      \multirow{2}{*}{FASTER} & (orig.) & \multirow{2}{*}{single} & \multirow{3}{*}{4} & \worst{93.4} & 7.9 & 48.1 & 13.0 & \best{6.8} & 1.4 & \best{0.0} & 5.0 & \best{0.0} \\
       & (CP) &  &  & \worst{93.4} & 4.5 & 36.5 & 13.2 & \best{6.8} & 1.4 & \best{0.0} & \best{0.0} & \best{0.0} \\
      \multicolumn{2}{c}{SANDO} & multi &  & \worst{93.4} & 1.3 & 3.2 & 13.2 & \best{6.8} & 1.4 & \best{0.0} & \best{0.0} & \best{0.0} \\

      \midrule

      \multirow{2}{*}{FASTER} & (orig.) & \multirow{2}{*}{single} & \multirow{3}{*}{5} & \best{100.0} & 16.8 & 31.3 & \best{8.5} & \best{6.8} & \worst{8.4} & \best{0.0} & \worst{38.0} & \best{0.0} \\
       & (CP) &  &  & \best{100.0} & 18.0 & 67.6 & 11.1 & \best{6.8} & 1.9 & \best{0.0} & \best{0.0} & \best{0.0} \\
      \multicolumn{2}{c}{SANDO} & multi &  & \best{100.0} & 6.8 & 9.7 & 11.1 & \best{6.8} & 1.9 & \best{0.0} & \best{0.0} & \best{0.0} \\

      \midrule

      \multirow{2}{*}{FASTER} & (orig.) & \multirow{2}{*}{single} & \multirow{3}{*}{6} & \best{100.0} & \worst{24.6} & \worst{94.9} & 9.8 & \best{6.8} & 2.8 & \best{0.0} & 8.2 & \best{0.0} \\
       & (CP) &  &  & \best{100.0} & 20.2 & 56.8 & 9.8 & \best{6.8} & 2.8 & \best{0.0} & \best{0.0} & \best{0.0} \\
      \multicolumn{2}{c}{SANDO} & multi &  & \best{100.0} & 17.2 & 19.7 & 9.8 & \best{6.8} & 2.8 & \best{0.0} & \best{0.0} & \best{0.0} \\
      \bottomrule
    \end{tabular}
  }
\end{table*}

\subsection{Benchmarking against State-of-the-Art Methods in Static Environments}\label{sec:benchmarking-static-environments}

\begin{figure}
  \centering
  \includegraphics[width=\columnwidth]{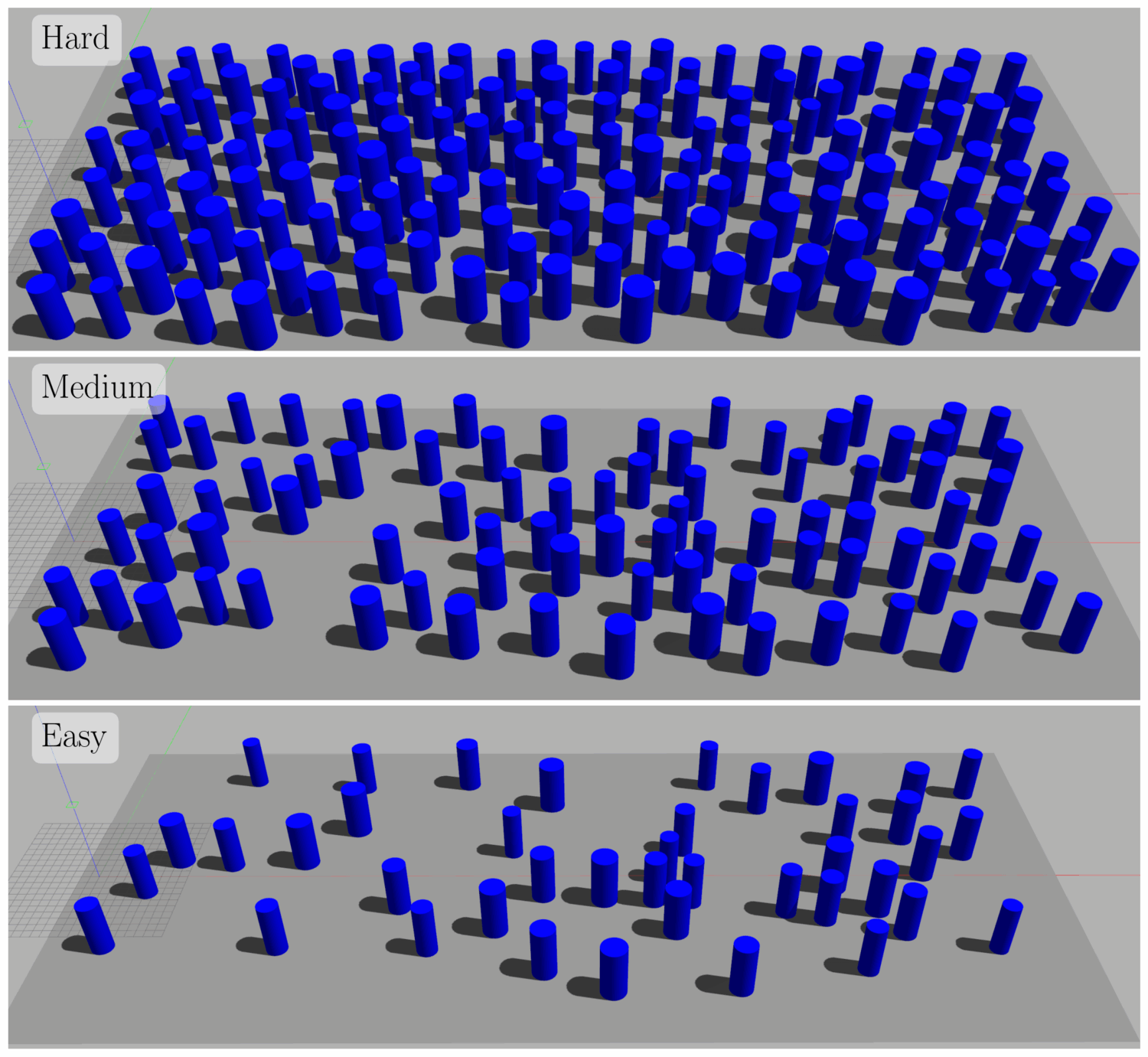} %
  \caption{Static Benchmarking: The static forest Gazebo environment used for benchmarking SANDO against state-of-the-art methods.\label{fig:static_benchmarking_environment}}
\end{figure}

\begin{figure}
  \centering
  \includegraphics[width=\columnwidth]{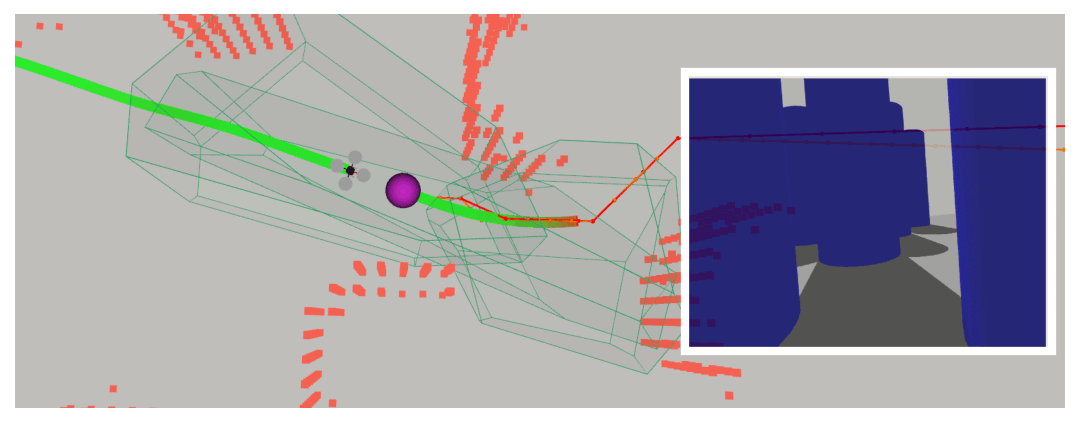} %
  \caption{Rviz Visualization of Static Benchmarking (Hard Environment): One of the simulation runs in the hard static environment. The purple dot is the replanning point (Point A), the green polytopes are the safe flight corridors, and the color-coded trajectory is the optimized trajectory, where green is faster and red is slower. The orange dots are occupied voxels, the orange line is the original global path, and the red line is a smoother version of the global path used for SFC generation.\label{fig:static_benchmarking_rviz}}
\end{figure}

To evaluate SANDO's full planning pipeline in realistic settings, we benchmarked it against state-of-the-art methods in obstacle-rich static forest environments at three difficulty levels.
\rev{Static cylindrical obstacles (radius \SIrange{1.0}{1.5}{m}, height \SI{6}{m}) are placed in a $\SI{100}{m}\times\SI{40}{m}$ Gazebo forest area by independent uniform sampling.
Each obstacle's $(x, y)$ center and radius are drawn uniformly, and sampling continues until the target area-coverage is reached: 5\%, 10\%, or 20\% for Easy, Medium, and Hard, respectively.}
The agent starts at $(0,0,3)\,\si{m}$ and the goal is $(105,0,3)\,\si{m}$.

We benchmarked SANDO against EGO-Swarm2~\cite{zhou2022swarm}, SUPER~\cite{ren2025super}, and FASTER~\cite{tordesillas2022faster}.
To simulate LiDAR data, we used the \texttt{livox\_ros\_driver2} package, which provides a ROS 2 interface for the Livox Mid-360 LiDAR sensor; all methods receive the same sensor data for fair comparison.
The dynamic constraints were set to $v_{\max} = 5.0$\,m/s, $a_{\max} = 20.0$\,m/s$^2$, and $j_{\max} = 100.0$\,m/s$^3$.
SANDO and FASTER enforce hard $L_\infty$ dynamic constraints, while EGO-Swarm2 and SUPER use soft constraints.
EGO-Swarm2 and SUPER originally use $L_2$ norms for dynamic constraints, but we relaxed them to $L_\infty$ norms for fair comparison since SANDO and FASTER enforce per-axis constraints.
We performed 10 simulations per difficulty level with the same evaluation metrics in Section~\ref{sec:benchmarking-standardized-static-environments} with the addition of $R_{\mathrm{succ}}$ [\%] (overall success rate; the percentage of runs in which the agent reaches the goal without collision) and $T_{\mathrm{replan}}$ [ms] (total replanning computation time).
Fig.~\ref{fig:static_benchmarking_rviz} shows a visualization of one of SANDO's simulation runs in RViz.

Table~\ref{tab:static_benchmark} summarizes the results.
SANDO achieves a 100\% success rate across all difficulty levels, whereas FASTER \rev{(orig.)} drops to 90\% and 70\% in the medium and hard cases, respectively, and EGO-Swarm2 drops to 50--60\% in the hard case.
SANDO also achieves the fastest travel time and the lowest total optimization computation time $T^{\mathrm{total}}_{\mathrm{opt}}$ across all cases.
Since SUPER and FASTER use both safe and exploratory trajectory optimization, we report the computation time of both trajectories in $T^{\mathrm{total}}_{\mathrm{opt}}$.
SANDO maintains no constraint violations in velocity, acceleration, and jerk owing to its hard-constraint formulation.
Although FASTER also enforces hard constraints, it exhibits velocity violations of 7.9\% in the hard case, likely due to the constraints only being applied at the first control point of each piece.
\rev{As in the standardized benchmark (Section~\ref{sec:benchmarking-standardized-static-environments}), FASTER (CP), which enforces the dynamic constraints at all control points, eliminates these violations entirely, with no velocity, acceleration, or jerk constraint violations across all difficulty levels.}
SUPER shows high jerk violation rates (8--12.7\%) because its soft-constraint solver does not strictly enforce dynamic limits.
Similarly, EGO-Swarm2 exhibits velocity violations (3--8.6\%) due to its soft-constraint optimization.
In terms of smoothness, SANDO's jerk integral $S_{\mathrm{jerk}}$ is lower than SUPER's and FASTER\rev{~(orig.)'s (FASTER~(CP) attains a lower $S_{\mathrm{jerk}}$ than SANDO only in the easy case)}, but higher than EGO-Swarm2's.
Overall, SANDO achieves the best overall performance with a 100\% success rate, the fastest travel time, and no constraint violations across all difficulty levels.

\begin{table*}
  \caption{Benchmark results against \rev{EGO-Swarm2~\cite{zhou2022swarm}, SUPER~\cite{ren2025super}, and FASTER~\cite{tordesillas2022faster}} in static environments. Since SUPER performs global path planning for both exploratory and safe trajectories, we list the corresponding computation times as {Exploratory | Safe} in the Global Path Planning Computation Time column. We mark in \best{green} the best value and in \worst{red} the worst value in each column per case.}
  \label{tab:static_benchmark}
  \centering
  \renewcommand{\arraystretch}{1.0}
  {\setlength{\fboxrule}{\ifshowrevisions2pt\else0pt\fi}\setlength{\fboxsep}{0pt}%
  \fcolorbox{revcolor}{white}{%
  \resizebox{\dimexpr\textwidth-2\fboxrule\relax}{!}{
    \begin{tabular}{c c c c c c c c c c c c c c}
      \toprule
      \multirow{2}{*}[-0.4ex]{\textbf{Env}}
      & \multicolumn{2}{c}{\multirow{2}{*}[-0.4ex]{\textbf{Algorithm}}}
      & \multicolumn{2}{c}{\multirow{2}{*}[-0.4ex]{\textbf{Constr.}}}
      & \multicolumn{1}{c}{\textbf{Success}}
      & \multicolumn{2}{c}{\textbf{Comp.\ Time}}
      & \multicolumn{3}{c}{\textbf{Performance}}
      & \multicolumn{3}{c}{\textbf{Constraint Violation}}
      \\
      \cmidrule(lr){6-6}
      \cmidrule(lr){7-8}
      \cmidrule(lr){9-11}
      \cmidrule(lr){12-14}
      &&&&&
      $R_{\mathrm{succ}}$ [\%]
      & $T^{\mathrm{total}}_{\mathrm{opt}}$ [ms]
      & $T_{\mathrm{replan}}$ [ms]
      & $T_{\mathrm{trav}}$ [s]
      & $L_{\mathrm{path}}$ [m]
      & $S_{\mathrm{jerk}}$ [m/s$^{2}$]
      & $\rho_{\mathrm{vel}}$ [\%]
      & $\rho_{\mathrm{acc}}$ [\%]
      & $\rho_{\mathrm{jerk}}$ [\%]
      \\
      \midrule

      \multirow{7}{*}{Easy} & \multicolumn{2}{c}{\multirow{2}{*}{EGO-Swarm2}} & \multirow{2}{*}{Soft} & $L_2$ & \best{100.0} & {5.4} & \best{6.2} & {25.7} & {108.5} & {146.0} & {7.7} & \best{0.0} & \best{0.0} \\
       & \multicolumn{2}{c}{} &  & $L_\infty$ & \best{100.0} & {5.7} & {6.5} & {25.0} & {107.2} & \best{142.4} & \worst{8.6} & \best{0.0} & \best{0.0} \\
       & \multicolumn{2}{c}{\multirow{2}{*}{SUPER}} & \multirow{2}{*}{Soft} & $L_2$ & \best{100.0} & 8.1 | 23.8 & \worst{32.0} & \worst{26.0} & \worst{122.2} & {1291.0} & \best{0.0} & \best{0.0} & \worst{8.3} \\
       & \multicolumn{2}{c}{} & & $L_\infty$ & \best{100.0} & 8.4 | 23.4 & {31.9} & {25.7} & {121.1} & \worst{1305.4} & \best{0.0} & \best{0.0} & {8.0} \\
       & \multirow{2}{*}{FASTER} & (orig.) & \multirow{2}{*}{Hard} & \multirow{2}{*}{$L_\infty$} & \best{100.0} & {19.4 | 38.4} & {23.9} & {22.8} & \best{105.5} & {266.9} & {4.4} & \best{0.0} & \best{0.0} \\
       & & (CP) & & & \best{100.0} & {\worst{20.5} | 38.4} & {25.2} & {22.1} & \best{105.5} & {155.8} & \best{0.0} & \best{0.0} & \best{0.0} \\
       & \multicolumn{2}{c}{SANDO} & Hard & $L_\infty$ & \best{100.0} & \best{3.7} & {12.3} & \best{21.6} & {105.6} & {184.9} & \best{0.0} & \best{0.0} & \best{0.0} \\

      \midrule

      \multirow{7}{*}{Medium} & \multicolumn{2}{c}{\multirow{2}{*}{EGO-Swarm2}} & \multirow{2}{*}{Soft} & $L_2$ & \best{100.0} & {5.4} & {6.0} & \worst{27.2} & {109.4} & {196.7} & \worst{7.0} & \best{0.0} & \best{0.0} \\
       & \multicolumn{2}{c}{} &  & $L_\infty$ & \best{100.0} & {4.7} & \best{5.2} & \worst{27.2} & {109.0} & \best{186.7} & {6.2} & \best{0.0} & \best{0.0} \\
       & \multicolumn{2}{c}{\multirow{2}{*}{SUPER}} & \multirow{2}{*}{Soft} & $L_2$ & \best{100.0} & 8.5 | 22.5 & \worst{31.1} & {25.8} & \worst{121.5} & {1358.9} & \best{0.0} & \best{0.0} & {8.4} \\
       & \multicolumn{2}{c}{} & & $L_\infty$ & \best{100.0} & 8.4 | 22.1 & {30.6} & {25.7} & {120.6} & \worst{1369.8} & \best{0.0} & \best{0.0} & \worst{8.6} \\
       & \multirow{2}{*}{FASTER} & (orig.) & \multirow{2}{*}{Hard} & \multirow{2}{*}{$L_\infty$} & \worst{90.0} & {22.2 | 38.4} & {26.7} & {23.1} & {105.7} & {320.8} & {4.5} & \best{0.0} & \best{0.0} \\
       & & (CP) & & & \best{100.0} & {\worst{23.9} | 42.1} & {29.4} & {23.4} & \best{105.6} & {330.4} & \best{0.0} & \best{0.0} & \best{0.0} \\
       & \multicolumn{2}{c}{SANDO} & Hard & $L_\infty$ & \best{100.0} & \best{4.5} & {14.9} & \best{21.6} & {106.1} & {239.7} & \best{0.0} & \best{0.0} & \best{0.0} \\

      \midrule

      \multirow{7}{*}{Hard} & \multicolumn{2}{c}{\multirow{2}{*}{EGO-Swarm2}} & \multirow{2}{*}{Soft} & $L_2$ & \worst{50.0} & {23.5} & {26.8} & \worst{34.2} & {120.4} & {311.7} & {3.0} & \best{0.0} & \best{0.0} \\
       & \multicolumn{2}{c}{} &  & $L_\infty$ & {60.0} & {21.9} & {24.5} & {32.3} & \worst{121.6} & \best{296.0} & {5.1} & \best{0.0} & \best{0.0} \\
       & \multicolumn{2}{c}{\multirow{2}{*}{SUPER}} & \multirow{2}{*}{Soft} & $L_2$ & \best{100.0} & 8.8 | 18.9 & {27.9} & {24.6} & {114.9} & {1504.2} & \best{0.0} & \best{0.0} & {11.9} \\
       & \multicolumn{2}{c}{} & & $L_\infty$ & \best{100.0} & 9.2 | 19.5 & {28.8} & {24.8} & {115.4} & \worst{1527.3} & \best{0.0} & \best{0.0} & \worst{12.7} \\
       & \multirow{2}{*}{FASTER} & (orig.) & \multirow{2}{*}{Hard} & \multirow{2}{*}{$L_\infty$} & {70.0} & {35.0 | 36.2} & {41.4} & {30.6} & {107.9} & {1348.7} & \worst{7.9} & \best{0.0} & \best{0.0} \\
       & & (CP) & & & {80.0} & {\worst{36.4} | 44.5} & \worst{43.5} & {27.9} & \best{107.8} & {952.7} & \best{0.0} & \best{0.0} & \best{0.0} \\
       & \multicolumn{2}{c}{SANDO} & Hard & $L_\infty$ & \best{100.0} & \best{6.1} & \best{21.6} & \best{22.2} & {108.7} & {581.4} & \best{0.0} & \best{0.0} & \best{0.0} \\
      \bottomrule
    \end{tabular}
  }}}
\end{table*}

\subsection{SANDO in Dynamic Environments}\label{sec:benchmarking-dynamic-environments}

\begin{figure}
  \centering
  \includegraphics[width=\columnwidth]{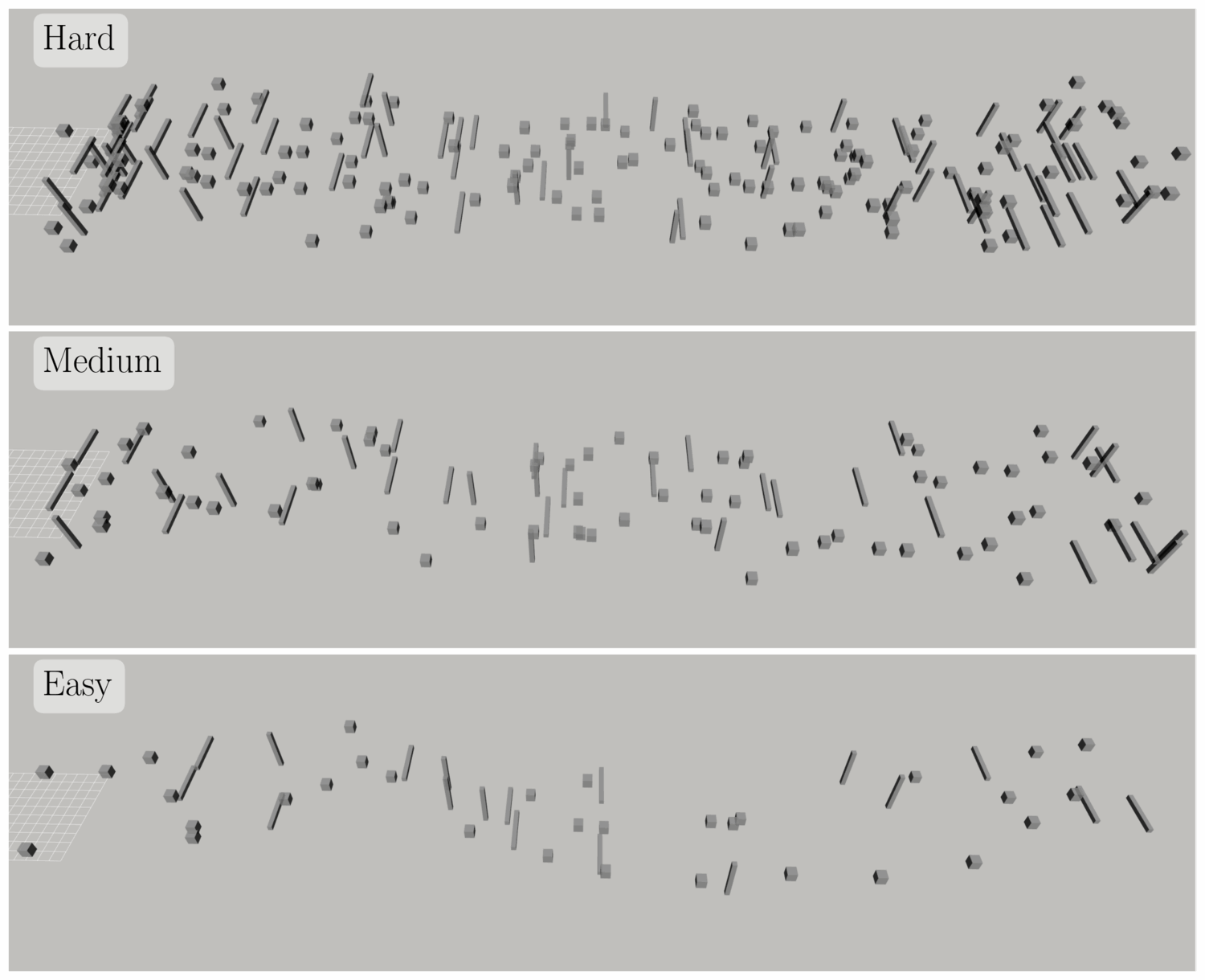} %
  \caption{Dynamic Benchmarking: The dynamic Gazebo environment used for benchmarking SANDO against state-of-the-art methods at three difficulty levels (Easy, Medium, and Hard)\label{fig:dynamic_benchmarking_environment}}
\end{figure}

We next evaluated SANDO in environments containing both static and dynamic obstacles to test its spatiotemporal collision avoidance capability.
\rev{The environment is a $\SI{100}{m}\times\SI{40}{m}$ forest with static bar-shaped obstacles ($\SI{0.4}{m}\times\SI{0.4}{m}\times\SI{4}{m}$, in vertical and horizontal orientations) and dynamic obstacles modeled as \SI{0.8}{m} cubes following trefoil-knot trajectories, and the positions are sampled independently from uniform distributions.}
\rev{Differentiating the trefoil-knot parameterization yields per-axis peak obstacle speeds of $v^{\mathrm{obs}}_{\max,x}=\SI{0.83}{m/s}$, $v^{\mathrm{obs}}_{\max,y}=\SI{0.94}{m/s}$, and $v^{\mathrm{obs}}_{\max,z}=\SI{1.5}{m/s}$ at the worst-case parameter draws ($s=4$ and $f_{\mathrm{period}}=4$).
We therefore set the obstacle per-axis velocity bounds to $(v^{\mathrm{obs}}_{\max,x}, v^{\mathrm{obs}}_{\max,y}, v^{\mathrm{obs}}_{\max,z})=(1.0, 1.0, 1.5)$\,m/s (Table~\ref{tab:experiment_parameters}), upper-bounding the worst-case per-axis peaks.
Since we inflate obstacles' reachable sets per-axis using these bounds (Formulation~(B) of Assumption~\ref{ass:bounded_vel}), as described in Section~\ref{subsec:obstacle_inflation}, by Theorem~\ref{thm:safety}, SANDO's trajectories are guaranteed to be collision-free.}
Three difficulty levels are defined: Easy (50 obstacles, ${\sim}33$ dynamic), Medium (100 obstacles, ${\sim}65$ dynamic), and Hard (200 obstacles, ${\sim}130$ dynamic), with approximately 65\% of obstacles being dynamic.
The agent starts at $(0,0,2)\,\si{m}$ with a goal at $(105,0,2)\,\si{m}$.

We benchmarked SANDO against EGO-Swarm2~\cite{zhou2022swarm}, I-MPC~\cite{xu2025intent}, and FAPP~\cite{lu2024fapp} with 10 simulations per difficulty level.
All methods use $L_\infty$ dynamic constraints set to $v_{\max} = 5.0$\,m/s, $a_{\max} = 20.0$\,m/s$^2$, and $j_{\max} = 100.0$\,m/s$^3$.
For I-MPC, we swept over six combinations of velocity and acceleration limits ($v_{\max} \in \{1,2,3,4,5\}$\,m/s with matched $a_{\max}$, plus $v_{\max}=5$\,m/s, $a_{\max}=20$\,m/s$^2$) because at the nominal limits ($v_{\max}=5$\,m/s, $a_{\max}=20$\,m/s$^2$), I-MPC reports large constraint violations.
\rev{Table~\ref{tab:dynamic_benchmark} reports all six I-MPC configurations: tightening its limits reduces velocity violations but slows the agent and never reaches a good operating point in the denser scenes, with its highest per-environment success rates being $90\%$ (Easy), $70\%$ (Medium), and $40\%$ (Hard).}
Constraint violations for all methods are evaluated against the shared limits $v_{\max}=5$\,m/s and $a_{\max}=20$\,m/s$^2$.
Each method uses different inputs for dynamic obstacles.
For instance, EGO-Swarm2 does not have a built-in dynamic obstacle estimator/predictor, but I-MPC and FAPP have built-in estimators/predictors.
For a fair comparison, we feed the ground truth positions and future predictions of dynamic obstacles to EGO-Swarm2 since it only receives a quintic polynomial prediction of each dynamic obstacle with a clearance radius.
For I-MPC, we feed the ground truth positions, velocities, and sizes of dynamic obstacles since that is what their built-in \texttt{fake\_detector} provides.
For FAPP, we also feed the ground truth positions and velocities.
FAPP uses a fixed squared-distance threshold of \SI{1.96}{\meter\squared} in its obstacle avoidance cost, corresponding to an effective avoidance radius of $\sqrt{1.96} \approx 1.4\,\mathrm{m}$ from the obstacle center regardless of the actual obstacle size; \ie it does not adapt per obstacle.
In contrast, for this benchmarking, SANDO receives only the current position of dynamic obstacles, and hence SANDO receives the least information about dynamic obstacles compared to other methods.

Table~\ref{tab:dynamic_benchmark} summarizes the results.
\rev{SANDO is configured with $N=5$ and $P=2$ throughout this benchmark, and SANDO achieves a $100\%$ success rate across all three difficulty levels, whereas every baseline fails in at least one case.
In the Hard environment SANDO also achieves the fastest travel time; the shortest paths in all three environments, and the fastest travel times in Easy and Medium, are achieved by either I-MPC or FAPP, which reach at most $90\%$ success and violate dynamic constraints.}
Most notably, SANDO maintains zero constraint violations in velocity, acceleration, and jerk across all cases, demonstrating that hard constraints in the MIQP formulation reliably enforce dynamic feasibility even in dense dynamic environments.
\rev{ 
EGO-Swarm2 and FAPP use soft collision and dynamic-feasibility costs, which are encouraged but not guaranteed: both violate the velocity limit (EGO-Swarm2 $5.1$--$11.8\%$, FAPP $3.1$--$8.1\%$), and their penalty-based avoidance cannot reliably clear the densest cases, dropping to $80\%$ (EGO-Swarm2) and $50\%$ (FAPP) success in Hard; FAPP is further limited by a fixed avoidance radius ($\approx\SI{1.4}{m}$) that ignores obstacle size. 
I-MPC does use hard constraints for dynamic constraints, but enforces them only at discrete MPC steps and on a closed-loop high-fidelity model, so tracking error degrades both safety and feasibility.
Its longer per-optimization times ($10.9$--$\SI{32.5}{ms}$ vs.\ $3.0$--$\SI{4.3}{ms}$ for SANDO) also lower its effective replanning rate.}
These results show that SANDO's STSFC approach with worst-case reachable set inflation maintains safety under the most conservative obstacle information assumption.

\rev{
Note that constraint-violation rates for SANDO, FASTER, SUPER, FAPP, and EGO-Swarm2 are computed on each planner's nominal kinematic trajectory (perfect tracking, each method's default published code), whereas I-MPC's published code uses a tracker that simulates real-world tracking behavior, so its reported violation rates include tracker dynamics and a perfectly tracked I-MPC would likely show lower violations.
}

\begin{table*}
  \caption{Dynamic obstacle benchmarking results \rev{against EGO-Swarm2~\cite{zhou2022swarm}, I-MPC~\cite{xu2025intent}, and FAPP~\cite{lu2024fapp} (metrics defined in Section~\ref{sec:benchmarking-standardized-static-environments}); for I-MPC we report the full $(v_{\max}, a_{\max})$ sweep}. We mark in \best{green} the best value and in \worst{red} the worst value in each column per case.}
  \label{tab:dynamic_benchmark}
  \centering
  \renewcommand{\arraystretch}{1.0}
  {\setlength{\fboxrule}{\ifshowrevisions2pt\else0pt\fi}\setlength{\fboxsep}{0pt}%
  \fcolorbox{revcolor}{white}{%
  \resizebox{\dimexpr\textwidth-2\fboxrule\relax}{!}{
    \begin{tabular}{c c @{\hskip -0.3em} c c c c c c c c c}
      \toprule
      \multirow{2}{*}[-0.4em]{\textbf{Env}}
      & \multicolumn{2}{c}{\multirow{2}{*}[-0.4em]{\textbf{Algorithm}}}
      & \multicolumn{1}{c}{\textbf{Success}}
      & \multicolumn{1}{c}{\textbf{Comp. Time}}
      & \multicolumn{3}{c}{\textbf{Performance}}
      & \multicolumn{3}{c}{\textbf{Constraint Violation}}
      \\
      \cmidrule(lr){4-4}
      \cmidrule(lr){5-5}
      \cmidrule(lr){6-8}
      \cmidrule(lr){9-11}
      &&&
      $R_{\mathrm{succ}}$ [\%] &
      $T^{\mathrm{per}}_{\mathrm{opt}}$ [ms] &
      $T_{\mathrm{trav}}$ [s] &
      $L_{\mathrm{path}}$ [m] &
      $S_{\mathrm{jerk}}$ [m/s$^{2}$] &
      $\rho_{\mathrm{vel}}$ [\%] &
      $\rho_{\mathrm{acc}}$ [\%] &
      $\rho_{\mathrm{jerk}}$ [\%]
      \\
      \midrule
      \multirow{9}{*}{Easy}
       & \multicolumn{2}{c}{EGO-Swarm2} & \best{100.0} & 5.7 & 25.4 & 107.5 & \best{130.4} & 11.8 & \best{0.0} & \best{0.0} \\
      & I-MPC & \multicolumn{1}{l}{($v$=1, $a$=1)} & 80.0 & 12.1 & \worst{97.3} & 117.3 & 1029.6 & \best{0.0} & \best{0.0} & {-} \\
      & I-MPC & \multicolumn{1}{l}{($v$=2, $a$=2)} & 90.0 & 12.3 & 52.7 & 115.6 & 1167.1 & \best{0.0} & \best{0.0} & {-} \\
      & I-MPC & \multicolumn{1}{l}{($v$=3, $a$=3)} & 90.0 & 12.1 & 38.3 & 117.7 & 1254.6 & \best{0.0} & \best{0.0} & {-} \\
      & I-MPC & \multicolumn{1}{l}{($v$=4, $a$=4)} & 90.0 & \worst{12.5} & 29.6 & \best{104.8} & 1236.6 & 1.3 & \best{0.0} & {-} \\
      & I-MPC & \multicolumn{1}{l}{($v$=5, $a$=5)} & \worst{70.0} & 10.9 & 25.7 & 118.8 & 1322.9 & \worst{56.2} & \best{0.0} & {-} \\
      & I-MPC & \multicolumn{1}{l}{($v$=5, $a$=20)} & 80.0 & 11.9 & 26.2 & \worst{131.1} & \worst{2383.3} & 47.2 & \worst{4.8} & {-} \\
      & \multicolumn{2}{c}{FAPP} & 80.0 & \best{0.6} & \best{22.8} & 106.3 & 178.9 & 8.1 & \best{0.0} & \best{0.0} \\
      & \multicolumn{2}{c}{SANDO} & \best{100.0} & 3.0 & 24.4 & 107.6 & 493.7 & \best{0.0} & \best{0.0} & \best{0.0} \\
      \midrule
      \multirow{9}{*}{Medium}
       & \multicolumn{2}{c}{EGO-Swarm2} & \best{100.0} & 8.5 & 26.2 & 108.3 & \best{145.3} & 10.4 & \best{0.0} & \best{0.0} \\
      & I-MPC & \multicolumn{1}{l}{($v$=1, $a$=1)} & 70.0 & 22.7 & \worst{96.7} & 123.6 & 1037.0 & \best{0.0} & \best{0.0} & {-} \\
      & I-MPC & \multicolumn{1}{l}{($v$=2, $a$=2)} & 70.0 & \worst{22.9} & 53.3 & 119.9 & 1181.9 & \best{0.0} & \best{0.0} & {-} \\
      & I-MPC & \multicolumn{1}{l}{($v$=3, $a$=3)} & 60.0 & 22.2 & 38.4 & 122.2 & 1123.2 & \best{0.0} & \best{0.0} & {-} \\
      & I-MPC & \multicolumn{1}{l}{($v$=4, $a$=4)} & \worst{30.0} & 20.2 & 30.1 & 120.9 & 1138.6 & 3.5 & \best{0.0} & {-} \\
      & I-MPC & \multicolumn{1}{l}{($v$=5, $a$=5)} & 50.0 & 17.6 & 25.1 & 120.8 & 1100.2 & \worst{60.3} & \best{0.0} & {-} \\
      & I-MPC & \multicolumn{1}{l}{($v$=5, $a$=20)} & \worst{30.0} & 17.3 & 27.0 & \worst{134.0} & \worst{2585.8} & 56.9 & \worst{4.2} & {-} \\
      & \multicolumn{2}{c}{FAPP} & 80.0 & \best{0.7} & \best{21.7} & \best{106.7} & 240.0 & 4.8 & \best{0.0} & \best{0.0} \\
      & \multicolumn{2}{c}{SANDO} & \best{100.0} & 3.3 & 24.8 & 111.5 & 675.3 & \best{0.0} & \best{0.0} & \best{0.0} \\
      \midrule
      \multirow{9}{*}{Hard}
       & \multicolumn{2}{c}{EGO-Swarm2} & 80.0 & 18.7 & 29.6 & 114.2 & \best{221.2} & 5.1 & \best{0.0} & \best{0.0} \\
      & I-MPC & \multicolumn{1}{l}{($v$=1, $a$=1)} & 40.0 & 31.7 & \worst{85.1} & 121.3 & 940.1 & \best{0.0} & \best{0.0} & {-} \\
      & I-MPC & \multicolumn{1}{l}{($v$=2, $a$=2)} & 30.0 & 30.9 & 51.0 & 123.3 & 942.8 & \best{0.0} & \best{0.0} & {-} \\
      & I-MPC & \multicolumn{1}{l}{($v$=3, $a$=3)} & 30.0 & \worst{32.5} & 36.7 & 119.9 & 1047.0 & \best{0.0} & \best{0.0} & {-} \\
      & I-MPC & \multicolumn{1}{l}{($v$=4, $a$=4)} & \worst{0.0} & {-} & {-} & {-} & {-} & {-} & {-} & {-} \\
      & I-MPC & \multicolumn{1}{l}{($v$=5, $a$=5)} & 10.0 & 26.8 & 28.5 & 132.3 & 983.9 & \worst{44.6} & \best{0.0} & {-} \\
      & I-MPC & \multicolumn{1}{l}{($v$=5, $a$=20)} & 10.0 & 30.0 & 26.4 & \worst{198.0} & \worst{3688.7} & 41.7 & \worst{36.3} & {-} \\
      & \multicolumn{2}{c}{FAPP} & 50.0 & \best{0.6} & 27.9 & \best{111.3} & 587.1 & 3.1 & \best{0.0} & \best{0.0} \\
      & \multicolumn{2}{c}{SANDO} & \best{100.0} & 4.3 & \best{26.3} & 117.8 & 1054.4 & \best{0.0} & \best{0.0} & \best{0.0} \\
      \bottomrule
    \end{tabular}
  }}}
\end{table*}

\subsection{Dynamic Environments without Ground Truth Obstacle Knowledge}\label{sec:dynamic-no-gt}

\begin{table*}
  \caption{Benchmark results in unknown dynamic environments \rev{across the trajectory piece ($N$) and polytope count ($P$) sweep (point-cloud sensing only, 10 trials per configuration). $\rho_{\mathrm{viol}}$ denotes the maximum rate among the velocity, acceleration, and jerk constraint-violation rates.} We highlight the \best{best} and \worst{worst} value for each environment.}
  \label{tab:unknown_dynamic_benchmark}
  \centering
  \renewcommand{\arraystretch}{1.0}
  {\setlength{\fboxrule}{\ifshowrevisions2pt\else0pt\fi}\setlength{\fboxsep}{0pt}%
  \fcolorbox{revcolor}{white}{%
  \resizebox{\dimexpr\textwidth-2\fboxrule\relax}{!}{
    \begin{tabular}{c c c c c c c c c c c}
      \toprule
      \multirow{2}{*}[-0.4ex]{\textbf{Env}}
      & \multirow{2}{*}[-0.4ex]{\textbf{$N$}}
      & \multirow{2}{*}[-0.4ex]{\textbf{$P$}}
      & \multicolumn{1}{c}{\textbf{Success}}
      & \multicolumn{3}{c}{\textbf{Comp.\ Time}}
      & \multicolumn{3}{c}{\textbf{Performance}}
      & \multicolumn{1}{c}{\textbf{Constr.\ Viol.}}
      \\
      \cmidrule(lr){4-4}
      \cmidrule(lr){5-7}
      \cmidrule(lr){8-10}
      \cmidrule(lr){11-11}
      &&&
      $R_{\mathrm{succ}}$ [\%]
      & $T^{\mathrm{per}}_{\mathrm{opt}}$ [ms]
      & $T_{\mathrm{replan}}$ [ms]
      & $T_{\mathrm{STSFC}}$ [ms]
      & $T_{\mathrm{trav}}$ [s]
      & $L_{\mathrm{path}}$ [m]
      & $S_{\mathrm{jerk}}$ [m/s$^{2}$]
      & $\rho_{\mathrm{viol}}$ [\%]
      \\
      \midrule
      \multirow{7}{*}{Easy} & \multirow{3}{*}{5} & 2 & \best{80.0} & \best{3.4} & \best{23.4} & \best{11.3} & 24.5 & 105.9 & 460.3 & \best{0.0} \\
       &  & 3 & \best{80.0} & 5.0 & 31.9 & 15.8 & \best{22.7} & \best{105.5} & \best{249.5} & \best{0.0} \\
       &  & 4 & \worst{50.0} & 6.6 & 46.2 & 23.5 & 23.2 & 105.8 & 272.3 & \best{0.0} \\
      \cmidrule(lr){2-11}
       & \multirow{4}{*}{6} & 2 & \worst{50.0} & 8.6 & 32.0 & 14.0 & \worst{35.7} & 106.2 & \worst{464.0} & \best{0.0} \\
       &  & 3 & \worst{50.0} & 13.2 & 45.7 & 20.5 & 22.8 & 105.9 & 300.5 & \best{0.0} \\
       &  & 4 & \worst{50.0} & 22.9 & 82.0 & 37.5 & 24.4 & 105.8 & 281.2 & \best{0.0} \\
       &  & 5 & 70.0 & \worst{26.7} & \worst{106.9} & \worst{50.1} & 25.7 & \worst{106.9} & 392.0 & \best{0.0} \\
      \midrule
      \multirow{7}{*}{Medium} & \multirow{3}{*}{5} & 2 & \best{60.0} & \best{3.9} & \best{31.3} & \best{15.1} & 24.5 & 106.8 & 635.6 & \best{0.0} \\
       &  & 3 & 40.0 & 7.2 & 48.5 & 24.0 & \best{22.8} & 106.0 & \best{336.0} & \best{0.0} \\
       &  & 4 & 40.0 & 10.4 & 73.7 & 36.5 & 32.3 & \worst{110.5} & \worst{809.1} & \best{0.0} \\
      \cmidrule(lr){2-11}
       & \multirow{4}{*}{6} & 2 & \worst{30.0} & 9.7 & 42.2 & 19.7 & \worst{37.8} & 106.5 & 659.2 & \best{0.0} \\
       &  & 3 & 50.0 & 17.9 & 64.8 & 29.1 & 23.7 & \best{105.8} & 344.7 & \best{0.0} \\
       &  & 4 & \worst{30.0} & 29.2 & 108.6 & 50.2 & 26.1 & 108.4 & 502.2 & \best{0.0} \\
       &  & 5 & 50.0 & \worst{38.8} & \worst{158.9} & \worst{75.8} & 26.1 & 108.0 & 480.2 & \best{0.0} \\
      \midrule
      \multirow{5}{*}{Hard} & \multirow{3}{*}{5} & 2 & \best{40.0} & \best{4.7} & \best{51.4} & \best{24.2} & 28.9 & 110.2 & 1043.1 & \best{0.0} \\
       &  & 3 & 10.0 & 8.8 & 65.6 & 31.1 & \best{22.5} & \best{106.4} & \best{548.9} & \best{0.0} \\
       &  & 4 & 10.0 & 11.2 & \worst{94.6} & \worst{49.0} & 30.9 & \worst{110.9} & \worst{1275.1} & \best{0.0} \\
      \cmidrule(lr){2-11}
       & \multirow{2}{*}{6} & 2 & 10.0 & \worst{14.9} & 66.6 & 30.5 & \worst{58.7} & 109.8 & 1183.6 & \best{0.0} \\
       &  & 3, 4, 5 & \worst{0.0} & {-} & {-} & {-} & {-} & {-} & {-} & {-} \\
      \bottomrule
    \end{tabular}
  }}}
\end{table*}

The dynamic benchmarks in Sections~\ref{sec:benchmarking-dynamic-environments} and~\ref{sec:temporal-sfc-ablation} provide all methods with ground truth obstacle information for fair comparison.
In practice, however, this information is not available.
This section evaluates SANDO in the same dynamic environments but without any ground truth obstacle knowledge, testing the full perception-to-planning pipeline.

The environment, obstacle configuration, dynamic constraints, start/goal positions, and trial settings are identical to those in Section~\ref{sec:benchmarking-dynamic-environments}.
For sensing, we used a simulated Intel RealSense D435 depth camera (via the \texttt{realsense-ros} package) rather than the Livox Mid-360 LiDAR used in the static benchmarks, because the simulated LiDAR produces too few points on the small dynamic obstacles for reliable detection.
This is a simulation artifact; in hardware experiments (Section~\ref{sec:hardware-experiments}), the real LiDAR sensor provides sufficient point density for dynamic obstacle detection.
\rev{Note that the D435's forward-facing field of view (FOV) is substantially narrower than the Mid-360's $360^\circ$ coverage, so significant portions of the environment are unobserved at any instant and are classified as unknown in the voxel map.
SANDO conservatively treats these regions as potentially occupied via the unknown-space inflation in STSFC construction (Section~\ref{subsubsec:unknown_space_inflation}), which makes the FOV-limited simulation a harder setting than a full-LiDAR deployment.}
Point clouds are processed by the temporal occupancy grid and AEKF-based dynamic obstacle tracker (Section~\ref{sec:obstacle_tracking}); the planner receives only estimated obstacle positions, bounding boxes, and predicted trajectories rather than ground truth.
\rev{
To analyze the impact of the $N$--$P$ trade-off, we sweep both the number of trajectory segments $N \in \{5, 6\}$ and the number of spatial polytopes $P \in \{2, \dots, 5\}$.
We require $P < N$ because the MIQP assigns the $N$ trajectory segments to the $P$ polytopes through binary assignment variables: setting $P = N$ would collapse this into a trivial one-to-one mapping and remove the assignment freedom that lets consecutive segments share, or transition between, polytopes.
We cap the sweep at $N = 6$ because the variable-elimination step precomputes the closed-form dependency of the eliminated variables only for $N \in \{4,5,6\}$: the symbolic solve became prohibitively slow at $N = 7$ (Section~\ref{subsubsec:variable_elimination}).
Disabling variable elimination would permit larger $N$ at the cost of slower per-iteration solves; the $N \le 6$ sweep, however, already reveals the trend clearly.}

Table~\ref{tab:unknown_dynamic_benchmark} summarizes the results, where $T_{\mathrm{STSFC}}$ denotes the spatiotemporal safe flight corridor generation time.
\rev{\emph{Success rate:}
as the environment gets denser, the success rate declines ($80\% \to 60\% \to 40\%$ at $N{=}5, P{=}2$) since the feasible space-time shrinks.
In Easy, increasing $N$ degrades the success rate ($80\% \to 50\%$ at $P{\in}\{2,3\}$), while increasing $P$ has no consistent effect. In fact, at $N{=}6$, $P{=}5$ outperforms $P \in \{2, 3, 4\}$ ($70\%$ vs.\ $50\%$).
Medium shows the same trend at $P{=}2$ ($60\% \to 30\%$), and at $N{=}6$, $P{=}5$ ties $P{=}3$ for the best success ($50\%$).
In Hard, increasing either $N$ or $P$ degrades the success rate significantly ($0\%$ for $P \geq 3$ at $N{=}6$), since with 200 obstacles the environment is genuinely hard and fast replanning becomes critical.
\emph{Computation time:}
computation times grow as $N$ and $P$ increase in every environment, since a larger $N \times P$ increases the number of binary assignment variables and polytopes, while at fixed $(N, P)$ they stay nearly the same across difficulties ($3.4 \to \SI{4.7}{ms}$ at $N{=}5, P{=}2$ from Easy to Hard).
\emph{Trajectory performance:}
$P{=}3$ generally achieves the peak performance.
It improves travel time, path length, and smoothness over $P{=}2$, while $P \in \{4, 5\}$ degrade them again.
This is because $P{=}2$ tends to find myopic, less performant trajectories, while the larger computational cost of $P \geq 4$ causes less frequent replanning, which leads to worse performance (the complexity effect is isolated from horizon length in Section~\ref{sec:psweep-disentangle}).
\emph{Constraint violation:}
no violation occurs in any environment or setting ($\rho_{\mathrm{viol}} = 0\%$).}
\rev{$P=2$ achieves the best (or tied-best) success in every environment, and no higher-$P$ configuration achieves a higher success rate.
We therefore fix $P=2$ on hardware rather than adapting it online, since raising $P$ adaptively could only lower the success rate.}

\rev{Compared to the ground-truth results in Section~\ref{sec:benchmarking-dynamic-environments}, where SANDO ($N{=}5$) achieves $100\%$ success across all difficulties, the perception-only pipeline reaches $80\%$ in Easy (at $N{=}5, P{=}2$) and shows a growing gap in Medium (best $60\%$) and Hard (best $40\%$), where perception uncertainty most degrades performance.}
The primary failure mode is late detection of dynamic obstacles: when obstacles are not detected until they are close to the agent, the planner has insufficient time to generate a collision-free trajectory, resulting in collisions.
This highlights the challenge of perception uncertainty in dense dynamic environments.
\rev{Nevertheless, these results show that SANDO's full perception-to-planning pipeline can operate without any ground truth obstacle information, succeeding in the majority of trials in the Easy and Medium environments at the $N{=}5$, $P{=}2$ operating point.}

\subsection{\rev{Disentangling MIQP Complexity from Horizon Length}}\label{sec:psweep-disentangle}

\rev{
The advantage of $P=2$ in Section~\ref{sec:dynamic-no-gt} combines two effects that vary together. Since the local planning horizon is $L = P \cdot \Delta_{\mathrm{vtx}}$, decreasing $P$ at fixed vertex spacing simultaneously (i)~lowers the MIQP complexity ($N \times P$ binary assignment variables) and (ii)~shortens the horizon $L$.
To separate them, we ran a controlled study in the Hard unknown-dynamic environment with $N=5$ and point-cloud sensing only (Table~\ref{tab:psweep_disentangle}).
Panel~(a) holds the horizon fixed at $L=3$\,m (setting $\Delta_{\mathrm{vtx}}=L/P$) while varying $P$, isolating MIQP complexity, and panel~(b) holds $P=2$ fixed while varying $\Delta_{\mathrm{vtx}}$, isolating the horizon $L$.
}

\rev{
In panel~(a), even with the horizon fixed, increasing $P$ from 2 to 3 to 4 raises the per-optimization time from $5.5$ to $8.8$ to $23.2$\,ms and drops the success rate from $40\%$ to $10\%$.
Since corridor conservatism is identical across these rows, the degradation is attributable to MIQP complexity and the slower, less frequent replanning it causes.
In panel~(b), success is highest at the shorter horizons ($40\%$ at both $L=2$\,m and $L=3$\,m) and degrades monotonically beyond ($20\%$ at $L=4$\,m, $10\%$ at $L=5$--$6$\,m), as longer horizons enlarge $dt$ and thus the per-layer inflation, making corridors more conservative and the optimization slower ($T^{\mathrm{per}}_{\mathrm{opt}}$ grows from $4.7$ to $8.7$\,ms by $L=5$\,m, $T_{\mathrm{STSFC}}$ from $24.2$ to $60.7$\,ms across $L=2$--$6$\,m).
Therefore both effects are real, but MIQP complexity dominates at the default vertex spacing of Section~\ref{sec:dynamic-no-gt}: moving from $P=2$ to $P=3$ at $\Delta_{\mathrm{vtx}}=1$\,m lengthens the horizon by only one segment ($L{=}2\!\to\!3$\,m), over which panel~(b) shows no change in success, whereas panel~(a) attributes the entire $40\%\!\to\!10\%$ drop to complexity alone.
}

\begin{table}
  \caption{\rev{Disentangling MIQP complexity from horizon length in the unknown-dynamic \textit{hard} environment ($N=5$, 10 trials each, point-cloud sensing only; Section~\ref{sec:psweep-disentangle}): \textbf{(a)} fixed horizon $L=3$\,m ($\Delta_{\mathrm{vtx}}=L/P$), varying $P$; \textbf{(b)} fixed $P=2$, varying $\Delta_{\mathrm{vtx}}$ and hence $L$. Within each scenario, we mark in \best{green} the best value and in \worst{red} the worst value of each metric column.}}
  \label{tab:psweep_disentangle}
  \centering
  \renewcommand{\arraystretch}{1.0}
  {\setlength{\fboxrule}{\ifshowrevisions2pt\else0pt\fi}\setlength{\fboxsep}{0pt}%
  \fcolorbox{revcolor}{white}{%
  \resizebox{\dimexpr\columnwidth-2\fboxrule\relax}{!}{
    \begin{tabular}{c c c c c c c c}
      \toprule
      \multicolumn{3}{c}{\textbf{Config}}
      & \multicolumn{1}{c}{\textbf{Success}}
      & \multicolumn{3}{c}{\textbf{Comp.\ Time}}
      & \multicolumn{1}{c}{\textbf{Perf.}}
      \\
      \cmidrule(lr){1-3}
      \cmidrule(lr){4-4}
      \cmidrule(lr){5-7}
      \cmidrule(lr){8-8}
      \multirow{2}{*}{$P$} & $\Delta_{\mathrm{vtx}}$ & $L$
      & $R_{\mathrm{succ}}$
      & $T^{\mathrm{per}}_{\mathrm{opt}}$
      & $T_{\mathrm{replan}}$
      & $T_{\mathrm{STSFC}}$
      & $T_{\mathrm{trav}}$
      \\
      & {[m]} & {[m]}
      & {[\%]}
      & {[ms]}
      & {[ms]}
      & {[ms]}
      & {[s]}
      \\
      \midrule
      \multicolumn{8}{l}{(a) Fixed horizon $L=3$\,m, varying $P$ (isolates MIQP complexity)} \\
      2 & 1.50 & \multirow{3}{*}{\textbf{3.0}} & \best{40.0} & \best{5.5} & \best{60.5} & \best{30.4} & \worst{28.0} \\
      3 & 1.00 &                               & \worst{10.0} & 8.8 & 65.6 & 31.1 & \best{22.5} \\
      4 & 0.75 &                               & \worst{10.0} & \worst{23.2} & \worst{123.3} & \worst{60.9} & 23.3 \\
      \midrule
      \multicolumn{8}{l}{(b) Fixed $P=2$, varying $\Delta_{\mathrm{vtx}}$ (isolates horizon $L$)} \\
      \multirow{5}{*}{\textbf{2}} & 1.00 & 2.0 & \best{40.0} & \best{4.7} & \best{51.4} & \best{24.2} & 28.9 \\
                                  & 1.50 & 3.0 & \best{40.0} & 5.5 & 60.5 & 30.4 & 28.0 \\
                                  & 2.00 & 4.0 & 20.0 & 7.4 & 59.3 & 29.1 & 27.8 \\
                                  & 2.50 & 5.0 & \worst{10.0} & \worst{8.7} & 89.7 & 44.5 & \best{25.9} \\
                                  & 3.00 & 6.0 & \worst{10.0} & 7.6 & \worst{125.5} & \worst{60.7} & \worst{68.2} \\
      \bottomrule
    \end{tabular}
  }}}
\end{table}

\subsection{\rev{Ablation: Variable Elimination}}\label{sec:variable-elimination-benchmarking}

We also benchmarked SANDO with and without variable elimination (VE) (see Section~\ref{subsubsec:variable_elimination}) to evaluate the effectiveness of the technique.
The benchmark was performed in the same standardized static environment with $N=4,5,6$ with multi-threaded SANDO under the same dynamic constraints as in Section~\ref{sec:benchmarking-standardized-static-environments}.
Table~\ref{tab:variable_elimination_benchmark} summarizes the results.
The metrics used in the table are the same as those in Section~\ref{sec:benchmarking-standardized-static-environments} except that we report all the constraint violation rates as a single value $\rho_{\mathrm{viol}}$ [\%] (the maximum rate among SFC/velocity/acceleration/jerk constraint violations).
VE achieves identical trajectory performance to the non-VE baseline across all $N$ while reducing per-optimization time by up to $7.4\times$\rev{: by eliminating decision variables and equality constraints, VE solves an equivalent problem, so the same optimal solution is obtained with much faster computation}.

\begin{table}
  \caption{Variable elimination (VE) benchmarking results in standardized environment.
  We highlight the best and worst values for each $N$ in \best{green} and \worst{red}, respectively.}
  \label{tab:variable_elimination_benchmark}
  \centering
  \renewcommand{\arraystretch}{1.0}
  \resizebox{\columnwidth}{!}{
    \begin{tabular}{c c c c c c c c}
      \toprule
      \multirow{2}{*}[-0.2em]{\textbf{N}}
      & \multirow{2}{*}[-0.2em]{\textbf{VE}}
      & $R^{\mathrm{opt}}_{\mathrm{succ}}$
      & $T^{\mathrm{per}}_{\mathrm{opt}}$
      & $T_{\mathrm{trav}}$
      & $L_{\mathrm{path}}$
      & $S_{\mathrm{jerk}}$
      & $\rho_{\mathrm{viol}}$
      \\
      & &
      {[\%]} &
      {[ms]} &
      {[s]} &
      {[m]} &
      {[m/s$^{2}$]} &
      {[\%]}
      \\
      \midrule

      \multirow{2}{*}{4} & \YesGreen & 93.4 & \best{1.3} & 13.2 & 6.8 & 1.4 & 0.0 \\
       & \NoRed & 93.4 & \worst{9.6} & 13.2 & 6.8 & 1.4 & 0.0 \\

      \midrule

      \multirow{2}{*}{5} & \YesGreen & \best{100.0} & \best{6.8} & 11.1 & 6.8 & 1.9 & 0.0 \\
       & \NoRed & \best{100.0} & \worst{12.8} & 11.1 & 6.8 & 1.9 & 0.0 \\

      \midrule

      \multirow{2}{*}{6} & \YesGreen & \best{100.0} & \best{17.2} & 9.8 & 6.8 & 2.8 & 0.0 \\
       & \NoRed & \best{100.0} & \worst{24.7} & 9.8 & 6.8 & 2.8 & 0.0 \\
      \bottomrule
    \end{tabular}
  }
\end{table}

\subsection{\rev{Ablation: STSFC}}\label{sec:temporal-sfc-ablation}

To evaluate the effectiveness of the STSFC approach, we compared it against a worst-case baseline that inflates all dynamic obstacles by \rev{their full-horizon reachable set}.
The STSFC approach inflates obstacles per layer \rev{by the per-axis reachable half-widths $v^{\mathrm{obs}}_{\max,i} \cdot (n{+}1) \cdot dt$}, where $(n{+}1) \cdot dt$ is the end time of layer $n$ (Section~\ref{subsec:temporal_safety_corridor}), whereas the worst-case baseline uses \rev{$v^{\mathrm{obs}}_{\max,i} \cdot T_{\text{traj}}$} for all layers, where $T_{\text{traj}} = N \cdot dt$ is the total trajectory duration.
\rev{We conducted the comparison across the Easy, Medium, and Hard dynamic environments from Section~\ref{sec:benchmarking-dynamic-environments} at $v_{\max} = 5.0$\,m/s, using ground-truth obstacle current positions only (the same input regime as SANDO in Section~\ref{sec:benchmarking-dynamic-environments}), the per-axis obstacle velocity bounds $(v^{\mathrm{obs}}_{\max,x}, v^{\mathrm{obs}}_{\max,y}, v^{\mathrm{obs}}_{\max,z}) = (1.0, 1.0, 1.5)$\,m/s, and $P = 2$, to examine how corridor inflation interacts with obstacle density.
In this ablation, both SFC modes are given the ground-truth current position of every dynamic obstacle but not its future trajectory.}

Table~\ref{tab:sfc_ablation} summarizes the results.
\rev{Across all environments, STSFC maintains a 100\% success rate by inflating obstacles proportionally to each time layer, preserving more free space in earlier layers while still guaranteeing safety.
The worst-case baseline, in contrast, inflates every obstacle by the full-horizon reachable set, which over-constrains the corridors as obstacle density increases.
Its success rate is already lower in Easy (90\% vs.\ 100\%, with one timeout), ties STSFC in Medium (both 100\%), and drops sharply in Hard to 30\% , where the agent struggles to find feasible trajectories through the densely populated scene.
This occurs even though the planner is given full knowledge of every obstacle's current position.
As the environment grows denser, the uniform full-horizon inflation leaves no feasible corridor no matter how accurately the obstacles are known, so the worst-case baseline cannot generate a trajectory regardless of perception quality.
This demonstrates that the failure is a property of the inflation scheme rather than of obstacle information.
Even when both modes succeed, STSFC yields faster travel time and substantially smoother trajectories.
These results show that the per-layer inflation of STSFC is essential for maintaining both safety and efficiency under a realistic obstacle velocity bound, especially in dense environments where the uniform worst-case inflation becomes infeasible.}

\begin{table}
  \caption{\rev{SFC ablation: Worst-Case full-horizon inflation vs.\ SANDO's per-layer STSFC inflation (setup and analysis in Section~\ref{sec:temporal-sfc-ablation}) in the Easy, Medium, and Hard dynamic environments of Table~\ref{tab:dynamic_benchmark}, at $v_{\max}=5$\,m/s with current-position obstacle information and $P=2$. We highlight the \best{best} and \worst{worst} value for each environment.}}
  \label{tab:sfc_ablation}
  \centering
  \renewcommand{\arraystretch}{1.0}
  \setlength{\tabcolsep}{4pt}
  {\setlength{\fboxrule}{\ifshowrevisions2pt\else0pt\fi}\setlength{\fboxsep}{0pt}%
  \fcolorbox{revcolor}{white}{%
  \resizebox{\dimexpr\columnwidth-2\fboxrule\relax}{!}{
    \begin{tabular}{c c c c c c c}
      \toprule
      \multirow{2}{*}[-0.2em]{\textbf{Env}} &
      \multirow{2}{*}[-0.2em]{\textbf{SFC Mode}} &
      $R_{\mathrm{succ}}$ &
      $T^{\mathrm{per}}_{\mathrm{opt}}$ &
      $T_{\mathrm{trav}}$ &
      $L_{\mathrm{path}}$ &
      $S_{\mathrm{jerk}}$
      \\
      & &
      {[\%]} &
      {[ms]} &
      {[s]} &
      {[m]} &
      {[m/s$^{2}$]}
      \\
      \midrule
      \multirow{2}{*}{Easy} & Worst-Case & \worst{90.0} & \best{2.9} & \worst{25.4} & \worst{108.1} & \worst{720.1} \\
       & SANDO (STSFC) & \best{100.0} & \worst{3.0} & \best{24.4} & \best{107.6} & \best{493.7} \\
      \midrule
      \multirow{2}{*}{Medium} & Worst-Case & \best{100.0} & \worst{3.4} & \worst{29.4} & \worst{113.4} & \worst{1212.3} \\
       & SANDO (STSFC) & \best{100.0} & \best{3.3} & \best{24.8} & \best{111.5} & \best{675.3} \\
      \midrule
      \multirow{2}{*}{Hard} & Worst-Case & \worst{30.0} & \worst{5.1} & \worst{34.0} & \worst{118.8} & \worst{2043.9} \\
       & SANDO (STSFC) & \best{100.0} & \best{4.3} & \best{26.3} & \best{117.8} & \best{1054.4} \\
      \bottomrule
    \end{tabular}
  }}}
\end{table}

\section{Hardware Experiments}\label{sec:hardware-experiments}

To evaluate the performance of SANDO, we conducted hardware experiments in static and dynamic environments.
\rev{All hardware experiments use unknown dynamic obstacles.
No motion capture or external tracking is used for any obstacle in the planning pipeline (in Experiments 7--10, motion capture is used only as ground truth for evaluation). Dynamic-obstacle positions, AABBs, and velocities are estimated entirely from the onboard Livox Mid-360 LiDAR point cloud.}
Fig.~\ref{fig:hardware_uav_picture} shows the UAV platform used in our experiments. 
For perception, we use a Livox Mid-360 LiDAR sensor, and for localization, we use onboard DLIO~\cite{chen2023dlio}.
SANDO runs on an Intel\textsuperscript{\textregistered} NUC 13 with an Intel\textsuperscript{\textregistered} Core~\textsuperscript{TM} i7 CPU $\times$16, 64 GB of RAM, and for low-level control, we use PX4~\cite{meier2015px4} on a Pixhawk flight controller~\cite{meier2011pixhawk}.  
\rev{PX4 implements a standard cascaded PID architecture with an inner-loop body-rate controller at $\SI{1}{kHz}$, an attitude controller at $\SI{250}{Hz}$, and an outer position/velocity loop at $\SI{50}{Hz}$.
The full reference state (position, velocity, acceleration, yaw) is sampled from the planned B\'ezier trajectory and streamed to PX4. 
Position and yaw serve as setpoints, while velocity, acceleration, roll, and pitch are passed as feedforward terms to the cascaded controller.
}
All perception, planning, control, and localization modules run onboard in real time, enabling fully autonomous operations.
As summarized in Table~\ref{tab:experiment_parameters}, we use $w_{\text{heat}} = 10.0$ for Experiments 7--18 (dynamic obstacles) to account for the higher collision risk in denser environments.

\begin{figure}[t]
    \centering
    \includegraphics[width=\columnwidth, clip, trim=0 120 0 10]{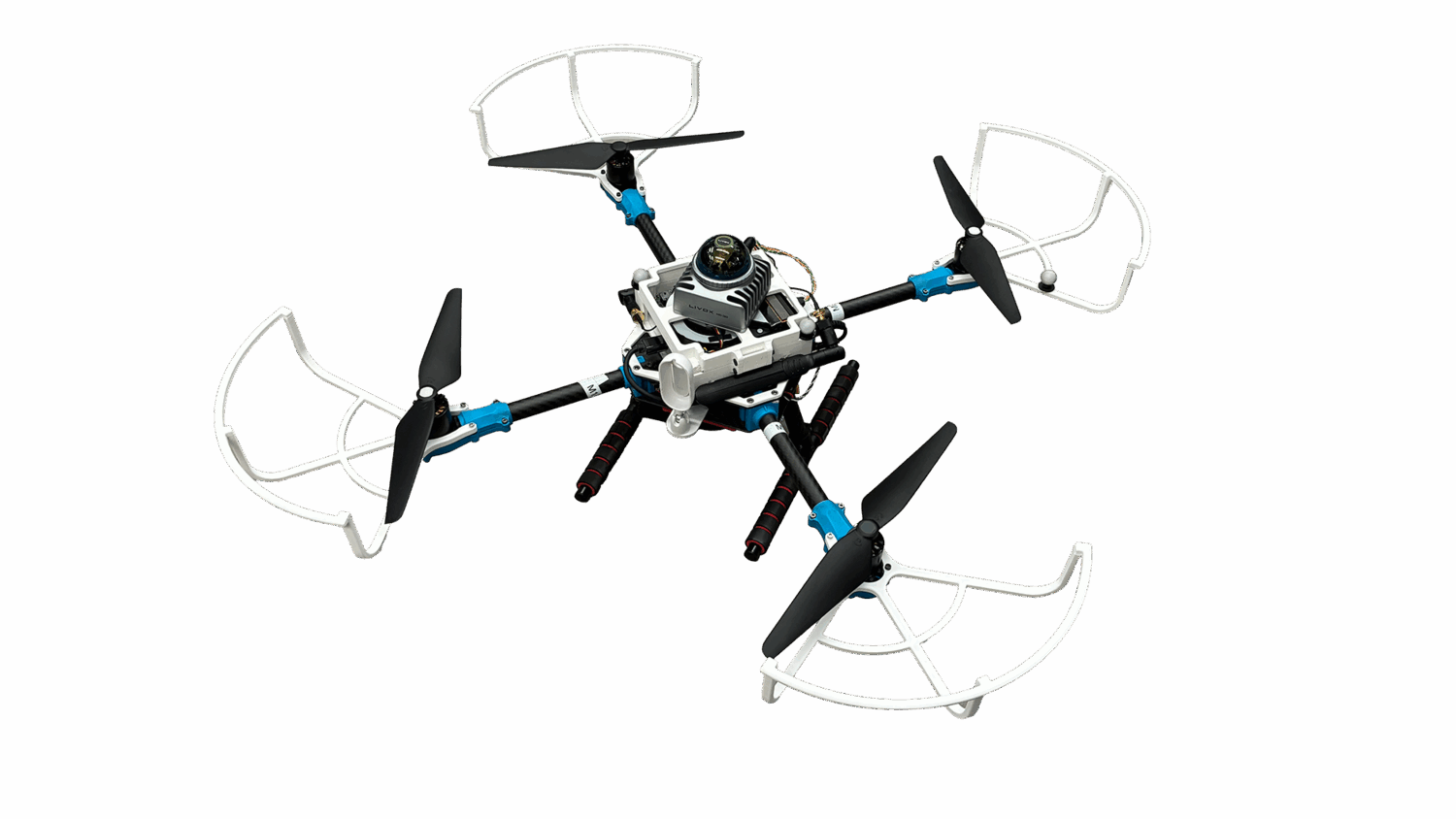}
    \caption{Holybro PX4 Development Kit X500 is equipped with a Pixhawk flight controller and protective propeller guards.
    A Livox Mid-360 LiDAR is mounted on top for perception and localization.\label{fig:hardware_uav_picture}}
\end{figure}

\subsection{Static Environments}

\begin{figure}[t]
    \centering
    \includegraphics[width=\columnwidth]{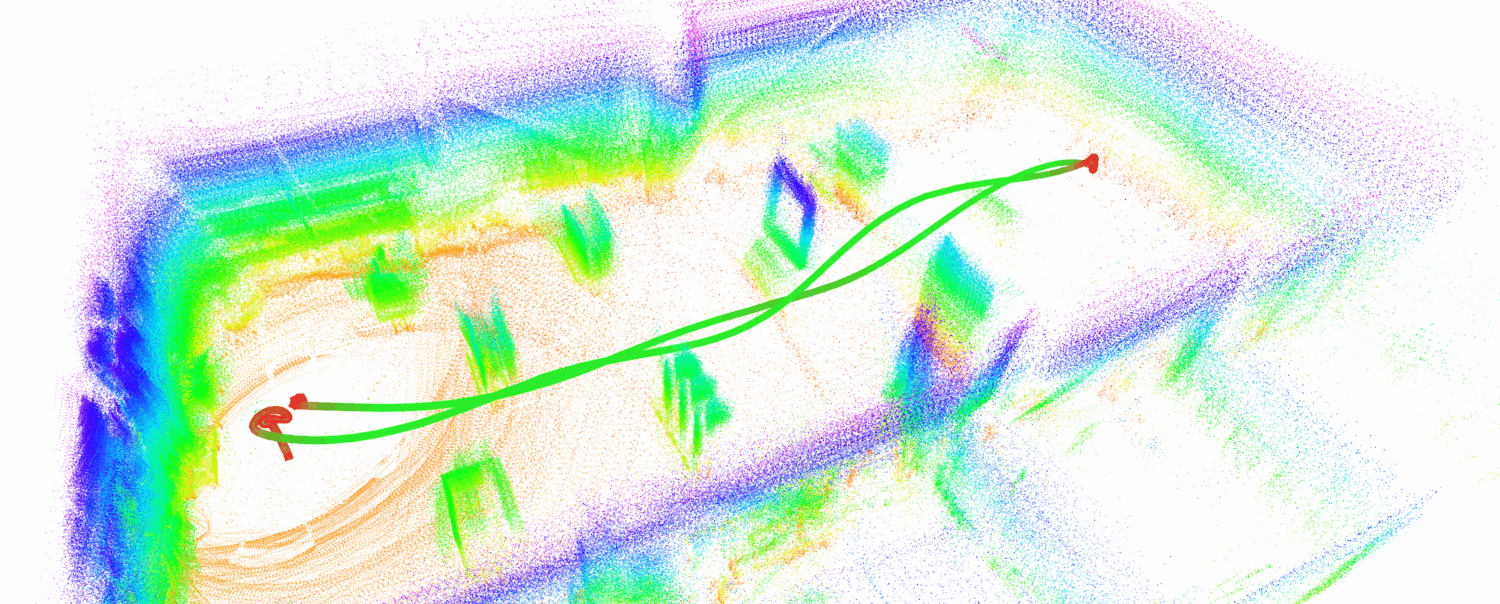}
    \caption{Static environment experiment (Experiment 3): the UAV successfully navigates through a cluttered environment with static obstacles. The agent's speed is color-coded (green for faster, red for slower), and the point cloud is colored by height.\label{fig:hw_static_pc}}
\end{figure}

We first evaluated SANDO in a static indoor environment where obstacles were placed in a $20\,\text{m} \times 8\,\text{m}$ area.
The UAV was tasked with flying from $x=0.0$\,m to $x=20.0$\,m and returning to the start.
To assess performance across a range of speeds, we conducted six flights with velocity limits $v_{\max} \in \{1.0, 2.0, 3.0, 4.0, 5.0, 6.0\}$\,m/s.
The acceleration and jerk limits were set to $a_{\max} = 5.0$\,m/s$^2$ and $j_{\max} = 10.0$\,m/s$^3$ for Experiments~1 to 4.
For Experiments~5 and 6, $j_{\max}$ was reduced to $7.5$\,m/s$^3$ to mitigate the larger tracking errors observed at higher speeds.
Fig.~\ref{fig:hw_static_pc} shows the resulting trajectory overlaid on the LiDAR point cloud; the UAV successfully avoids all static obstacles across the entire speed range.

\begin{table}
  \caption{Hardware flight computation times in static environments with increasing velocity limits\rev{ (constraint settings in the text)}.}
  \label{tab:hw_static}
  \centering
  \renewcommand{\arraystretch}{1.0}
  \resizebox{\columnwidth}{!}{
    \begin{tabular}{c c c c c c}
      \toprule
      \multirow{2}{*}[-0.4ex]{\textbf{Exp.}} & \multirow{2}{*}[-0.4ex]{\makecell{$v_{\max}$ \\ {[m/s]}}} & \multicolumn{4}{c}{\textbf{Computation Time [ms]}} \\
      \cmidrule(lr){3-6}
      & & $T_{\mathrm{replan}}$ & $T_{\mathrm{global}}$ & $T_{\mathrm{SSFC}}$ & $T_{\mathrm{opt}}$ \\
      \midrule
      1 & 1.0 & $35.8 \pm 8.8$ & $0.1 \pm 0.1$ & $6.1 \pm 3.0$ & $14.0 \pm 6.3$ \\
      2 & 2.0 & $35.7 \pm 8.8$ & $0.1 \pm 0.1$ & $6.5 \pm 3.2$ & $12.3 \pm 6.4$ \\
      3 & 3.0 & $34.5 \pm 7.6$ & $0.1 \pm 0.1$ & $7.7 \pm 3.2$ & $9.7 \pm 5.2$ \\
      4 & 4.0 & $34.8 \pm 8.0$ & $0.1 \pm 0.1$ & $8.9 \pm 4.0$ & $8.6 \pm 4.2$ \\
      5 & 5.0 & $33.9 \pm 7.0$ & $0.1 \pm 0.1$ & $7.9 \pm 3.6$ & $8.3 \pm 4.0$ \\
      6 & 6.0 & $32.6 \pm 6.8$ & $0.1 \pm 0.1$ & $7.7 \pm 3.6$ & $8.2 \pm 4.3$ \\
      \bottomrule
    \end{tabular}
  }
\end{table}

Table~\ref{tab:hw_static} reports the computation times for all six flights, where $T_{\mathrm{replan}}$, $T_{\mathrm{global}}$, $T_{\mathrm{SSFC}}$, and $T_{\mathrm{opt}}$ denote the average total replanning, global planning, spatial safe flight corridor (SSFC) generation, and trajectory optimization times, respectively.
As in the static simulation benchmark (Section~\ref{sec:benchmarking-standardized-static-environments}), SANDO uses SSFCs rather than STSFCs since there are no dynamic obstacles.
Although the computation times are higher than in simulation due to the less powerful onboard computer, SANDO consistently maintains real-time performance, with average replanning times around \SI{35}{\ms} across all velocity limits.
\rev{In all six flights, the UAV satisfies the dynamic constraints while maintaining smooth velocity transitions.}

\begin{figure}[t]
    \centering
    \includegraphics[width=\columnwidth, clip, trim=0 200 300 300]{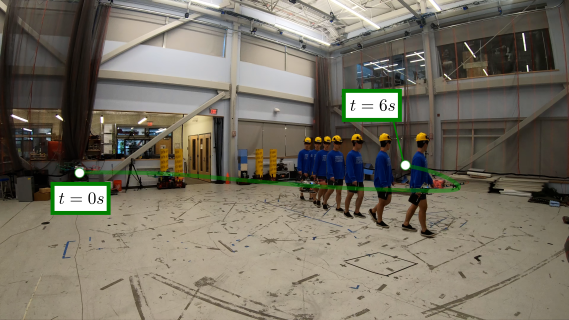}
    \caption{Experiment 10: the UAV avoids a single dynamic obstacle moving at \SI{0.5}{\m/\s} in an $8\,\text{m} \times 8\,\text{m}$ area. Screenshots are taken every 0.5\,s.\label{fig:hardware_uav_dynamic_exp10}}
\end{figure}

\begin{figure}[th]
    \centering
    \includegraphics[width=\columnwidth, clip, trim=0 0 0 0]{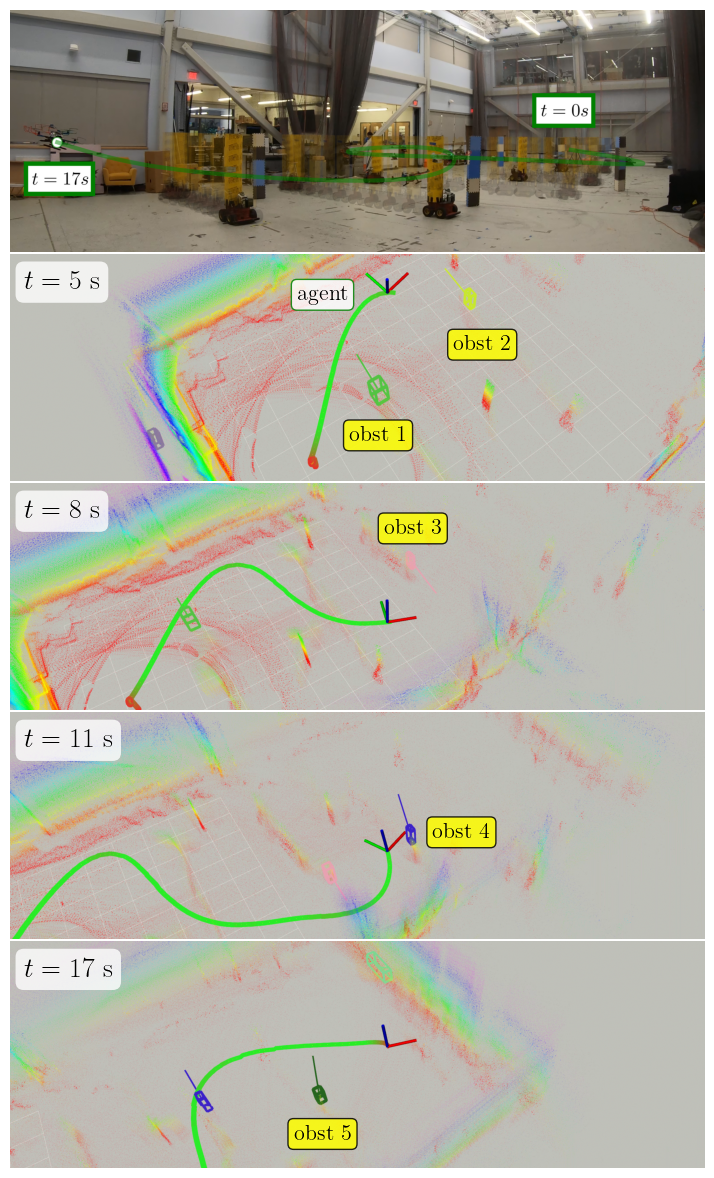} %
    \caption{Experiment 16: the UAV flies at $v_{\max}=2.0$\,m/s through an $8\,\text{m} \times 20\,\text{m}$ area with five dynamic obstacles and additional static obstacles.
    The detected dynamic obstacles are visualized as boxes, and the future predicted trajectories of the dynamic obstacles are shown as a line.
    \label{fig:hardware_uav_dynamic_exp16}}
\end{figure}

\begin{table}
  \caption{\rev{Computation times and safety metrics for the single-dynamic-obstacle flights for Experiments 7--10; metric definitions in the text.}}
  \label{tab:hw_single_dynamic_obstacle}
  \centering
  \renewcommand{\arraystretch}{1.05}
  \resizebox{\columnwidth}{!}{
    \begin{tabular}{c c c c c c c c c c}
      \toprule
      \multirow{3}{*}[-0.4ex]{\textbf{Exp.}} & \multirow{3}{*}[-0.4ex]{\makecell{\textbf{Obst.} \\ \textbf{Traj.}}} & \multicolumn{4}{c}{\textbf{Computation Time [ms]}} & \multicolumn{4}{c}{\rev{\textbf{Performance}}} \\
      \cmidrule(lr){3-6} \cmidrule(lr){7-10}
      & & $T_{\mathrm{replan}}$ & $T_{\mathrm{global}}$ & $T_{\mathrm{STSFC}}$ & $T_{\mathrm{opt}}$ & \rev{\makecell{$\sigma_{\mathrm{track}}$ \\ {[m]}}} & \rev{\makecell{$f_{\mathrm{fail}}$ \\ {[\%]}}} & \rev{\makecell{$d_{\min}$ \\ {[m]}}} & \rev{\makecell{$e_{\mathrm{est}}$ \\ {[m]}}} \\
      \midrule
      7  & Line       & \rev{$18.6$} & \rev{$0.1$} & \rev{$4.9$}  & \rev{$4.6$} & \rev{0.32} & \rev{5.9} & \rev{2.01} & \rev{0.30} \\
      8  & Circle     & \rev{$18.4$} & \rev{$0.1$} & \rev{$4.8$} & \rev{$4.6$} & \rev{0.30} & \rev{3.0} & \rev{1.51} & \rev{0.29} \\
      9  & Fig. Eight & \rev{$19.1$} & \rev{$0.1$} & \rev{$4.9$} & \rev{$4.9$} & \rev{0.27} & \rev{12.4} & \rev{1.59} & \rev{0.30} \\
      10 & Person     & \rev{$18.9$} & \rev{$0.1$} & \rev{$5.0$} & \rev{$4.7$} & \rev{0.42} & \rev{6.4} & \rev{1.33} & \rev{0.43} \\
      \bottomrule
    \end{tabular}
  }
\end{table}

\begin{table}
  \caption{\rev{Computation times and \revtwo{performance} metrics for the multi-dynamic-obstacle flights (Experiments 11--18) in the cluttered $8\,\mathrm{m}\times20\,\mathrm{m}$ space.} \revtwo{$T^{\max}_{\mathrm{fail}}$ is the longest run of consecutive replanning failures, and $\hat{d}_{\min}$ the minimum center-to-center agent-to-obstacle distance as estimated onboard (not ground truth; see Section~\ref{sec:hardware-experiments}). Remaining metric definitions in the text.}}
  \label{tab:hw_multiple_dynamic_obstacles}
  \centering
  \renewcommand{\arraystretch}{1.05}
  \resizebox{\columnwidth}{!}{
    \begin{tabular}{c c c c c c c c c c c}
      \toprule
      \multirow{3}{*}[-0.4ex]{\textbf{Exp.}} & \multirow{3}{*}[-0.4ex]{\makecell{\textbf{Obst.} \\ \textbf{Type}}} & \multirow{3}{*}[-0.4ex]{\makecell{$v_{\max}$ \\ {[m/s]}}} & \multicolumn{4}{c}{\textbf{Computation Time [ms]}} & \multicolumn{4}{c}{\rev{\textbf{Performance}}} \\
      \cmidrule(lr){4-7} \cmidrule(lr){8-11}
      & & & $T_{\mathrm{replan}}$ & $T_{\mathrm{global}}$ & $T_{\mathrm{STSFC}}$ & $T_{\mathrm{opt}}$ & \rev{\makecell{$\sigma_{\mathrm{track}}$ \\ {[m]}}} & \rev{\makecell{$f_{\mathrm{fail}}$ \\ {[\%]}}} & \revtwo{\makecell{$T^{\max}_{\mathrm{fail}}$ \\ {[ms]}}} & \revtwo{\makecell{$\hat{d}_{\min}$ \\ {[m]}}} \\
      \midrule
      11 & \multirow{3}{*}{5 Dyn.} & \multirow{3}{*}{2.0} & \rev{$22.1$} & \rev{$0.1$} & \rev{$5.3$} & \rev{$4.7$} & \rev{0.28} & \rev{21.5} & \revtwo{677} & \revtwo{0.34} \\
      12 &                         &                      & \rev{$22.4$} & \rev{$0.2$} & \rev{$5.3$} & \rev{$5.2$} & \rev{0.30} & \rev{13.2} & \revtwo{700} & \revtwo{1.04} \\
      13 &                         &                      & \rev{$21.2$} & \rev{$0.1$} & \rev{$4.8$} & \rev{$4.7$} & \rev{0.28} & \rev{0.4} & \revtwo{44} & \revtwo{1.50} \\
      \midrule
      14 & \multirow{5}{*}{\makecell{5 Dyn. \\ \& Static}} & \multirow{3}{*}{2.0} & \rev{$20.8$} & \rev{$0.1$} & \rev{$5.3$} & \rev{$4.2$} & \rev{0.28} & \rev{9.6} & \revtwo{365} & \revtwo{2.23} \\
      15 &                                                  &  & \rev{$21.1$} & \rev{$0.1$} & \rev{$5.0$} & \rev{$4.9$} & \rev{0.39} & \rev{7.9} & \revtwo{662} & \revtwo{1.27} \\
      16 &                                                  &  & \rev{$23.5$} & \rev{$0.3$} & \rev{$5.9$} & \rev{$5.8$} & \rev{0.57} & \rev{8.9} & \revtwo{291} & \revtwo{0.93} \\
      \cmidrule(lr){3-11}
      17 &                                                  & 3.0 & \rev{$19.7$} & \rev{$0.1$} & \rev{$4.3$} & \rev{$4.1$} & \rev{0.21} & \rev{2.3} & \revtwo{105} & \revtwo{2.04} \\
      18 &                                                  & 4.0 & \rev{$20.1$} & \rev{$0.1$} & \rev{$4.7$} & \rev{$4.4$} & \rev{0.32} & \rev{8.1} & \revtwo{205} & \revtwo{1.07} \\
      \bottomrule
    \end{tabular}
  }
\end{table}

\subsection{UAV in Dynamic Environments}

To evaluate SANDO in dynamic environments, we conducted hardware experiments involving one or five dynamic obstacles, each created by \revtwo{mounting a foam obstacle on} a wheeled platform.
\revtwo{Each obstacle is a cylinder or a square box about \SI{0.3}{m} across in plan view, and the obstacle and platform together stand about \SI{2.0}{m} tall.}
In Experiments 7 to 10, the UAV flew back and forth ($x=0.0$\,m to $x=8.0$\,m, \rev{five round trips}) in an $8\,\text{m} \times 8\,\text{m}$ area with a single dynamic obstacle moving at approximately \SI{0.5}{\m/\s}: Experiments 7--9 used linear, circular, and figure-eight obstacle trajectories, while Experiment 10 (Fig.~\ref{fig:hardware_uav_dynamic_exp10}) used a person walking randomly to test SANDO's ability to handle unpredictable human motion.
Experiments 11 to 13 had five dynamic obstacles in an $8\,\text{m} \times 20\,\text{m}$ area, where each obstacle follows a linear trajectory at \SI{0.5}{\m/\s}.
The agent was tasked with flying from $x=0.0$\,m to $x=20.0$\,m, and in Experiments 14 to 18 (Fig.~\ref{fig:hardware_uav_dynamic_exp16}), the same five dynamic obstacles were placed in an environment with additional static obstacles\rev{, requiring the UAV to avoid both static and dynamic obstacles while flying through the cluttered scene}.
\rev{Experiments 14, 15, and 16 used $v_{\max}=\SI{2.0}{m/s}$; Experiments 17 and 18 raise the speed constraint to $v_{\max}=\SI{3.0}{m/s}$ and $\SI{4.0}{m/s}$ respectively.}
The dynamic constraints were set to $v_{\max} = 2.0\,\text{m/s}$, $a_{\max} = 5.0\,\text{m/s}^2$, and $j_{\max} = 7.5\,\text{m/s}^3$ for Experiments 7 to 16; $v_{\max}$ was raised to \rev{$\SI{3.0}{m/s}$ in Experiment 17 and $\SI{4.0}{m/s}$ in Experiment 18}, with $a_{\max}$ and $j_{\max}$ unchanged.

Tables~\ref{tab:hw_single_dynamic_obstacle} and~\ref{tab:hw_multiple_dynamic_obstacles} summarize the computation times \rev{and performance metrics} for all twelve dynamic environment experiments, where $T_{\mathrm{STSFC}}$ denotes the spatiotemporal safe flight corridor generation time.
\rev{Table~\ref{tab:hw_single_dynamic_obstacle} shows the single-obstacle flights (Experiments 7--10), where motion capture tracks both the drone and the obstacle. We report four performance metrics: the minimum center-to-center drone-to-obstacle distance $d_{\min}$; the root-mean-square (RMS) Euclidean tracking error $\sigma_{\mathrm{track}}$ between the commanded setpoint trajectory and the onboard pose estimate; the per-cycle replanning failure rate $f_{\mathrm{fail}}$, the fraction of replan attempts that return infeasibility; and the mean AEKF obstacle-position estimation error $e_{\mathrm{est}}$ against motion capture.
Table~\ref{tab:hw_multiple_dynamic_obstacles} shows the multi-obstacle flights (Experiments 11--18) in the larger $8\,\text{m}\times20\,\text{m}$ area, where motion capture does not cover the entire space.
Since the obstacles are tracked from the onboard LiDAR only, $d_{\min}$ and $e_{\mathrm{est}}$ cannot be referenced to ground truth and are omitted.}

\revtwo{For those flights we instead report $T^{\max}_{\mathrm{fail}}$, the longest time the agent flew its previous trajectory without a new plan, and $\hat{d}_{\min}$, the smallest center-to-center agent-to-obstacle distance estimated onboard (the same quantity as $d_{\min}$).
This distance comes from the same tracker output the planner uses, so it is an estimate and not ground truth, and it is taken over all tracked obstacles.
Across Experiments 11 to 18, the largest $T^{\max}_{\mathrm{fail}}$ is \SI{700}{ms}.
The committed trajectory is never exhausted, and the agent never reaches the stop state of Section~\ref{subsec:recursive_feasibility}.
The value of $\hat{d}_{\min}$ stays between \SI{0.34}{m} and \SI{2.23}{m}.
The smallest value is from Experiment 11, where the tracker first reported the obstacle involved only \SI{1.2}{s} before the closest approach.
An obstacle reported this late is already close when the planner first sees it, and Section~\ref{subsec:assumption_violations} lists this as a failure mode.
The estimated geometry at that instant would have put the agent in contact with the estimated obstacle, which did not occur, so the estimate was in error there by about the same amount as the \SI{0.29}{m} to \SI{0.43}{m} measured in Experiments 7 to 10.
The same episode accounts for most of the replanning failures in that flight.
Replanning failed on \SI{78}{\percent} of cycles in the \SI{2.5}{s} after that obstacle was first tracked and on \SI{6.6}{\percent} over the rest of the flight, so the \SI{21.5}{\percent} rate in Table~\ref{tab:hw_multiple_dynamic_obstacles} comes from one short episode rather than from sustained infeasibility.}

\rev{These metrics also reveal a gap between Theorem~\ref{thm:safety} and the real platform.
The theorem assumes a bounded estimation error $\epsilon$ and perfect tracking (Assumptions~\ref{ass:estimation_error} and~\ref{ass:perfect_tracking}), but neither holds exactly here.
The onboard estimation error $e_{\mathrm{est}}$ is about \SI{0.30}{m} (\SI{0.43}{m} for the walking person in Experiment 10), at or above the \SI{0.3}{m} we set for $\epsilon$, and tracking is imperfect, with $\sigma_{\mathrm{track}}$ between \num{0.21} and \SI{0.57}{m}.
Tracking is degraded for two reasons: the hardware drone is large and heavy because it carries all computation onboard, which reduces tracking agility, and its pose is estimated onboard rather than measured by motion capture.
The formal guarantee therefore does not strictly hold on hardware.
Even so, the obstacle inflation, the \SI{0.3}{m} buffer, and continuous replanning lead to safe navigation in practice: the \revtwo{center-to-center distance} $d_{\min}$ never drops below \SI{1.33}{m} \revtwo{in Experiments 7 to 10, where motion capture is available, and no collision occurred in any of the twelve flights}.
The failure rate $f_{\mathrm{fail}}$ rises in denser scenes (up to \SI{21.5}{\percent} in Experiment 11), but a failed replan keeps the previous trajectory, \revtwo{which is collision-free with respect to the obstacles that were known when it was planned}.}

Overall, SANDO \revtwo{completes} all \rev{twelve} experiments \revtwo{without collision} while maintaining real-time performance.
\section{Conclusions}\label{sec:conclusion}

This paper presented SANDO, a safe trajectory planner for 3D dynamic unknown environments.
SANDO combines a heat map-based global planner with STSFC generation that inflates dynamic obstacles by worst-case reachable sets, a variable elimination technique that reduces hard-constraint MIQP optimization to as few as one free variable per axis (for $N=4$), and an AEKF-based dynamic obstacle tracker.
\rev{In simulation, SANDO was the only planner to reach a 100\% success rate across all difficulty levels of the forest and dynamic benchmarks while never violating its dynamic constraints, and matched the trajectory quality of the hard-constraint baselines in the standardized benchmark at a lower computation time;} perception-only experiments without ground truth obstacle information \rev{demonstrated the full perception-to-planning pipeline} under realistic sensing conditions.
Hardware experiments on a quadrotor validated the approach.

\rev{The current framework has three main limitations.}
First, as discussed in Section~\ref{subsec:recursive_feasibility}, SANDO does not guarantee recursive feasibility due to its moving subgoal structure.
Second, the worst-case reachable set inflation can become overly conservative in dense environments, as evidenced by the perception-only results in Section~\ref{sec:dynamic-no-gt}.
\rev{Third, the safety guarantee of Theorem~\ref{thm:safety} holds under the strong assumptions of Section~\ref{subsec:theoretical_safety_guarantees}, including bounded obstacle velocity, bounded estimation error, and tracking error, and the hardware measurements in Section~\ref{sec:hardware-experiments} show that these assumptions are not satisfied exactly on the real platform.}
\rev{Future work will address these limitations by integrating terminal velocity constraints for recursive feasibility guarantees, incorporating learning-based prediction to reduce inflation conservatism, relaxing the assumptions of Theorem~\ref{thm:safety} by explicitly accounting for estimation and tracking errors in the corridor construction, extending to larger-scale multi-agent deployments, and validating the system in larger outdoor environments with more complex obstacle configurations.}

\bibliographystyle{IEEEtran}
\bibliography{root}

\end{document}